\documentclass[11pt, oneside]{article}
\usepackage{geometry}
\usepackage{graphicx}	
\usepackage{amssymb}
\usepackage{amsfonts,latexsym,amsthm,amssymb,amsmath,amscd,euscript}
\usepackage{bbm}
\usepackage{framed}
\usepackage{fullpage}
\usepackage{mathrsfs}
\usepackage{enumitem}
\usepackage{hyperref}
\hypersetup{
  colorlinks=true,
  linkcolor=blue,
  filecolor=blue,
  citecolor = black,      
  urlcolor=cyan,
}
\usepackage{accents}
\usepackage{setspace}
\hypersetup{colorlinks=true,citecolor=blue,urlcolor=black,linkbordercolor={1 0 0}}
\usepackage{tikz-cd}
\usepackage[boxruled,vlined,nofillcomment,linesnumbered]{algorithm2e}

\makeatletter
\newtheorem*{rep@theorem}{\rep@title}
\newcommand{\newreptheorem}[2]{%
\newenvironment{rep#1}[1]{%
 \def\rep@title{#2 \ref{##1}}%
 \begin{rep@theorem}}%
 {\end{rep@theorem}}}
\makeatother

\theoremstyle{plain}
\newtheorem{theorem}{Theorem}
\newreptheorem{theorem}{Theorem}
\newtheorem{lemma}[theorem]{Lemma}
\newtheorem{corollary}[theorem]{Corollary}

\newtheorem{proposition}[theorem]{Proposition}

\newtheorem{claim}[theorem]{Claim}

\theoremstyle{definition}
\newtheorem{definition}{Definition}
\newtheorem{defn}[definition]{Definition}

\numberwithin{theorem}{section}
\numberwithin{definition}{section}

\newcommand{\nc}{\newcommand}
\nc{\DMO}{\DeclareMathOperator}
\newcount\Comments  %
\DeclareMathOperator*{\argmin}{arg\,min} %

\DMO{\prox}{prox}
\DMO{\Span}{span}
\DMO{\diag}{Diag}
\nc{\wtil}{\widetilde}
\nc{\mix}[1]{{\rm mix}({#1})}
\nc{\LL}{D}
\nc{\disc}[2]{\left \lfloor {#1} \right\rfloor_{#2}}

\nc{\todo}[1]{\ifnum\Comments=1 {\color{red}  [TODO: #1]}\fi}
\nc{\old}[1]{\ifnum\Comments=1 {\color{brown}  [OLD: #1]}\fi}
\nc{\noah}[1]{\ifnum\Comments=1 {\color{purple} [ng: #1]}\fi}
\nc{\costis}[1]{\ifnum\Comments=1 {\color{red} [cd: #1]}\fi}

\nc{\tvnorm}[1]{\left\| {#1} \right\|_1}
\nc{\infnorm}[1]{\left\| {#1} \right\|_\infty}
\nc{\infnorms}[2]{\left\| {#1} \right\|_{\infty, {#2}}}

\nc{\can}[2]{{\rm can}_{#2}({#1})}
\nc{\lam}{\lambda}
\nc{\lamp}{{\lam'}}
\nc{\lampp}{{\lam''}}
\nc{\Lam}{\Lambda}
\nc{\Lamdim}{{\Lam-6}}
\nc{\Lampdim}{{\Lam-3}}
\nc{\hvconst}{C_\Lambda} %
\nc{\hvval}{3072\Lambda^3}
\nc{\til}{\tilde}
\nc{\tvd}[2]{\Delta({#1},{#2})}
\nc{\Delc}{\Delta^\circ}
\nc{\Algfo}{\mathcal{A}^{\rm fo}}
\nc{\wsc}[2]{[{#1},{#2}]}

\nc{\mbproper}[1]{\texttt{MB-PROP}({#1})}
\nc{\OSE}{\texttt{Optimistic SOA-Experts}\xspace}
\nc{\OSEgame}{\texttt{Independent Learning in a Game}\xspace}
\nc{\HIL}{\texttt{Multi-scale Improper Learner}\xspace}
\nc{\HPL}{\texttt{Multi-scale Proper Learner}\xspace}
\nc{\OH}{\texttt{Optimistic Hedge}\xspace}
\nc{\PSR}{\texttt{PSR-Learner}\xspace}
\nc{\etaOH}{\eta_{\texttt{OH}}}
\nc{\etaPSR}{\eta_{\texttt{PSR}}}

\nc{\CA}{\mathscr{A}}
\nc{\lgame}[3]{{#1}_{#2}^{#3}}
\DMO{\Hagg}{HAgg}
\DMO{\TERR}{TErr}
\nc{\Terr}[1]{\TERR({#1})}
\nc{\cutoff}[1]{\bar \lambda({#1})}
\nc{\hagg}[1]{\Hagg({#1})}
\DMO{\Ldim}{Ldim}
\DMO{\VCdim}{VCdim}
\DMO{\VOTE}{Vote}
\DMO{\RAND}{Rand}
\nc{\Vote}[2]{\VOTE_{#1}({#2})}
\nc{\Votef}[1]{\RAND({#1})}
\DMO{\Median}{Med}
\DMO{\sfat}{sfat}
\DMO{\rfat}{rsfat}
\DMO{\fat}{fat}
\DMO{\SOA}{SOA}
\DMO{\ROA}{rSOA}
\nc{\soa}[2]{\SOA({#1})({#2})}
\nc{\soaf}[1]{\SOA({#1})}
\nc{\soal}[3]{\SOA({#1},{#2})({#3})}
\nc{\soalf}[2]{\SOA({#1},{#2})}
\nc{\roa}[2]{\ROA_{#1}(#2)}
\nc{\roaf}[1]{\ROA_{#1}}
\nc{\roal}[3]{\ROA_{#1}^{#2}(#3)}
\nc{\roalf}[2]{\ROA_{#1}^{#2}}
\DMO{\MAJ}{Maj}
\nc{\Maj}[2]{\MAJ({#1})({#2})}
\nc{\Majf}[1]{\MAJ({#1})}
\DMO{\HIGHVOTE}{HighVote}
\nc{\Highvote}[1]{\HIGHVOTE({#1})}
\DMO{\conv}{conv}
\nc{\st}{\star}
\nc{\lng}{\langle}
\nc{\rng}{\rangle}
\nc{\grad}{\nabla}
\nc{\MG}{\mathcal{G}}
\nc{\MP}{\mathcal{P}}
\nc{\MQ}{\mathcal{Q}}
\nc{\MV}{\mathcal{V}}
\nc{\MM}{\mathcal{M}}
\nc{\PP}{\mathbb{P}}
\nc{\TT}{\mathbb{T}}
\nc{\TTmax}{\TT_{\max}}
\DMO{\Reg}{Reg}
\DMO{\Ham}{Ham}
\DMO{\Gap}{Gap}
\DMO{\GD}{GD}
\DMO{\GDA}{GDA}
\DMO{\EG}{EG}
\DMO{\OGDA}{OGDA}
\DMO{\Unif}{Unif}
\DMO{\Tr}{Tr}
\nc{\CF}{\mathscr{F}}
\nc{\CE}{\mathscr{E}}

\nc{\algnst}[1]{\begin{align*}#1\end{align*}}
\nc{\algn}[1]{\begin{align}#1\end{align}}
\nc{\matx}[1]{\left(\begin{matrix}#1\end{matrix}\right)}

\nc{\nuu}{\nu}

\nc{\bv}{\mathbf{v}}
\nc{\bt}{\mathbf{t}}
\nc{\bone}{\mathbf{1}}
\nc{\bX}{\mathbf{X}}
\nc{\bY}{\mathbf{Y}}
\nc{\bG}{\mathbf{G}}
\nc{\bz}{\mathbf{z}}
\nc{\bw}{\mathbf{w}}
\nc{\bB}{\mathbf{B}}
\nc{\bA}{\mathbf{A}}
\nc{\bJ}{\mathbf{J}}
\nc{\bK}{\mathbf{K}}
\nc{\bb}{\mathbf{b}}
\nc{\ba}{\mathbf{a}}
\nc{\bc}{\mathbf{c}}
\nc{\bC}{\mathbf{C}}
\nc{\BR}{\mathbb R}
\nc{\BA}{\mathbb{A}}
\nc{\BP}{\mathbb{P}}
\nc{\BC}{\mathbb C}
\nc{\bx}{\mathbf{x}}
\nc{\bs}{\mathbf{s}}
\nc{\bp}{\mathbf{p}}
\nc{\bS}{\mathbf{S}}
\nc{\bM}{\mathbf{M}}
\nc{\bR}{\mathbf{R}}
\nc{\bN}{\mathbf{N}}
\nc{\by}{\mathbf{y}}
\nc{\sy}{y}
\nc{\sx}{x}

\nc{\MO}{\mathcal O}
\nc{\MU}{\mathcal{U}}
\nc{\ME}{\mathcal{E}}
\nc{\MN}{\mathcal{N}}
\nc{\MC}{\mathcal{C}}
\nc{\MK}{\mathcal{K}}
\nc{\MS}{\mathcal{S}}
\nc{\MT}{\mathcal{T}}
\nc{\BF}{\mathbb F}
\nc{\BQ}{\mathbb Q}
\nc{\MX}{\mathcal{X}}
\nc{\MI}{\mathcal{I}}
\nc{\MA}{\mathcal{A}}
\nc{\MD}{\mathcal{D}}
\nc{\MB}{\mathcal{B}}
\nc{\MZ}{\mathcal{Z}}
\nc{\MW}{\mathcal{W}}
\nc{\MY}{\mathcal{Y}}
\nc{\BZ}{\mathbb Z}
\nc{\BN}{\mathbb N}
\nc{\ep}{\epsilon}
\nc{\BH}{\mathbb H}
\nc{\BG}{\mathbb{G}}
\nc{\MF}{\mathcal{F}}
\nc{\MFabs}{\MF_{{\rm abs}}}
\nc{\ML}{\mathcal{L}}
\nc{\MH}{\mathcal{H}}
\nc{\One}{\mathbbm{1}}
\nc{\bOne}{\mathbf{1}}
\nc{\cumls}{cumulative loss\xspace}

\nc{\SP}{\mathsf P}
\nc{\SQ}{\mathsf Q}

\nc{\DO}{\accentset{\circ}{\D}}
\nc{\mf}{\mathfrak}
\nc{\mfp}{\mathfrak{p}}
\nc{\mfq}{\mf{q}}
\nc{\Sp}{\mbox{Spec}}
\nc{\Spm}{\mbox{Specm}}
\nc{\hookuparrow}{\mathrel{\rotatebox[origin=c]{90}{$\hookrightarrow$}}}
\nc{\hookdownarrow}{\mathrel{\rotatebox[origin=c]{-90}{$\hookrightarrow$}}}
\nc{\hra}{\hookrightarrow}
\nc{\tra}{\twoheadrightarrow}
\nc{\sgn}{{\rm sgn}}
\nc{\aut}{{\rm Aut}}
\nc{\Hom}{{\rm Hom}}
\nc{\img}{{\rm Im}}
\DMO{\id}{Id}
\DMO{\supp}{supp}
\DMO{\KL}{KL}
\DMO{\BSS}{BSS}
\DMO{\BES}{BES}
\DMO{\BGS}{BGS}
\DMO{\poly}{poly}
\nc{\indep}{\perp}

\nc{\p}{\mathbb{P}}
\renewcommand{\Pr}{\p}
\nc{\E}{\mathbb{E}}
\nc{\ra}{\rightarrow}
\renewcommand{\t}{\top}

\title{Fast Rates for Nonparametric Online Learning:\\ From Realizability to Learning in Games}

\author{Constantinos Daskalakis\thanks{MIT CSAIL. \url{costis@csail.mit.edu}. Supported by NSF Awards CCF-1901292, DMS-2022448 and DMS-2134108, by a Simons Investigator Award, by the Simons Collaboration on the Theory of Algorithmic Fairness, by a DSTA grant, and by the DOE PhILMs project (No. DE-AC05-76RL01830).} \and Noah Golowich\thanks{MIT CSAIL. \url{nzg@mit.edu}. Supported by a Fannie \& John Hertz Foundation Fellowship and an NSF Graduate Fellowship.}}
\date{\today}

\begin{document}
\maketitle
\thispagestyle{empty}

\begin{abstract}
We study fast rates of convergence in the setting of nonparametric online regression, namely where regret is defined with respect to an arbitrary function class which has bounded complexity. Our contributions are two-fold:
\begin{itemize}
\item In the realizable setting of nonparametric online regression with the absolute loss, we propose a randomized proper learning algorithm which gets a near-optimal \cumls in terms of the sequential fat-shattering dimension of the hypothesis class. In the setting of online classification with a class of Littlestone dimension $d$, our bound reduces to $d \cdot \poly \log T$. This result answers a question as to whether proper learners could achieve near-optimal \cumls; previously, even for online classification, the best known \cumls was $\tilde O( \sqrt{dT})$. Further, for the real-valued (regression) setting, a \cumls bound with near-optimal scaling on sequential fat-shattering dimension was not even known for \emph{improper} learners, prior to this work.
\item Using the above result, we exhibit an independent learning algorithm for general-sum binary games of Littlestone dimension $d$, for which each player achieves regret $\tilde O(d^{3/4} \cdot T^{1/4})$. This result generalizes analogous results of Syrgkanis et al.~(2015) who showed that in finite games the optimal regret can be accelerated from $O(\sqrt{T})$ in the adversarial setting to $O(T^{1/4})$ in the game setting.
  \end{itemize}
To establish the above results, we introduce several new techniques, including: a hierarchical aggregation rule to achieve the optimal \cumls for real-valued classes, a multi-scale extension of the proper online realizable learner of Hanneke et al.~(2021), an approach to show that the output of such nonparametric learning algorithms is stable, and a proof that the minimax theorem holds in all online learnable games.
\end{abstract}

\newpage

\tableofcontents
\thispagestyle{empty}
\newpage
\setcounter{page}{1}

\section{Introduction}

The success of deep learning has increased the importance of studying the learnability of nonparametric and high-dimensional models across all  areas within learning theory and its applications. In this paper our goal is to advance our understanding of learning such models in two prominent settings, online learning and games.

In the classical setting of online learning \cite{cesa-bianchi_prediction_2006,shalev-shwartz_online_2011}, a learner observes a  sequence of labeled examples $(x_t, y_t)$, generated adaptively by an adversary, and, at each round $t \geq 1$, is asked to make a prediction $f_t(x_t)$ about the true label $y_t$, by choosing a \emph{hypothesis} $f_t$ that depends only on the history of previous examples. A common goal is to minimize \emph{regret}: for a loss function $\ell(\hat y,y)$ giving the penalty for predicting $\hat y$ when the true label is $y$, and a class $\MF$ of hypotheses, the regret is the difference between the learner's total prediction loss, $\sum_{t=1}^T \ell(f_t(x_t),y_t)$, and the best possible loss in hindsight the learner could have obtained by choosing a single $f^\st \in \MF$ over all $T$ rounds. %

Online learning has been studied for various instantiations of $\MF$ and $\ell$ as well as various constraints on the learner and the adversary, %
drawing its importance from its versatility and intimate connections to other learning settings. Indeed, given the adversarial nature of the sequence of examples $(x_t,y_t)$, online learning generalizes supervised learning, where these pairs are i.i.d., while beautiful connections have been forged between online learning and private learning \cite{bun_equivalence_2020,alon_private_2019}, contextual bandits \cite{foster_beyond_2020}, reinforcement learning \cite{dong_provable_2021}, adversarial sampling \cite{alon_adversarial_2021}, learning of quantum states \cite{aaronson_online_2019}, and learning in games \cite{rakhlin_optimization_2013,syrgkanis_fast_2015}. %
In particular, online learning is a central primitive whose study  unlocks understanding in many other learning-theoretic settings.

The starting point for our work is that, while the optimal no-regret algorithms are very well understood when the hypothesis class ${\cal F}$ is finite, low-dimensional, or parametric,
our understanding of the optimal regret bounds and the algorithms achieving them is much more limited for nonparametric %
classes. 
For example, while a celebrated paper by Littlestone \cite{littlestone_learning_1988} determines the optimal regret bound of online classification in the \emph{realizable setting}, namely when $f^\st$ achieves 0 loss, his algorithm is not \emph{proper}, namely the hypotheses $f_t$ may not belong to the model class $\MF$ (see also \cite{hanneke_online_2021,angluin_queries_1988}); the optimal regret for proper realizable learning (by a randomized algorithm) remains elusive. In the non-Boolean setting (i.e., of regression) much less is known.

The contributions of our work, overviewed in the next section, are two-fold. First, we answer several outstanding questions, obtaining near-optimal regret bounds for proper online  learning (for both classification and regression) in the realizable setting. %
Second, we use our new results to advance our understanding of learning in games in the nonparametric setting, which has become increasingly important due to the applications of adversarial training in robust learning, generative adversarial networks, and multi-agent reinforcement learning. Here, our results are the first to obtain fast rates for regret of independent learning in two-player zero-sum nonparametric games and more generally in multi-player general-sum nonparametric games. %

\subsection{Model and overview of results}
\label{sec:model-results}
We consider the standard setting of online learning with absolute loss: two agents, a \emph{learner} and an \emph{adversary}, interact over a total of $T$ rounds, for some $T \in \BN$. The learner and adversary are given at the onset a set $\MX$ and a set $\MF$ consisting of $[0,1]$-valued functions on $\MX$, known as \emph{hypotheses}. In the setting of \emph{proper online learning}, the players perform the following for each round $1 \leq t \leq T$: the learner chooses a hypothesis $f_t \in \MF$ (which may be random), and the adversary picks $(x_t, y_t) \in \MX \times [0,1]$ (which may be random) denoting a \emph{feature} $x_t$ together with its \emph{label} $y_t$. Then the example $(x_t, y_t)$ is revealed to the learner, who suffers loss $|f_t(x_t) - y_t|$. We allow the adversary to be \emph{adaptive}, meaning that it can choose each example $(x_t, y_t)$ based on the history of moves $f_1, \ldots, f_{t-1}$ and $(x_1, y_1), \ldots, (x_{t-1}, y_{t-1})$. 
In general, the goal of the learner is to minimize its expected \emph{regret}, namely
\begin{align}
\Reg_T := \E \left[\sum_{t=1}^T |f_t(x_t) - y_t| - \inf_{f \in \MF} \sum_{t=1}^T |f(x_t) - y_t|\right] \label{eq:regret}. 
\end{align}
In this paper we are concerned with the case when the function class $\MF$ is \emph{nonparametric} in nature, meaning that it is infinite or extremeley large; thus regret guarantees in terms of $\log |\MF|$ are insufficient, and we instead aim for guarantees in terms of combinatorial complexity measures of $\MF$.\footnote{Note that our informal usage of the term ``nonparametric'' differs slightly from some other instances in the literature, such as \cite{rakhlin_empirical_2017,foster_contextual_2018}, in which it is used to refer specifically to classes with inverse-polynomial growth in the empirical entropy numbers. While we consider, for example, generic binary classes of finite Littlestone dimension to be nonparametric, works such as \cite{rakhlin_empirical_2017,foster_contextual_2018} would not do so.}  %
\paragraph{Near-optimal \cumls for the realizable setting.} A fundamental setting in which online learning is studied is the \emph{realizable setting}, which means that the adversary is constrained to choose the sequence $(x_t, y_t)$, $1 \leq t \leq T$ so that there is some $f^\st \in \MF$ so that $f^\st(x_t) = y_t$ for all $t$. In this case, the regret of the learner (\ref{eq:regret}) reduces to $\E\left[\sum_{t=1}^T |f_t(x_t) - y_t|\right]$, namely the total expected error made by the learner over all $T$ rounds; we call this quantity the \emph{\cumls} of the learner.

The study of the optimal \cumls for an online learner dates back to the seminal work of Littlestone \cite{littlestone_learning_1988}, who showed that in the case of \emph{online classification} (namely, where hypotheses $f \in \MF$ map to $\{0,1\}$, and $y_t \in \{0,1\}$ for all $t$), the optimal \cumls (also known as the \emph{mistake bound} in the binary setting) for a hypothesis class $\MF$ is given by a combinatorial parameter of $\MF$ known as the \emph{Littlestone dimension}, denoted $\Ldim(\MF) \in \BN$. One limitation of the result of \cite{littlestone_learning_1988} is that this mistake bound was only shown for an \emph{improper} learner, meaning that the learner's hypotheses $f_t$ may not belong to the class $\MF$. In many settings, such as the setting of learning in games discussed below, %
 an improper learning algorithm is insufficient
to solve the task at hand: for instance, for learning in games, the hypothesis class $\MF$ denotes the set of actions available to the learning agent, who must choose a valid action at each time step. However, there are many settings in which improper learners have better statistical or computational properties than proper learners (such as \cite{hazan_classification_2015,hazan_non-generative_2016,hanneke_optimal_2016,hazan_computational_2016,foster_logistic_2018,angluin_queries_1988,daniely_optimal_2014,montasser_vc_2019}; see \cite{hanneke_online_2021} for further examples). %
It is therefore natural to ask whether there is a similar proper-improper gap in the setting of online realizable classification: \emph{is there a near-optimal (randomized) proper learner for online classification in the realizable setting}?

We further consider the generalization of the above question to the setting of \emph{regression}, i.e., the general case where hypotheses $f \in \MF$ and labels $y_t$ are real-valued. In this case the natural generalization of the Littlestone dimension is the \emph{sequential fat-shattering dimension} (Definition \ref{def:sfat}). Surprisingly, prior work has not characterized the optimal \cumls for real-valued hypothesis classes in terms of the sequential fat-shattering dimension, even for improper learners. Therefore, our fully general question for the realizable setting is the following:
\begin{equation}
  \tag{$\star$}\label{eq:q-rlz}
  \parbox{15cm}{\centering\text{\it What is the optimal \cumls (in terms of sequential fat-shattering dimension) }\\ \text{\it for realizable online regression? Can it be achieved by a (randomized) proper algorithm?}}
\end{equation}
There appears to be some confusion in the literature regarding the latter part (proper learnability) of the question (\ref{eq:q-rlz}), even for the special case of binary classification: \cite{hanneke_online_2021} states (without proof) that ``unlike the realizable setting, in the agnostic setting nearly optimal randomized proper learners can exist.''\footnote{See the bottom of page 4 in \cite{hanneke_online_2021}.}  %
We show that a randomized proper learner can obtain a cumulative loss bound in the realizable setting that is off from the optimal bound (of $\Ldim(\MF)$) by only a $\poly \log T$ factor.\footnote{Whether such a $\poly \log T$ factor can be removed remains an open question.} Furthermore, we can extend our upper bound to the more general setting of online regression:
\begin{theorem}[Informal version of Theorem \ref{thm:proper-stable-main}] %
  \label{thm:proper-rlz-informal}
There is a randomized proper learner that achieves \cumls of $O \left(\inf_{\alpha \in [0,1]} \left\{ \alpha T + \int_{\alpha}^1 \sfat_\delta(\MF) d\delta \right\}\right) \cdot \poly \log T$ in the realizable setting. In the special case of online classification, this bound becomes $O(\Ldim(\MF)) \cdot \poly \log T$. 
\end{theorem}
We remark that randomization is necessary for proper realizable learning: there are trivial classes, such as the class of point functions on an infinite domain, which have Littlestone dimension 1 but for which any deterministic proper learner cannot achieve any finite \cumls bound. Nevertheless, we show in Proposition \ref{prop:reg-realizable} that there is a deterministic \emph{improper} learner that achieves the \cumls bound of Theorem \ref{thm:proper-rlz-informal}.

As alluded to above, we further show that the \cumls bound of Theorem \ref{thm:proper-rlz-informal} is optimal (up to a $\poly \log T$ factor) among any bound that depends only on sequential fat-shattering dimension:
\begin{proposition}[Lower bound]
  \label{prop:sfat-lb}
  For any non-increasing function $s : [0,1] \ra \BZ_{\geq 0}$ and $T \in \BN$, there is some function class $\MF$ so that $\sfat_\alpha(\MF) \leq s(\alpha)$ for all $\alpha \in [0,1]$, but for which any algorithm (not necessarily proper) has \cumls at least
  \begin{align}
\Omega \left( \frac{1}{\log T} \cdot \inf_{\alpha \in [1/T,1]} \left\{ \alpha T + \int_\alpha^1 s(\eta) d\eta \right\} \right)\label{eq:sfat-lower-bound}
  \end{align}
\end{proposition}
We remark that unlike in the case of classification, the matching bound of Theorem \ref{thm:proper-rlz-informal} and Proposition \ref{prop:sfat-lb} for online regression is not \emph{instance optimal}: there may be some classes $\MF$ for which an algorithm can achieve a \cumls much smaller than (\ref{eq:sfat-lower-bound}).\footnote{This could be the case, for instance, if for each $x \in \MX$, the values $\{ f(x) : f \in \MF \}$ are all distinct. Thus, after seeing a single example $(x_1, f^\st(x_1))$, the algorithm knows the identity of $f^\st$ and can predict correctly at all rounds $t > 1$.} We leave the question of determining a quantity that characterizes the optimal \cumls in an instance-dependent manner to future work. Nevertheless, we believe that the bound of Theorem \ref{thm:proper-rlz-informal} is of interest for the following reasons: first, in \cite{block_majorizing_2021}, under a mild growth condition, sequential fat-shattering dimension is shown to characterize the minimax regret in the agnostic (non-realizable) setting, meaning it is natural to ask what its relationship to the optimal \cumls is in the related realizable setting; second, the result of Theorem \ref{thm:proper-rlz-informal}, for the real-valued setting (regression), is a crucial component in the proof of our result for learning in games (Theorem \ref{thm:games-informal} below), \emph{even for binary-valued games} (at a high level, this is the case because players can randomize their actions). 

\paragraph{A stable proper learner and applications.}
Next, we describe some applications of Theorem \ref{thm:proper-rlz-informal}, culminating in our result giving fast rates for learning in games in a nonparametric setting (Theorem \ref{thm:games-informal}). First, it is necessary to describe a strengthening of Theorem \ref{thm:proper-rlz-informal}, namely that the \cumls bound of Theorem \ref{thm:proper-rlz-informal} holds for a learner that produces \emph{stable} predictions. Traditionally, \emph{stability} of the predictions produced by an online learner, in the sense that the predictions do not change much from round to round, has been a hallmark of online learning algorithms. In the finite-dimensional setting, such stability is classically achieved via the use of an appropriate regularizer \cite{beck_mirror_2003,cesa-bianchi_prediction_2006}, but can also be obtained from the use of more unorthodox methods such as follow-the-perturbed-leader \cite{kalai_efficient_2005}. Further, such stability has inspired connections between online learning and other areas in learning theory, such as differentially private learning \cite{bun_equivalence_2020}, and the study of generalization in deep neural networks \cite{hardt_train_2015}.

Despite the recent growth of work on online learning in nonparametric settings, we are not aware of any results establishing stability of the predictions. Moreover, many of the techniques in nonparametric settings, such as the nonconstructive approach that proceeds via application of a minimax theorem together with symmetrization \cite{rakhlin_online_2015,rakhlin_sequential_2015}, seem fundamentally unable to establish such stability bounds (see Section \ref{sec:related}). Proposition \ref{prop:stability-rlz} below address this deficiency of existing work; to state the result, we introduce the following notation. For a hypothesis class $\MF$, we denote the set of finite-support distributions on $\MF$ by $\Delc(\MF)$; elements of $\Delc(\MF)$ will typically be denoted with bars, e.g., $\bar f \in \Delc(\MF)$. For $\bar f_1, \bar f_2 \in \Delc(\MF)$, let $\frac 12 \tvnorm{\bar f_1 - \bar f_2}$ denote the total variation distance between $\bar f_1, \bar f_2$, which is well-defined by the finite-supportedness of $\bar f_1, \bar f_2$. We remark that the proper randomized learner of Theorem \ref{thm:proper-rlz-informal} outputs, for each round $t$, a hypothesis distributed according to a finite-support distribution $\bar f_t \in \Delc(\MF)$.
\begin{proposition}[Stability; informal version of Theorem \ref{thm:proper-stable-main}]
  \label{prop:stability-rlz}
Fix any $\eta > 0$. The proper randomized learner of Theorem \ref{thm:proper-rlz-informal}, which chooses $\bar f_t \in \Delc(\MF)$ for each round $t$, may be modified to satisfy $\tvnorm{\bar f_t - \bar f_{t+1}} \leq \eta$ for all $t$, at the cost of a \cumls of $\frac{\poly \log T}{\eta} \cdot O \left(\inf_{\alpha \in [0,1]} \left\{ \alpha T + \int_{\alpha}^1 \sfat_\delta(\MF) d\delta \right\}\right)$. %
\end{proposition}
\noah{mention finite games as special case?} \noah{, it follows immediately from first-order \noah{cite fo} or second-order \noah{cite so} regret bounds that standard online algorithms such as multiplicative weights achieve logarithmic regret in the case of a finite hypothesis class $\MF$}

While Proposition \ref{prop:stability-rlz}, which applies only to the realizable setting, is of some interest in its own right, we believe it is most notable for its applications: broadly speaking, we use Proposition \ref{prop:stability-rlz} to establish that many guarantees of online learning in the finite-dimensional \emph{non-realizable} (i.e., agnostic) setting that make use of stability extend to the nonparametric case as well. %

\paragraph{Application: fast rates for learning in games.}
An extensive line of work over the last decade (starting with \cite{daskalakis_near-optimal_2011}; see Section \ref{sec:related}) has shown that minimax $\Omega(\sqrt T)$ lower bounds on regret can be circumvented if multiple agents implement learning algorithms from a particular family in the context of repeatedly playing a (finite) game. In Theorem \ref{thm:games-informal} below, we show that such results hold true in the nonparametric setting as well. To state Theorem \ref{thm:games-informal}, we introduce the following preliminaries: we consider general-sum games with $K$ players, who have action sets $\MF_1, \ldots, \MF_K$. For simplicity, we restrict our attention to binary-valued games, namely where each player $k$'s payoff function is of the form $\ell_k : \MF_1 \times \cdots \times \MF_K \ra \{0,1\}$. %
We only assume that for each player $k$, the class $\lgame{\MF}{k}{\ell_k} := \{f_{-k} \mapsto \ell_k(f_k, f_{-k}) : f_k \in \MF_k\}$ has finite Littlestone dimension.\footnote{Some assumption on the game is necessary to guarantee existence of Nash equilibria: there is a 2-player zero-sum game, ``Guess the larger number'' (GTLN), which has infinite Littlestone dimension, but which has no $\ep$-approximate Nash equilibrium for any $\ep < 1$ (see \cite{hanneke_online_2021}). A necessary and sufficient condition for a binary-valued game and all its subgames to contain approximate Nash equilibria is that the game not contain an embedded copy of GTLN; we leave it as an interesting future direction to extend our results in some form to such games.} %

In the setting of \emph{independent learning algorithms for repeated game playing} \cite{daskalakis_near-optimal_2011,rakhlin_optimization_2013}, the following procedure occurs over $T$ rounds: for each round $t \leq T$, each player $k \in [K]$ plays a (finite-support) distribution over actions, denoted $\bar f_k^t \in \Delc(\MF_k)$. Then each player $k$ suffers loss $\E_{(f_1, \ldots, f_K) \sim (\bar f_1^t, \ldots, \bar f_K^t)} [ \ell_k(f_1, \ldots, f_K)]$. Further, each player $k$ observes the function mapping each of its actions $f_k \in \MF_k$ to its expected loss (under $\bar f_1^t, \ldots, \bar f_{k-1}^t, \bar f_{k+1}^t, \ldots, \bar f_K^t$) had it played $f_k$; it uses this information to adapt its play in future rounds. In the setting of independent learning algorithms in \emph{finite} normal form games, the foundational work of \cite{syrgkanis_fast_2015} showed that if each player implements the algorithm Optimistic Exponential Weights (also known as Optimistic Hedge), then each player $k$ can achieve regret $O (\log^{3/4}|\MF_k| \cdot \sqrt{K} \cdot T^{1/4})$. This result has since been improved multiple times, culminating in \cite{daskalakis_near-optimal_2021} which obtains regret $O(K \cdot \poly(\log |\MF_k|, \log T))$. Our main result for independent learning in games is an analogue of the result of \cite{syrgkanis_fast_2015} for the nonparametric setting:
\begin{theorem}[Informal version of Theorem \ref{thm:games-formal}]
  \label{thm:games-informal}
There is an independent learning algorithm (\OSE, Algorithm \ref{alg:ose}), so that the following holds.  Fix a game $G$ of finite Littlestone dimension, as above. 
 If the players repeatedly play $G$ with each player using \OSE, then each player $k$ suffers regret $\tilde O(\Ldim(\lgame{\MF}{k}{\ell_k})^{3/4} \cdot \sqrt{K} \cdot T^{1/4})$. 
\end{theorem}

\noah{\paragraph{Application: optimal regret with movement costs.}}

\subsection{Related work}
\label{sec:related}
The present paper lies at the confluence of many distinct lines of work on both statistical (i.i.d.) and adversarial (online) learning, as well as game theory, which we summarize below.

\paragraph{Fast rates in online \& offline learning.}
Our two main results, Theorems \ref{thm:proper-rlz-informal} and  \ref{thm:games-informal}, both beat a $\Omega(\sqrt T)$ lower bound on regret for many hypothesis classes of interest, such as classes of finite Littlestone dimension, by making additional assumptions about the adversary. A multitude of such results on \emph{fast rates} has been established, for both offline and online problems, over the past two decades. In finite-dimensional online settings, fast rates (in many cases, on the order of $\log T$) can be obtained if the loss function has special structure, such as if it is exp-concave \cite{cesa-bianchi_prediction_2006}, or more generally satisfies a \emph{mixability} \cite{vovk_game_1995,haussler_sequential_1998,vovk_competitive_2001} or \emph{stochastic mixability} \cite{van_erven_fast_2015} condition. Such results have been extended to the nonparametric setting for several special cases of exp-concave losses, including the square loss \cite{rakhlin_online_2014} and log loss \cite{rakhlin_sequential_2015,bilodeau_tight_2020,foster_logistic_2018}. Unlike our results, these works often allow for a (arbitrary) non-realizable adversary. In the case of a realizable adversary but for the harder case of absolute loss, the halving algorithm \cite{shalev-shwartz_online_2011} obtains logarithmic regret for finite classes $\MF$ in the case of binary classification; it is generalized by the Standard Optimal Algorithm (SOA) of \cite{littlestone_learning_1988} for the infinite case.

A similarly extensive line of work has pursued fast rates in the offline setting (i.e., where the examples $(x_t, y_t)$, $1 \leq t \leq T$, are i.i.d.~according to some distribution). It has long been known that in the realizable setting for binary classification, excess risk of $\tilde O(\VCdim(\MF)/T)$ is achievable\footnote{In the offline case, statistical rates are usually normalized by the number of samples $T$; we follow this convention, noting that in the online case we do not normalize by $T$.} by a proper learning algorithm \cite{vapnik_estimation_2006,blumer_learnability_1989}, such as empirical risk minimization. This bound has been improved by logarithmic factors several times \cite{haussler_predicting_1994,simon_almost_2015,hanneke_optimal_2016}. This work was generalized to the real-valued (regression) setting in \cite{mendelson_improving_2002}, in which an analogue of Theorem \ref{thm:proper-rlz-informal} for the offline realizable setting was established. \noah{Rather surprisingly, the statistical rates for the offline and online setting under realizability are different in their dependence on the growth of the relevant fat-shattering parameter.}
If the (non-sequential) $\alpha$-fat-shattering dimension grows as $\alpha^{-p}$, $p \in (0,2)$, the rate obtained by \cite[Theorem 4.1]{mendelson_improving_2002} for the \emph{offline} setting is $\tilde O(T^{-2/(2+p)})$, whereas if the sequential fat-shattering dimension grows as $\alpha^{-p}$, the (normalized) rate of Theorem \ref{thm:proper-rlz-informal} for the \emph{online} setting is $\tilde O(T^{-\min\{1,1/p\}})$. While Proposition  \ref{prop:sfat-lb} shows that the bound of Theorem \ref{thm:proper-rlz-informal} is best possible, it is unclear if this is the case for the offline rates of \cite{mendelson_improving_2002}; we note, though, that \cite{mendelson_improving_2002} conjectured that their rates were best-possible in the offline settting. \noah{Thus, though the online setting generalizes the offline setting, the stronger nature of sequential fat-shattering dimension leads to faster rates.} \noah{lower bound to mendelson?} %

\paragraph{Local Rademacher complexities and fast rates.}
Extending the techniques of \cite{mendelson_improving_2002}, several works in the offline setting introduced
\emph{local Rademacher complexities} \cite{koltchinskii_oracle_2011,koltchinskii_rademacher_2004,bousquet_local_2004,bartlett_local_2005,mendelson_learning_2014,rakhlin_empirical_2017}; these works derive fast rates, which are often data-dependent in nature, under a wider spectrum of assumptions generalizing realizability. In particular, the rates are generally phrased in terms of a fixed point of the modulus of continuity of the local Rademacher complexities around an optimal hypothesis. These results on local Rademacher complexities generalized and unified many previous papers (such as \cite{srebro_optimistic_2012}) which showed that, as in the online setting, fast rates are attainable in the offline setting under additional restrictions on the loss function, such as smoothness. %
As such, it would be of great interest to have a similarly powerful theory of local Rademacher complexities in the online setting. Initial steps toward this objective were made in \cite{rakhlin_relax_2012}, but the notion of local sequential Rademacher complexity from \cite{rakhlin_relax_2012} seems quite limited in nature, as it does not recover most of the existing nonparametric results on fast online rates mentioned above, as well as our own results. In the online setting, the effect of localization can be obtained in some special cases, such as learning with square loss, by using \emph{offset Rademacher complexities} \cite{rakhlin_online_2014,liang_learning_2015}; extending such techniques to our setting of realizability with absolute loss is an interesting open problem.

\paragraph{Fast rates for learning in games.}
A parallel line of work proving fast rates for regret of online 
learning assumes that the adversary for the online learning algorithm is itself the output of another 
online learner, in the context of repeated game-playing.
The seminal result in this direction was that of \cite{daskalakis_near-optimal_2011}, which described an algorithm for learning with $d$ experts that achieves the minimax regret of $O(\sqrt{T \log d})$ for a (worst-case) adversary, but which obtains regret $\poly(\log T, \log d)$ when it plays against itself in a two-player zero-sum game with $d$ actions per player. Similar results have since been shown for various other algorithms in two-player, zero-sum games \cite{hsieh_adaptive_2021,rakhlin_optimization_2013}. For the more challenging case of multi-player general-sum games, \cite{syrgkanis_fast_2015} showed that when all players use any algorithm from the family of Optimistic Mirror Descent (OMD) algorithms, each player has regret $O(T^{1/4} \cdot \log^{3/4} d)$. This was subsequently improved by \cite{chen_hedging_2020} who showed a regret bound of $O(T^{1/6} \cdot \log^{5/6} d)$ for each player when there are only 2 players and both use Optimistic Hedge (a special case of OMD), and then by \cite{daskalakis_near-optimal_2021} which obtained a near-optimal regret bound of $O(\log d \cdot \log^4 T)$ for any number of players under Optimistic Hedge. The techniques used to achieve these results have been successfully extended to achieve fast rates in various other settings, including learning in games with bandit feedback \cite{wei_more_2018,bubeck_improved_2019,wei_taking_2020} and learning in extensive-form games \cite{farina_optimistic_2019}. %
All existing works in this direction are parametric in nature, considering a finite expert (i.e., hypothesis) class. Our Theorem \ref{thm:games-informal} is the first to consider such results in the nonparametric setting. %

\paragraph{Relation to existing work on constrained adversaries.} Several of our results, such as Theorem \ref{thm:games-informal} and our stable path-length regret bound in Theorem \ref{thm:path-stable} (which is used to prove Theorem \ref{thm:games-informal}) may be seen as showing that minimax regret lower bounds can be broken if certain constraints are placed on the adversary. The work \cite{rakhlin_online_2011} %
develops a version of sequential Rademacher complexity to characterize the optimal rates for online learning with constrained adversaries in a general setting. It is shown in \cite[Proposition 13]{rakhlin_online_2011} that this generic technique recovers the path-length regret bound of \cite{syrgkanis_fast_2015} for learning with $d$ experts. This result for finite hypothesis classes is \emph{not} extended in \cite{rakhlin_online_2011} to the more general nonparametric setting of Theorem \ref{thm:path-stable}, though such an extension appears possible in principle, if an appropriate Bernstein-type uniform convergence lemma for trees could be shown. %
However, there is a more significant limitation of the framework of  \cite{rakhlin_online_2011}, which is that this framework is nonconstructive in nature and thus the implied learning algorithm is not shown to be stable in the sense of Proposition \ref{prop:stability-rlz}. If the learner is not stable, then if used in a game, it does not produce stable losses for the other agents and thus it is impossible to use the framework of \cite{rakhlin_online_2011} to derive Theorem \ref{thm:games-informal}.\footnote{One might hope that the constructive relaxation-based approach of \cite{rakhlin_relax_2012} would allow one to use the framework of \cite{rakhlin_online_2011} to produce stable learners. While this is the case in some finite-dimensional settings (see Section 10 of \cite{rakhlin_relax_2012}), this strategy fails in the general nonparametric setting since the basic Meta-algorithm of \cite{rakhlin_relax_2012} requires the computation of a fixed point each iteration, which may be non-stable. Our analysis runs into a similar challenge involving a per-round fixed-point operation, but we are able to overcome it using the particular structure of our algorithm, and this technique does not appear to extend to the setting of \cite{rakhlin_relax_2012}; see Section \ref{sec:techniques}.}

\noah{Stuff to include:
  \begin{itemize}
  \item Stability overview?
  \item Learning with movement costs (if we include it)
  \end{itemize}
}

\section{Preliminaries}
\paragraph{Notation} We use the following generic notation. 
Given a sequence $X_1, \ldots, X_n$ (e.g., of sets, or elements of sets), we will denote it by $X_{1:n}$. For a set $\MS$, let $\Delc(\MS)$ denote the set of finite support measures on $\MS$. For $n \in \BN$, let $[n] = \{1, 2, \ldots, n\}$. For integers $\lam \in \BZ$, we set $\alpha_\lam := 2^{-\lam}$ to denote the various scales our algorithms will operate at; typically (but not always) we will have $\lam \geq 1$. For a function $f : \MX \ra \BR$, let $\infnorms{f}{\MX} = \sup_{x \in \MX} |f(x)|$, and for finite-support distributions $\bar f_1, \bar f_2 \in \Delc(\MF)$, let $\tvnorm{\bar f_1 - \bar f_2}$ denote twice their total variation distance.

\subsection{Online learning: combinatorial quantities}
We introduce some notation and definitions regarding the setting of online nonparametric regression. Consider sets $\MX, \MY$ and let $\MY^\MX$ denote the set of $\MY$-valued functions on $\MX$; usually we will have either $\MY = \{0,1\}$ or $\MY = [0,1]$. %
Throughout the paper we will assume that $\MF \subset [0,1]^\MX$ is a known hypothesis class. We first define the sequential fat-shattering dimension, which is a combinatorial quantity that characterizes online learnability in the real-valued setting. To do so, we review some notation (from \cite{rakhlin_online_2015}) regarding binary trees. For a set $\MZ$, a \emph{$\MZ$-valued tree $\bz$} is a complete rooted binary tree each of whose nodes are labeled by an element of $\MZ$. Let $d$ denote the depth of the tree. For each $1 \leq t \leq d$, we identify the $2^{t-1}$ nodes of $\bz$ at depth $t$ with the sequences  $\ep_{1:t-1} = (\ep_1, \ldots, \ep_{t-1}) \in \{-1,1\}^{t-1}$; the value $\ep_i$ determines whether one must take the left or right child at the $i$th step on the path from the root of $\bz$ to the given vertex. We denote the label of the vertex $\ep_{1:t-1}$ by $\bz_t(\ep_{1:t-1})$, so that $\bz_t$ is a function mapping $\{-1,1\}^{t-1} \ra \MZ$. The data of the tree $\bz$ consists of the $d$-tuple $(\bz_1, \ldots, \bz_d)$. Nodes at the final level, namely of the form $(\ep_1, \ldots, \ep_d)$, are called \emph{leaves}.
\begin{defn}[Sequential fat-shattering dimension]
  \label{def:sfat}
For a class $\MF \subset [0,1]^\MX$ and $\alpha > 0$, its \emph{$\alpha$-sequential fat-shattering dimension}, denoted $\sfat_\alpha(\MF)$, is the largest positive integer $d$ so that there is a complete $\MX$-valued binary tree $\bx$ and a complete $[0,1]$-valued binary tree $\bs$, both of depth $d$, so that for all $k_{1:d} \in \{-1,1\}^d$, there is some $f \in \MF$ so that $k_t \cdot (f(\bx_t(k_{1:t-1})) - \bs_t(k_{1:t-1})) \geq \alpha / 2$ for all $t \in [d]$. In such a case, the class $\MF$ is said to \emph{$\alpha$-shatter} the tree $\bx$, as \emph{witnessed} by $\bs$.
\end{defn}
To work with the sequential fat-shattering dimension at a given scale $\alpha$, it is often useful to discretize the class $\MF$ in the sense given in the below definition:
\begin{defn}[Scale-sensitive restrictions]
  \label{def:ss-res}
  Fix a class $\MF \subset [0,1]^\MX$ and $\alpha \in (0,1)$. Fix $(x,y) \in \MX \times [0,1]$. We define the \emph{$\alpha$-restriction of $\MF$ to $(x,y)$}, denoted $\MF|^\alpha_{(x,y)}$, to be the set:
  $$
\MF|^\alpha_{(x,y)} := \left\{ f \in \MF : \lfloor y / \alpha \rfloor = \lfloor f(x) / \alpha \rfloor \right\}.
$$
Equivalently, we have that $\MF|^\alpha_{(x,y)} = \{ f \in \MF : f(x) \in [j\alpha, (j+1)\alpha) \}$, where $j = \lfloor y/\alpha \rfloor$.
\end{defn}

In the case where $\MF$ is $\{0,1\}$-valued (in which the sequential fat-shattering dimension reduces to the Littlestone dimension), the well-known \emph{standard optimal algorithm (SOA)} \cite{littlestone_learning_1988} gives the optimal \cumls (i.e., mistake bound) in the realizable setting. The SOA is an improper learning algorithm, but the hypotheses it outputs nevertheless have a certain structure which will prove useful in our setting as well; Definition \ref{def:soa-hypothesis} below generalizes such ``SOA hypotheses'' to the real-valued setting. 
\begin{defn}[SOA hypothesis]
  \label{def:soa-hypothesis}
  Fix a class $\MF \subset [0,1]^\MX$ and a parameter $\alpha \in (0,1)$. For each $x \in \MX$, set $s_x := \sfat_\alpha(\MF)$, and for $0 \leq j <\lfloor 1/\alpha \rfloor + 1$, set
  $
s_{x,j} := \sfat_\alpha \left( \MF|_{(x,j\alpha)}^\alpha \right). %
$

The \emph{SOA hypothesis for $\MF$ at scale $\alpha$}, denoted $\soalf{\MF}{\alpha} \in [0,1]^\MX$, is defined as follows. Fix any $x\in \MX$, and let $j^\st$ to be chosen as small as possible so that $s_{x,j^\st} \geq s_{x,j}$ for all $0 \leq j < \lfloor 1/\alpha \rfloor + 1$. Then set $\soal{\MF}{\alpha}{x} := (j^\st+1)\alpha$.
\end{defn}
Roughly speaking, the SOA hypothesis discretizes $\MF$ using the scale parameter $\alpha$ and maps each $x$ into the bucket such that the restricted class has maximum sequential fat-shattering dimension. 
Finally we introduce the notion of dual classes: for $\MF \subset [0,1]^\MX$, its \emph{dual class}, denoted $\MF^\st \subset [0,1]^\MF$, is the class $\{ f \mapsto f(x) \ : \ x \in \MX\}$; thus $\MF^\st$ is in bijection with $\MX$. 

\paragraph{The binary case: Littlestone classes.} Finally, we mention the specialization of the above concepts to the binary-valued case, namely when $\MF \subset \{0,1\}^\MX$. Here, $\sfat_\alpha(\MF)$ is constant as a function of $\alpha \in [0,1]$, and this constant value is called the \emph{Littlestone dimension} of $\MF$, denoted $\Ldim(\MF)$. The scale parameter $\alpha$ in Definitions \ref{def:ss-res} and \ref{def:soa-hypothesis} is unnecessary, so restrictions are denoted $\MF|_{(x,y)}$ and the SOA hypothesis is denoted by $\soaf{\MF}$. 

\section{Overview of techniques}
\label{sec:techniques}
The main technical innovation in our paper is a stable proper learning algorithm in the realizable setting, \HPL (Algorithm \ref{alg:hpl}), which obtains the guarantees of Theorem \ref{thm:proper-rlz-informal} and Proposition \ref{prop:stability-rlz} (stated formally in Theorem \ref{thm:proper-stable-main}). As mentioned previously, the guarantee of Theorem \ref{thm:proper-rlz-informal} is new even for an improper learning algorithm in the setting of realizable regression. We therefore begin by describing a simple improper learning algorithm,  \HIL (Algorithm \ref{alg:hil}), which obtains the optimal \cumls of $O \left( \min_\alpha \left\{ \alpha T + \int_0^1  \sfat_\eta(\MF) d\eta \right\} \right)$, as stated in Proposition \ref{prop:reg-realizable}.

\subsection{A multi-scale improper learner}
The starting point for the algorithm \HIL is the following simple algorithm which generalizes the Standard Optimal Algorithm of \cite{littlestone_learning_1988}: fix some $\alpha > 0$ at the beginning of the learning procedure, and set $\MF^1 = \MF$. For each $t \geq 1$, predict the hypothesis $\soalf{\MF^t}{\alpha}$. After observing each example $(x_t, y_t)$, if it is the case that $|\soal{\MF^t}{\alpha}{x_t} - y_t| > \alpha$, then set $\MF^{t+1} \gets \MF^t|^\alpha_{(x_t, y_t)}$, and otherwise set $\MF^{t+1} \gets \MF^t$. It is straightforward to show (see Lemma \ref{lem:far-sfat-dec}) that if $|\soal{\MF}{\alpha}{x_t} - y_t| > \alpha$, then $\sfat_\alpha(\MF^t|^\alpha_{(x_t, y_t)}) < \sfat_\alpha(\MF^t)$, meaning that for each round $t$ at which this algorithm makes a mistake larger than $\alpha$, we have $\sfat_\alpha(\MF^{t+1}) < \sfat_\alpha(\MF^t)$. Thus the \cumls for the algorithm is at most $\alpha T + \sfat_\alpha(\MF)$. Even if we optimize over $\alpha$, thus obtaining a \cumls of $\min_{\alpha \in [0,1]} \{ \alpha T + \sfat_\alpha(\MF)\}$, we still do not get close to the optimal \cumls. For instance, if $\sfat_\alpha(\MF) =\Theta(\alpha^{-p})$ %
for some $p \in (0,1)$, then the bound of Proposition \ref{prop:reg-realizable} is constant in the horizon $T$, whereas $\min_\alpha \{ \alpha T + \alpha^{-p} \} =\Theta(T^{p/(1+p)})$.

The key to obtaining better rates is to understand how to aggregate the predictions of SOA hypotheses at multiple scales $\alpha$. This is similar in spirit to the technique of \emph{chaining} \cite{dudley_central_1978}, which can be used to bound excess risk with an integral of (empirical) entropies by constructing a multi-scale cover. In our setting, though, it is the actual \emph{predictions} of an algorithm which we wish to aggregate over multiple scales, and doing so appears to be quite different from chaining \emph{covers} at multiple scales. %

In \HIL, we address this challenge as follows: for an appropriate parameter $\Lambda \leq \log T$, we maintain a total of $\Lambda$ subclasses of $\MF$, denoted $\MF_1, \ldots, \MF_\Lambda$, at each round $t$. Letting $\alpha_\lam = 2^{-\lam}$ for each $\lam \in [\Lam]$, each subclass $\MF_\lam$ is updated in response to the examples $(x_t, y_t)$ as described above, for the scale $\alpha_\lam$. For each round $t$, and each possible point $x_t \in \MX$, the $\Lambda$ subclasses each produce a prediction, $\soal{\MF_\lam}{\alpha_\lam}{x_t} \in [0,1]$, for $\lam \in [\Lam]$. The main difficulty one faces is: \emph{which of these $\Lam$ options should be chosen as the algorithm's prediction for $x_t$?}

\HIL answers this question using a simple aggregation rule we call the \emph{hierarchical aggregation rule} (see Definitions \ref{def:hagg} and \ref{def:soa-chains}). For any given point $x_t$, this rule chooses a single prediction out of the $\Lam$ elements $g_\lam := \soal{\MF_\lam}{\alpha_\lam}{x_t}$, $1 \leq \lam \leq \Lam$, as follows: it chooses $g_{\bar \lam}$, where $\bar \lam \geq 1$ is as small as possible so that $|g_{\bar \lam} - g_{\bar \lam + 1} | > 2 \alpha_{\bar \lam}$ (if no such $\bar \lam < \Lam$ exists, set $\bar \lam = \Lam$). 
This aggregation rule satisfies the following key property (see Lemma \ref{lem:real-pot-dec}): fix any choice of the true label $y_t$ for the point $x_t$, and set $\delta_t := |y_t - \bar g_\lam|$ to be the algorithm's error. Then there is some $\lamp \in [\Lam]$ so that $\alpha_\lamp \geq \Omega(\delta_t)$ and $|\soal{\MF_\lamp}{\alpha_\lamp}{x_t} - y_t| > \alpha_\lamp$; the proof of this fact requires some delicate case-work. Thus, the potential function $\sum_{\lam=1}^\Lam \alpha_\lam \cdot \sfat_{\alpha_\lam}(\MF_\lam)$ decreases by $\Omega(\delta_t)$ at each round $t$, which allows us to bound the total error over all $T$ rounds by the integral $\alpha_\Lam T + \int_{\alpha_\Lam}^1 \sfat_\eta (\MF) d\eta$.

\subsection{Obtaining the optimal \cumls for a proper learner}
We proceed to describe our proper learner (\HPL, Algorithm \ref{alg:hpl}) which obtains the same \cumls (up to a $\poly \log T$ factor) as \HIL. At a high level, \HPL uses the constructive framework of \cite{hanneke_online_2021} to ``make proper'' our improper learning algorithm. However, the algorithm and its analysis is not merely a case of generalizing that of \cite{hanneke_online_2021}, which only treated the case of classification, to the real-valued (regression) setting. Rather, as mentioned in Section \ref{sec:model-results}, our proper learner improves quantitatively upon the state of the art even in the special case of classification: we manage to obtain a poly-logarithmic (in $T$) \cumls for general Littlestone classes, and the previously best known bound was $O(\sqrt T)$ \cite{ben-david_agnostic_2009,rakhlin_online_2015,hanneke_online_2021}.

Thus, we begin by describing how we can obtain an improved \cumls bound for a (randomized) proper learning algorithm for binary classification. At a high level, we build off the approach of \cite{hanneke_online_2021}: roughly speaking, this approach maintains a multiset $\MT$ of subclasses $\MF^i$ of $\MF$, each accompanied by a weight $w^i \geq 0$. At each iteration, it considers the distribution $Q$ over the hypotheses $\soaf{\MF^i}$ weighted according to the values $w^i$, and tries to find a finite-support distribution $\bar f$ over hypotheses in $\MF$, whose expectation is close to that of $Q$. If it can find such a distribution $\bar f \in \Delc(\MF)$, it uses $\bar f_t := \bar f$ as its output on the next iteration $t$. If such $\bar f$ does not exist, an application of the minimax theorem implies the existence of a sequence of elements $(x_j, y_j)$ in $\MX \times \{0,1\}$, such that, when we replace the $\MF^i$ by the restrictions $\MF^i|_{(x_j,y_j)}$ for all $i$ and $j$, a certain potential function of $\MT$ is decreased by an appreciable amount. This potential function can only decrease a bounded number of times, which implies that we must eventually come to a point at which a desired $\bar f$ can be found.

The main limitation of the above approach that precludes a $\poly\log T$ \cumls bound is the notion of closeness of the weighted average (improper) hypothesis $h$ to the randomized (proper) hypothesis $\bar f \in \Delc(\MF)$. In \cite{hanneke_online_2021}, a certain fixed scale $\alpha$ was chosen, and it was shown that we can find $\bar f$ so that for all $(x,y)$ satisfying $\E_{h \sim Q}[|h(x) - y|] < \alpha$, then $\E_{f \sim \bar f}[|\bar f(x) - y|] \leq O(\alpha)$. %
This approach leads to \cumls of $\alpha T + \tilde O(\Ldim(\MF)/\alpha)$, which is never less than $O(\sqrt{\Ldim(\MF) \cdot T})$. To improve upon this bound, we have to find $\bar f$ so that for \emph{all} scales $\alpha \in [1/T, 1]$, if $\E_{h \sim Q}[|h(x) - y|] < \alpha$, then $\E_{f \sim \bar f}[|f(x) - y|] \leq O(\alpha)$ (step \ref{it:ms-voting-2} of Algorithm \ref{alg:hpl}). In the case that there does not exist a $\bar f$ satisfying this \emph{stronger} condition, then when we apply the minimax theorem, we end with a sequence in $\MX \times \{0,1\}$ satisfying a \emph{weaker} condition (Lemma \ref{lem:apply-minmax}). Via a careful analysis of the potential function alluded to above, it turns out that this weaker condition is still sufficient to ensure a decrease in the potential (Lemma \ref{lem:2fail-w-ub}).

Furthermore, because of the multi-scale nature of this argument, our application of the minimax theorem is to a general real-valued function class, even in the case when $\MF$ is binary-valued. Of course, it is necessary to prove that the minimax theorem actually holds in such settings. We show that it is sufficient for our needs to establish that the minimax theorem holds in general for real-valued classes which are online learnable (i.e., have sequential fat-shattering dimension finite at all scales). This fact, in turn, is proven in Section \ref{sec:minimax}. 

\subsection{A multi-scale proper learner for regression}
The proof of Theorem \ref{thm:proper-rlz-informal} (obtained by \HPL, Algorithm \ref{alg:hpl}) follows, roughly speaking, by combining the hierarchical aggregation of SOA hypotheses (from the improper learner for realizable regression) with the insights from the previous section needed to obtain the optimal \cumls for binary classes. In particular, the weighted average hypothesis $h$ formed each round from the previous section is replaced by a weighted average of hierarchically aggregated SOA hypotheses in the sense of Definition \ref{def:soa-chains}; the SOA hypotheses to be aggregated are collected in a data structure we call a \emph{weighted subclass collection} (Definition \ref{def:wsc}). The resulting algorithm is ``doubly multi-scale'' in the following sense: we need to use multiple scales in the sense described in the previous paragraph to characterize the closeness of $h$ and $\bar f$, but we also need multiple scales to deal with the growth of $\sfat_\alpha(\MF)$ as $\alpha \ra 0$. This creates additional technical challenges; see Section \ref{sec:proper-realizable} for details.

\subsection{Making the proper learner stable} Next we address the stability property of \HPL, namely the proof of Proposition \ref{prop:stability-rlz}. We begin with the case of improper learning for binary classification, in which case the Standard Optimal Algorithm simply outputs the hypothesis $\soaf{\MF^t}$ at each round $t$, and updates $\MF^{t+1} \gets \MF^t|_{(x_t,y_t)}$ if it incorrectly predicts $(x_t, y_t)$ (and otherwise sets $\MF^{t+1} \gets \MF^t$). The key insight that allows a stable improper learner here is that $\MF^t$ is only updated in the event of a mistake\footnote{Some instantions of the Standard Optimal Algorithm restrict $\MF^{t+1} \gets \MF^t|_{(x_t, y_t)}$ even if there is not a mistake at step $t$, though this is not necessary.}, and there are only $\Ldim(\MF)$ mistakes overall. Thus, for any $\eta > 0$, if we instead output the uniform distribution over the past $1/\eta$ hypotheses, $\soaf{\MF^t}, \ldots, \soaf{\MF^{t-(1/\eta)}}$, each original mistake will incur at most $1/\eta$ new ones, leading to a \cumls of $\Ldim(\MF)/\eta$. Further, the total variation distance between consecutive averages of $1/\eta$ hypotheses is at most $2\eta$. Thus, for \emph{improper learning for classification}, we immediately obtain the guarantee of Proposition \ref{prop:stability-rlz}. To obtain the same \cumls for a proper learner (and in the regression setting), we essentially pass the above insight into the machinery described in the previous sections. In particular, we show that due to the fact that we only make updates to subclasses $\MF^i$ of the weighted subclass collection $\MT$ when $\MF^i$ makes a mistake, the collection $\MT$ changes slowly (Lemmas \ref{lem:single-level-wchange} and \ref{lem:vote-slow-change-2}), which allows us to show that averaging over a window of $1/\eta$ rounds only degrades the \cumls by a factor of $1/\eta$. 

\subsection{Application: fast rates for learning in games}
At last we can overview the proof of Theorem \ref{thm:games-informal}, which leans heavily on our stable proper learner (Theorem \ref{thm:proper-rlz-informal} and Proposition \ref{prop:stability-rlz}). The main technical component of Theorem \ref{thm:games-informal} is a \emph{path-length} regret bound for a stable proper learner (Theorem \ref{thm:path-stable}), which shows (for a stable learning algorithm) that if consecutive losses fed by the adversary are close, then we can obtain improved regret (i.e., beating $O(\sqrt T)$). 
At a high level, the idea of the proof of Theorem \ref{thm:path-stable} is to use the ``SOA-experts'' technique of \cite{ben-david_agnostic_2009,rakhlin_online_2015}\footnote{For the latter reference \cite{rakhlin_online_2015}, see in particular the version at \url{https://arxiv.org/pdf/1006.1138v1.pdf}.} which uses the existence of an online cover of bounded size for the hypothesis class $\MF$ for any data sequence $x_1, x_2, \ldots, x_T$.\footnote{See also the online version of the Sauer-Shelah lemma, \cite[Theorem 13.7]{rakhlin_statistical_2014}.} Each element of this online cover is interpreted as an expert, which runs an instance of our proper realizeable learner (\HPL). Typically one uses an online experts algorithm (such as exponential weights, i.e., Hedge) to learn the best expert in this cover. In order to obtain path-length regret bounds, we replace Hedge with Optimistic Hedge \cite{rakhlin_optimization_2013,syrgkanis_fast_2015} and use (as a black-box) the path-length regret bound of \cite{syrgkanis_fast_2015}. %
Crucially, the stability property of the output of \HPL (Proposition \ref{prop:stability-rlz}) implies that (a) the outputs of the experts produce slowly-changing losses for the Optimistic Hedge algorithm, which is necessary to get strong path-length regret bounds, and (b) the outputs of the Optimistic Hedge are therefore slowly changing, meaning that in the game setting, other agents' losses are slowly changing. One additional challenge that occurs in the proof is that because each agent is playing randomized strategies, the function class we must work with is that which takes as input a distribution over examples $\MX$, and thus is real-valued (even though we are in the setting of a binary game). In Lemma \ref{lem:sfat-mix}, we nevertheless show that its sequential fat-shattering dimension can be bounded in terms of the Littlestone dimension of the original binary-valued class, which allows us to use our results for proper realizable learning in the real-valued setting. The full proof may be found in Sections \ref{sec:stable-proper} and \ref{sec:games}.

\noah{
  \subsection{Application: optimal regret for nonparametric online learning with movement costs}
  }

\section{A near-optimal improper \cumls bound}
\label{sec:improper-realizable}
As a warm-up, we derive an optimal \cumls bound in the realizable setting for the easier case of improper learning, which remained open prior to this work. %
As noted previously, a \cumls bound of $O \left(\min_{\alpha \in [0,1]} \left\{ \alpha T + \sfat_\alpha(\MF) \right\}\right)$ is immediate from the definition of $\sfat_\alpha(\cdot)$, but this regret bound is suboptimal in many cases, for instance when the sequential fat-shattering dimension exhibits growth $\sfat_\alpha(\MF) \asymp \alpha^{-p}$ for some $p \in (0,1)$.

To improve upon this trivial bound, it is necessary to consider the sequential fat-shattering dimension at multiple scales $\alpha$, somewhat analogously to how chaining is used to improve statistical rates in the agnostic setting. Our techniques for doing so differ substantially from chaining since rather than considering covers at different scales, we consider different hypotheses at different scales. To aggregate the predictions of the hypotheses at varying scales, we introduce \emph{hierarchical aggregation rules} in Definition \ref{def:hagg} below. First, we define the scales we will consider: for $\lambda \in \BZ$, define $\alpha_\lambda := 2^{-\lambda}$. Throughout this section, we will fix some $\Lambda$ (which ultimately will depend on the growth of $\sfat_\alpha(\MF)$ as $\alpha \ra 0$) and consider scales $\alpha_1, \alpha_2, \ldots, \alpha_\Lambda$. 
\begin{defn}[Hierarchical aggregation]
\label{def:hagg}
For a sequence of real numbers $g_1, \ldots, g_\Lambda \in [0,1]$, we define the \emph{hierarchical aggregation rule} $\hagg{g_1, \ldots, g_\Lambda} \in [0,1]$ to be $g_{\bar \lambda}$, where $\bar \lambda$ is chosen so that for $2 \leq \lambda' \leq \bar \lambda$, it holds that $|g_{\lambda'} - g_{\lambda'-1}| \leq 2\alpha_{\lambda'-1}$, yet $|g_{\bar \lambda} - g_{\bar \lambda+1}| > 2 \alpha_{\bar \lambda}$, if such $\bar \lambda$ exists; if no such $\bar \lambda$ exists, set $\bar \lambda = \Lambda$. We will call this value of $\bar \lambda$ the \emph{cutoff point} and denote it by $\bar \lambda = \cutoff{g_1, \ldots, g_\Lambda}$.  
\end{defn}

The individual hypotheses referred to above (to which a hierarchical aggregation rule is applied) will be the SOA hypotheses at differing scales (Definition \ref{def:soa-hypothesis}), namely $\soalf{\MF_\lambda}{\alpha_\lambda}$, for various classes $\MF_\lambda$. We next define the \emph{SOA hypothesis for a sequence}:
\begin{defn}[SOA hypotheses for sequences]
  \label{def:soa-chains}
  Given a sequence $\MF_{1:\lambda} = (\MF_1, \ldots, \MF_\lambda)$ of hypothesis classes, define its \emph{SOA hypothesis}, denoted $\soaf{\MF_{1:\Lambda}}$, as
  $$
 \soa{\MF_{1:\Lambda}}{x} = \hagg{\soal{\MF_1}{\alpha_1}{x}, \soal{\MF_2}{\alpha_2}{x}, \ldots, \soal{\MF_\Lambda}{\alpha_\Lambda}{x}}.
 $$
 We also denote the \emph{cutoff point} for the sequence $(\soal{\MF_1}{\alpha_1}{x}, \ldots, \soal{\MF_\Lambda}{\alpha_\Lambda}{x})$ by
 $$
\cutoff{\MF_{1:\Lambda}, x} := \cutoff{\soal{\MF_1}{\alpha_1}{x}, \ldots, \soal{\MF_\Lambda}{\alpha_\Lambda}{x}}.
 $$
\end{defn}

\begin{algorithm}[!htp]
  \caption{\bf \HIL}\label{alg:hil}
  \KwIn{Function class $\MF \subset [0,1]^\MX$, time horizon $T \in \BN$, scale parameter $\Lambda \in \BN$.} 
  \begin{enumerate}[leftmargin=14pt,rightmargin=20pt,itemsep=1pt,topsep=1.5pt]
  \item For $1 \leq \lambda \leq \Lambda$, initialize $\MF_\lambda \gets \MF$. %
  \item For $1 \leq t \leq T$:
    \begin{enumerate}
    \item Observe  $x_t$, and predict $\hat y_t := \soa{\MF_{1:\Lambda}}{x_t}$.
    \item  Observe $y_t$, suffer loss $\delta_t := |y_t - \soa{\MF_{1:\Lambda}}{t}|$.
      \item Set $\lambda_t$, $1 \leq \lambda_t \leq \Lambda+1$, to be $\Lambda+1$ if $\delta_t \leq \alpha_\Lambda$, and otherwise as small as possible so that $\delta_t > \alpha_{\lambda_t}$.\label{it:def-ht}
      \item Set $\bar \lambda_t$, $1 \leq \bar \lambda_t \leq \Lambda$, to be the cutoff point $\bar \lambda_t := \cutoff{\MF_{1:\Lambda}, x_t}$. %
        \label{it:def-bar-ht}
    \item Update $\MF_{\lambda'} \gets \MF_{\lambda'} |^{\alpha_{\lambda'}}_{(x_t, y_t)}$ for all $\lambda' \geq \min\{\lambda_t, \bar \lambda_t+1 \}$ such that $\sfat_{\alpha_{\lambda'}}(\MF_{\lambda'}|^{\alpha_{\lambda'}}_{(x_t,y_t)}) < \sfat_{\alpha_{\lambda'}}(\MF_{\lambda'})$. %
    \end{enumerate}
  \end{enumerate}
\end{algorithm}

Algorithm \ref{alg:hil}, \HIL, presents an improper proper learner that uses the SOA hypothesis for sequences presented in Definition \ref{def:soa-chains}. The following proposition upper bounds the number of mistakes made by \HIL.
\begin{proposition}[Optimal \cumls bound for improper learning]
  \label{prop:reg-realizable}
  Suppose $(x_t, y_t) \in \MX \times [0,1]$ and $y_t = f^\st(x_t)$ for some $f^\st \in \MF$ for all $t \in [T]$. Then the predictions $\hat y_t$, $t \in [T]$ of \HIL (Algorithm \ref{alg:hil}) satisfy
  \begin{align}
\sum_{t=1}^T |\hat y_t - y_t| \leq C \cdot \inf_{\alpha \in [0,1]} \left\{ \alpha T + \int_\alpha^1 \sfat_\eta (\MF) d\eta\right\}.\label{eq:improper-integral}
  \end{align}
for some constant $C$.
\end{proposition}
\begin{proof}
  
  Choose $\alpha \in [0,1]$ which minimizes $\alpha T + \int_\alpha^1 \sfat_\eta(\MF) d\eta$; since we assume that $\sfat_c(\MF) \geq 1$ for a constant $c$, we can assume that $\alpha \geq 1/T$ with the loss of a constant factor. %
  Set $\Lambda = \lfloor \log 1/(2\alpha) \rfloor$. By bounding the integral in (\ref{eq:improper-integral}) below by the appropriate Riemann sum, it suffices to show that for some constant $C > 0$, we have
    $$
\sum_{t=1}^T |\hat y_t - y_t| \leq T \alpha_\Lambda  + C \sum_{\lambda=0}^\Lambda \sfat_{\alpha_\lambda}(\MF) \cdot \alpha_\lambda.
$$
  Define, for $1 \leq t \leq T+1$,
  $$
\Phi_t(\MF_{1:\Lambda}) := (T+1-t) \cdot \alpha_\Lambda + 16 \sum_{\lambda=1}^\Lambda \alpha_{\lambda} \cdot \sfat_{\alpha_\lambda}(\MF_\lambda).
$$
Below we will abbreviate $\Phi_t$ for the value $\Phi_t(\MF_{1:\Lambda})$, where $\MF_{1:\Lambda}$ is the sequence maintained by the algorithm at the beginning of round $t$. It is straightforward that $\Phi_1 = T\alpha_\Lam + 16 \sum_{\lambda=1}^\Lambda \alpha_\lambda \cdot \sfat_{\alpha_\lambda}(\MF)$. %
Moreover, $\Phi_t$ is non-negative for all $t \leq T+1$. We will show that $\Phi_t - \Phi_{t+1} \geq \delta_t$ for all $t$, which will imply the statement of the lemma.

Fix any $t \leq T$, and let $\MF_1, \ldots, \MF_\Lambda$ denote the subclasses maintained by \HIL at the beginning of round $t$. We apply Lemma \ref{lem:real-pot-dec} for the sequence $\MF_1, \ldots, \MF_\Lam$, $\delta = \delta_t$, and $(x,y) = (x_t, y_t)$. Note that the parameter $\lambda$ in the statement of Lemma \ref{lem:real-pot-dec} is $\lambda_t$, and $\cutoff{\MF_{1:\Lambda}, x} = \bar \lambda_t$. Lemma \ref{lem:real-pot-dec} then implies that at least one of the following holds:
\begin{itemize}
\item Either $\delta_t \leq \alpha_\Lam$, which implies that $\Phi_t - \Phi_{t+1} \geq \alpha_\Lam \geq \delta_t$, as desired; or
\item There is some $\lambda' \in [\Lam]$ satisfying $\lambda' \geq \min \{ \bar \lambda_t+1, \lambda_t \}$ so that $|\soal{\MF_{\lam'}}{\alpha_{\lam'}}{x_t} - y_t| > \alpha_\lamp \geq \delta_t/16$. By Lemma \ref{lem:far-sfat-dec}, it follows that $\sfat_{\alpha_\lamp}(\MF_{\lamp}|^{\alpha_\lamp}_{(x_t,y_t)}) < \sfat_{\alpha_\lamp}(\MF_{\lamp})$, which implies that $\Phi_t - \Phi_{t+1} \geq 16 \alpha_\lamp \geq \delta_t$, as desired. %
\end{itemize}
In both cases, we thus get a decrease in the potential of at least $\delta_t$, completing the proof of the proposition.
\end{proof}

Lemma \ref{lem:real-pot-dec} is the main technical lemma used in the proof of Proposition \ref{prop:reg-realizable}, used to show a decrease in the potential function therein.
\begin{lemma}
  \label{lem:real-pot-dec}
  Fix any $\Lambda \in \BN$, any sequence of subclasses $\MF_1, \ldots, \MF_\Lambda \subset \MF$, and consider any $(x,y) \in \MX \times [0,1]$. Set $\delta := |y  - \soa{\MF_{1:\Lambda}}{x}|$, and define $\lambda \in \{ 1, 2, \ldots, \Lambda \}$ to be $\Lambda + 1$ if $\delta \leq \alpha_\Lambda$, and otherwise as small as possible so that $\delta > \alpha_\lambda$. Then at least one of the below holds:
  \begin{itemize}
  \item $\delta \leq \alpha_\Lambda$ and $ \cutoff{\MF_{1:\Lambda}, x} = \Lambda$; or
  \item For some $\lambda'$ satisfying $\min\{ \cutoff{\MF_{1:\Lambda},x}+1, \lambda \} \leq \lambda ' \leq \min \{\cutoff{\MF_{1:\Lambda},x} + 1, \Lam\}$, we have $|\soal{\MF_{\lambda'}}{\alpha_{\lambda'}}{x} - y| > \alpha_{\lambda'} \geq \delta/16$. 
  \end{itemize}
\end{lemma}
\begin{proof}
  Set $\bar \lambda = \cutoff{\MF_{1:\Lambda},x}$.   Note that the choice of $\lambda$ in the statement of the lemma ensures that $\delta \leq 2 \alpha_{\lambda}$.     Further, by Definition \ref{def:soa-chains}, we have that $\soal{\MF_{\bar \lambda}}{\alpha_{\bar \lambda}}{x} = \soa{\MF_{1:\Lambda}}{x}$.

  We consider the following cases:
  \begin{itemize}
  \item First suppose that $\delta \leq \alpha_\Lam$ (i.e., $\lam = \Lam+1$). %
    If $\bar \lam = \Lam$, then we are done; otherwise, it must hold that $| \soal{\MF_{\bar \lam}}{\alpha_{\bar \lam}}{x} - \soal{\MF_{\bar \lam +1}}{\alpha_{\bar \lam+1}}{x} | > 2 \alpha_{\bar \lam}$.  But $\delta = |y - \soal{\MF_{\bar \lam}}{\alpha_{\bar \lam}}{x}| \leq \alpha_\Lam \leq \alpha_{\bar \lam}$, meaning that $|\soal{\MF_{\lamp}}{\alpha_{\lamp}}{x} - y| > \alpha_{\bar \lam} > \alpha_\lamp \geq \delta/2$ with $\lamp = \bar \lam + 1$, thus verifying the second item in the lemma's statement.
  \item In the next case, suppose that $\Lam \geq \lam \geq \bar \lam + 1$.  If it is not the case that $|\soal{\MF_{\bar \lambda}}{\alpha_{\bar \lambda}}{x} - \soal{\MF_{\bar \lambda +1}}{\alpha_{\bar \lambda+1}}{x}| > 2 \alpha_{\bar \lambda}$, then we must have $\bar \lambda = \Lambda$ and so $\lambda =\Lambda+1$, which contradicts our assumption of $\lam \leq \Lam$ in this case. %
    Otherwise, $|\soal{\MF_{\bar \lambda}}{\alpha_{\bar \lambda}}{x} - \soal{\MF_{\bar \lambda +1}}{\alpha_{\bar \lambda+1}}{x}| > 2 \alpha_{\bar \lambda}$ holds. Moreover we have
      $$
    |\soal{\MF_{\bar \lambda}}{\alpha_{\bar \lambda}}{x} - y|=  \delta \leq 2 \alpha_{\lambda} \leq 2 \alpha_{\bar \lambda+1} = \alpha_{\bar \lambda},$$

    Hence $|\soal{\MF_{\bar \lambda_{t}+1}}{\alpha_{\bar \lambda+1}}{x} - y| > \alpha_{\bar \lambda} > \alpha_{\bar \lambda+1} \geq \delta/2$, so in this case we may again choose $\lamp = \bar \lam+1$.  %
  \item In the final case, $\lambda \leq \bar \lambda$ (and the previous cases do not apply).
    Thus here $|y - \soal{\MF_{\bar \lambda}}{\alpha_{\bar \lambda}}{x}| = \delta > \alpha_{\lambda} \geq \alpha_{\bar \lambda}$. %
    We consider two sub-cases:
    \begin{itemize}
      \item In the event that $\lambda \geq \bar \lambda - 3$ (i.e., $\lambda \in \{ \bar \lambda-3,\bar \lambda-2, \bar \lambda-1, \bar \lambda \}$), we therefore have that
    $
\delta \leq 2 \alpha_{\lambda} \leq 16 \alpha_{\bar \lambda} ,
$
meaning that, with $\lamp = \bar \lam$, $|y - \soal{\MF_\lamp}{\alpha_\lamp}{x}| > \alpha_{\lamp} \geq \delta / 16$. %
\item In the other subcase, we have $\lambda < \bar \lambda - 3$; then we have
\begin{align}
  & |\soal{\MF_{\bar \lambda}}{\alpha_{\bar \lambda}}{x} - \soal{\MF_{\lambda+3}}{\alpha_{\lambda+3}}{x} | \nonumber\\
  \leq & \sum_{\lambda'=\lambda +3}^{\bar \lambda-1} | \soal{\MF_{\lambda'}}{\alpha_{\lambda'}}{x} - \soal{\MF_{\lambda'+1}}{\alpha_{\lambda'+1}}{x}| \nonumber\\
  \leq & \sum_{\lambda'=\lambda+3}^{\bar \lambda-1} 2 \alpha_{\lambda'} \nonumber\\
  \leq & 4 \alpha_{\lambda+3} = \alpha_{\lambda}/2\nonumber.
\end{align}
It follows that $|\soal{\MF_{\lambda+3}}{\alpha_{\lambda+3}}{x} - y| > \alpha_{\lambda}/2 > \alpha_{\lambda+3} \geq \delta/16$. %
\end{itemize}
\end{itemize}
\end{proof}

\section{A near-optimal proper \cumls bound}
\label{sec:proper-realizable}
Throughout this section, we consider a real-valued class $\MF \subset [0,1]^\MX$, so that $\sfat_\alpha(\MF) < \infty$ for all $\alpha > 0$. 
Having established a \cumls bound in Proposition \ref{prop:reg-realizable} for improper learning of $\MF$ in the realizable setting, we turn to the more challenging case of proper learning of $\MF$. In addition to the ideas on hierarchical aggregation used in the case of improper learner (Section \ref{sec:improper-realizable} above), a key tool we use is a generalization of the proper online realizable learner of \cite{hanneke_online_2021}. It is necessary, however, to make substantial modifications to the algorithm of \cite{hanneke_online_2021}: for one, it only applies in the setting of binary classification, but even in that setting it provides a (significantly) suboptimal \cumls of $\tilde O(\sqrt{\Ldim(\MF) \cdot T})$; our proper algorithm gives \cumls of $O(\Ldim(\MF) \cdot \poly \log(T))$. Thus, we introduce new techniques to correct both of these shortcomings. 

We begin by introducing some notation. Fix some scale parameter $\Lambda \in \BN$ (ultimately, $\Lambda$ will be chosen identically as in the proof of Proposition \ref{prop:reg-realizable}). We will consider scales $\alpha_\lambda = 2^{-\lam}$ for $\lam \in \BZ$. We will primarily be considering values of $\lambda$ in the set $[\Lam] = \{1, 2, \ldots, \Lam\}$ but occasionally will refer to $\alpha_\lam$ for other (integral) values of $\lam$. %
\begin{defn}[Weighted subclass collection]
  \label{def:wsc}
  A \emph{weighted subclass collection} $\MT$ is a tuple $\MT = (\MT_1, \ldots,\MT_\Lambda)$, where for each $\lambda \in [\Lambda]$, $\MT_\lambda$ is a multiset of tuples of the form
  $\MT_\lam = \{ (\MG_\lambda^1, w_\lambda^1), \ldots, (\MG_\lambda^{|\MT_\lambda|}, w_\lambda^{|\MT_\lambda|})\}$, where for each $1 \leq v_\lam \leq |\MT_\lam|$, %
  we have $w_\lambda^{v_\lambda} \geq 0$ and $\MG_\lambda^{v_\lambda} \subset \MF$. %

  We will use the letter $w$ to denote the collection of all $w_\lambda^{v_\lambda}$, for $\lambda \in [\Lambda]$ and $1 \leq v_\lambda \leq |\MT_\lambda|$, and the letter $\MG$ to denote the collection of all $\MG_\lam^{v_\lam}$, for $\lam \in [\Lam]$ and $1 \leq v_\lam \leq |\MT_\lam|$. We introduce the following notation to denote a weighted subclass collection $\MT$: we will write $\MT = \wsc{\MG}{w}$.\footnote{This notation emphasizes the use of the letters $\MG, w$ to denote the subclasses and weights, respectively, belonging to $\MT$. If we wish to describe another weighted subclass collection, we might notate it as $\MS = \wsc{\MH}{z}$, replacing the pairs $(\MG_\lam^{v_\lam}, w_\lam^{v_\lam})$ with the pairs $(\MH_\lam^{v_\lam}, z_\lam^{v_\lam})$.}
\end{defn}
In words, the weighted subclass collection $\MT$ denotes a collection of subclasses of $\MF$ together with non-negative weights for each scale $\lambda$; our algorithm will use a weighted aggregation of the SOA hypotheses of these subclasses, according to the weights $w_\lam^{v_\lam}$. For a weighted subclass collection $\MT$, let $\MN_\Lambda(\MT)$ denote the set of sequences $(v_1, \ldots, v_\Lambda)$, where for $\lambda \in [\Lambda]$, $1 \leq v_\lam \leq |\MT_\lam|$ (i.e.,  $(\MG_\lambda^{v_\lambda}, w^{v_\lambda}) \in \MT_\lambda$). %
We will abbreviate the sequence $(v_1, \ldots, v_\Lambda)$ with the letter $v$, so that we have $v \in \MN_\Lambda(\MT)$. We also abbreviate the sequence  $(\MG_1^{v_1}, \ldots, \MG_\Lambda^{v_\Lambda})$ as $ \MG_{1:\Lambda}^v $ and the sequence $(w_1^{v_1}, \ldots, w_\Lambda^{v_\Lambda})$ as $ w_{1:\Lambda}^v $.  For each $v \in \MN_\Lambda(\MT)$, we next define
$$
P_w(v) := \prod_{\lambda=1}^\Lambda \frac{w^{v_\lambda}_\lambda}{\sum_{u_\lambda = 1}^{|\MT_\lambda|} w^{u_\lambda}_\lambda}.
$$
It is evident from the above definition that $\sum_{v \in \MN_\Lambda(\MT)} P_w(v) = 1$.

For a sequence $ \MG_{1:\Lambda} $ and $(x,y) \in \MX \times [0,1]$, define its \emph{truncated error at the point $(x,y)$} as
$$
\Terr{ \MG_{1:\Lambda} , x, y} := \max \left\{ |\soa{ \MG_{1:\Lambda} }{x} - y |, \alpha_{\cutoff{ \MG_{1:\Lambda} , x}}\right\}.
$$
The intuition behind the truncated error is as follows: recall that $\soa{\MG_{1:\Lam}}{x} = \soal{\MG_{1:\Lam}}{\alpha_{\cutoff{\MG_{1:\Lam},x}}}{x}$, meaning that $\soa{\MG_{1:\Lam}}{x}$ is, in general, only accurate up to an additive $\alpha_{\cutoff{\MG_{1:\Lam},x}}$. Thus, if it happens that $| \soa{\MG_{1:\Lam}}{x} - y| \ll \alpha_{\cutoff{\MG_{1:\Lam},x}}$, then this is due to ``luck''; it turns out that in order to ensure that certain potential functions always decrease it is convenient to  still force us to pay $\alpha_{\cutoff{\MG_{1:\Lam},x}}$ in our error bounds when such ``lucky'' situations occur.

Now fix a weighted subclass collection $\MT = \wsc{\MG}{w}$; we will write 
$$
\Terr{\MT, x, y} := \sum_{v \in \MN_\Lam(\MT)} P_w(v) \cdot \Terr{\MG_{1:\Lam}^v, x, y}
$$
to denote the average truncated error of a element $\MG_{1:\Lam}^v$, drawn according to the distribution $P_w(v)$.
Next define for $x \in \MX$, 
$$
\Vote{\MT}{x} := %
\sum_{v \in \MN_{\Lambda}(\MT)} P_w(v) \cdot \soa{\MG_{1:\Lambda}^v}{x}. %
$$
For $\ep \in [0,1]$, define
\begin{align}
  \label{eq:def-highvote}
\Highvote{\MT, \ep} := \left\{ (x, \Vote{\MT}{x}) : x \in \MX, \ %
  \Terr{\MT, x, \Vote{\MT}{x}} \leq \ep
\right\}. %
\end{align}
In words, $\Highvote{\MT, \ep}$ is the set of tuples $(x, \Vote{\MT}{x})$ for which the truncated error of $\Terr{\MG_{1:\Lambda}^v, x, \Vote{\MT}{x}}$ is at most $\ep$, when $v$ is drawn from the distribution induced by $P_w(v)$, for $v \in \MN_{\Lambda}(\MT)$. The set $\Highvote{\MT, \ep}$ should be interpreted as the set of tuples $(x, \Vote{\MT}{x})$ about which the hypotheses $\soaf{\MG_{1:\Lam}^v}$, weighted according to $v \sim P_w(\cdot)$, are ``nearly unanimous'' (up to error $\ep$) about the label of the point $x$. The quantities $\Vote{\MT}{x}$, $\Highvote{\MT, \ep}$ are generalizations of the analogous quantities defined in \cite{hanneke_online_2021} to the real-valued case.

For $f \in \MF$, let $\delta_f \in \Delc(\MF)$ denote the point mass at $f$. For $N \in \BN$ we define the class
\begin{align}
\Votef{\MF^N} = \left\{ \frac 1N \cdot \left( \delta_{f_1} + \cdots + \delta_{f_n} \right) \ : \ f_1, \ldots, f_N \in \MF \right\} \subset \Delc(\MF)\label{eq:votef-def},
\end{align}
to be the collection of (uniform) averages of $N$ hypotheses in $\MF$. %
We will often denote elements of $\Votef{\MF^N}$ with  bars, e.g., given $f_1, \ldots, f_N$, we will denote the corresponding element of $\Votef{\MF^N}$ by $\bar f$, so that $\bar f = (\delta_{f_1} + \cdots + \delta_{f_N}) / N$. The algorithm \HPL outputs elements of $\Votef{\MF^N}$ for an appropriate integer $N$.

\subsection{Some results on weighted subclass collections}
We begin by proving some results on weighted subclass collections (Definition \ref{def:wsc}) and their relation to the notions of truncated error and Highvote defined above. 
Lemma \ref{lem:mean-2approx} shows that the prediction made by the voting hypothesis, $\Vote{\MT}{x}$, achieves the optimal truncated error up to a constant factor (of 2).
\begin{lemma}
  \label{lem:mean-2approx}
Fix a weighted subclass collection $\MT = \wsc{\MG}{w}$. For any $x \in \MX$, it holds that
  \begin{align}
    \Terr{\MT, x, \Vote{\MT}{x}} \leq 2 \cdot \min_{y \in [0,1]} \left\{ \Terr{\MT, x, y } \right\} \nonumber.
    \end{align}
  \end{lemma}
  Note that the conclusion of Lemma \ref{lem:mean-2approx} can be rewritten as:
  $$
\sum_{v \in \MN_\Lambda(\MT)} P_w(v) \cdot \Terr{\MG_{1:\Lambda}^v, x, \Vote{\MT}{x}} \leq 2 \cdot \min_{y \in [0,1]} \sum_{v \in \MN_\Lambda(\MT)} P_w(v) \cdot \Terr{\MG_{1:\Lambda}^v, x, y}.
  $$
\begin{proof}[Proof of Lemma \ref{lem:mean-2approx}]
Set $y_0 := \argmin_{y \in [0,1]} \sum_{v \in \MN_\Lambda(\MT)} P_w(v) \cdot \Terr{\MG_{1:\Lambda}^v, x, y}$ and $y_1 = \Vote{\MT}{x}$. We assume that $y_1 \geq y_0$ (the other case $y_1 \leq y_0$ is treated in a symmetric manner). Set
  \begin{align}
\MS_-(x) :=& \left\{ v \in \MN_\Lambda(\MT) :\ \soa{\MG_{1:\Lambda}^v}{x} < y_1  \right\}\nonumber\\
\MS_+(x) :=& \left\{ v \in \MN_\Lambda(\MT) :\ \soa{\MG_{1:\Lambda}^v}{x} \geq y_1 \right\}\nonumber.
  \end{align}
  Since $y_1$ is the weighted mean of the quantities $\soa{\MG_{1:\Lam}^v}{x}$ (according to the weights $P_w(v)$), over $v \in \MN_\Lam(\MT)$, it holds that
  \begin{align}
    \sum_{v \in \MS_-(x)} P_w(v) \cdot |y_1 - \soa{\MG_{1:\Lambda}^v}{x}| = \sum_{v \in \MS_+(x)} P_w(v) \cdot |y_1 - \soa{\MG_{1:\Lambda}^v}{x}|.\label{eq:com}
  \end{align}
  For $v \in \MS_+(x)$, since $\soa{\MG_{1:\Lam}^v}{x} \geq y_1 \geq y_0$, we have that $|\soa{\MG_{1:\Lam}^v}{x} - y_0| \geq |\soa{\MG_{1:\Lam}^v}{x} - y_1|$, and thus
  \begin{align}
    \label{eq:terr-y1-y0}
    \Terr{\MG_{1:\Lambda}^v, x, y_0} \geq \Terr{\MG_{1:\Lambda}^v, x, y_1}.
  \end{align}
  Hence
  \begin{align}
\sum_{v \in \MS_+(x)} P_w(v) \cdot \Terr{\MG_{1:\Lam}^v, x, y_1} \leq \sum_{v \in \MS_+(x)} P_w(v) \cdot \Terr{\MG_{1:\Lam}^v, x, y_0}.\label{eq:bound-splus}
  \end{align}
  
  Note that for any $v \in \MS_-(x)$, if $\alpha_{\cutoff{\MG_{1:\Lambda}^v, x}} > |\soa{\MG_{1:\Lambda}^v}{x} -y_1|$, we again have that
  \begin{align}
\Terr{\MG_{1:\Lambda}^v, x, y_0} \geq \Terr{\MG_{1:\Lambda}^v, x, y_1} = \alpha_{\cutoff{\MG_{1:\Lambda}^v, x}}\nonumber,
  \end{align}
  since $\alpha_{\cutoff{\MG_{1:\Lam}^v, x}}$ is the minimum, for all $y \in [0,1]$, of $\Terr{\MG_{1:\Lam}^v, x, y}$. 
  Thus, using (\ref{eq:com}) for the first inequality below,
  \begin{align}
    \sum_{v \in \MS_-(x)} P_w(v) \cdot \Terr{\MG_{1:\Lambda}^v, x, y_1} \leq & \sum_{v \in \MS_+(x)} P_w(v) \cdot |\soa{\MG_{1:\Lam}^v}{x} - y_1| + \sum_{v \in \MS_-(x)} P_w(v) \cdot \Terr{\MG_{1:\Lambda}^v, x, y_0} \nonumber\\
    \leq &  \sum_{v \in \MS_+(x)} P_w(v) \cdot \Terr{\MG_{1:\Lambda}^v, x, y_0} + \sum_{v \in \MS_-(x)} P_w(v) \cdot \Terr{\MG_{1:\Lambda}^v, x, y_0}\label{eq:use-y1-y0-shift}\\
    = & \sum_{v \in \MN_\Lambda(\MT)} P_w(v) \cdot \Terr{\MG_{1:\Lambda}^v, x, y_0}\label{eq:bound-sminus},
  \end{align}
  where (\ref{eq:use-y1-y0-shift}) uses (\ref{eq:terr-y1-y0}). Combining (\ref{eq:bound-splus}) and (\ref{eq:bound-sminus}) gives the desired conclusion.
\end{proof}

Lemma \ref{lem:qdist-lb} shows that if the truncated error $\Terr{\MT, x, y}$ is large for some $\MT$ and $(x,y)$, then we can get a lower bound on the weight of hypotheses $(\MG_\lam^{v_\lam}, w_\lam^{v_\lam})$ in the multisets $\MT_\lam$ for which $\soal{\MG_\lam^{v_\lam}}{\alpha_\lam}{x}$ is not close to $y$. This lemma will be used to show that if the truncated error for some example $(x_t, y_t)$ is large, then we can update the classes $\MG_\lam^{v_\lam}$ and the weights $w_\lam^{v_\lam}$ in a way that decreases a certain potential function. %
\begin{lemma}
  \label{lem:qdist-lb}
  Fix a weighted subclass collection $\MT = \wsc{\MG}{w}$. For any example $(x,y) \in \MX \times [0,1]$, it holds that %
  \begin{align}
    \label{eq:pwuv-inequality}
 \alpha_\Lambda +    \sum_{\lambda=1}^{\Lambda} \alpha_\lam \cdot \frac{\sum_{v_\lam=1}^{|\MT_\lam|} w^{v_\lam}_\lam \cdot \One[|\soal{\MG_{\lambda}^{v_\lam}}{\alpha_\lambda}{x} - y | > \alpha_\lambda]}{\sum_{v_\lam=1}^{|\MT_\lam|} w_\lam^{v_\lam}} %
    \geq \frac{1}{16} \cdot {\sum_{v\in \MN_{\Lambda}(\MT)} P_w(v) \cdot \Terr{\MG_{1:\Lambda}^v,x,y}}. %
  \end{align}
\end{lemma}
\begin{proof}
  Consider any $v \in \MN_{\Lambda}(\MT)$. We will assign the mass $P_w(v)$ to some tuple $(\MG_\lam^{v_\lam}, w_\lam^{v_\lam}) \in \MT_\lam$, for some $\lambda \in [\Lambda]$ which satisfies $\alpha_\lambda + \alpha_\Lambda \geq \Omega(\Terr{\MG_{1:\Lambda}^v, x,y})$, in the following manner. %
  Set $\bar \lam_v := \cutoff{\MG_{1:\Lambda}^v, x}$. Then 
by Lemma \ref{lem:real-pot-dec} with $\MF_{1:\Lam} = \MG_{1:\Lam}^v = (\MG_1^{v_1}, \ldots, \MG_\Lam^{v_\Lam})$, at least one of the below is the case:
  \begin{itemize}
  \item $|\soa{\MG_{1:\Lambda}^{v}}{x} - y | \leq \alpha_\Lambda$ and $\bar \lam_v = \Lambda$; or
  \item For some $\lambda' \leq \bar \lam_v + 1$, we have $|\soal{\MG_{\lambda'}^{v_{\lam'}}}{\alpha_{\lambda'}}{x} - y | > \alpha_{\lambda'} \geq | \soa{ \MG_{1:\Lambda}^{v} }{x} - y |/16$.
  \end{itemize}
  Set $\lambda'$ to be the value guaranteed by the second item above, in the event that it holds, and $\lambda' = \perp$ otherwise. 
  Then %
    $$
\alpha_\Lambda + \alpha_{\lambda'} \cdot \One[| \soal{\MG_{\lambda'}^{v_\lamp}}{\alpha_{\lambda'}}{x} - y | > \alpha_{\lambda'}] \geq \max \left\{\frac{1}{16} \cdot | \soa{ \MG_{1:\Lambda}^v }{x} - y |, \frac{\alpha_{\bar \lam_v}}{2} \right\} \geq \frac{1}{16}\cdot \Terr{\MG_{1:\Lambda}^v, x, y}.
$$
(Note that in the event $\lambda' = \perp$, we have that $\One[|\soal{\MG_\lamp^{v_\lamp}}{\alpha_\lamp}{x} - y| > \alpha_\lamp] = 0$, meaning that the left-hand side of the above expression is well-defined.)
For the element $v$, we now assign weight $P_w(v)$ to the tuple $(\MG_\lamp^{v_\lamp}, v_\lamp^{v_\lamp}) \in \MT_\lamp$, in the event that $\lamp \neq \perp$ (and do not assign the weight $P_w(v)$ to any tuple in the event that $\lamp = \perp$).

Note that total weight (taken over all $v \in \MN_\Lam(\MT)$) that could be assigned to any tuple $(\MG_\lam^{v_\lam}, w_\lam^{v_\lam}) \in \MT_\lam$, for any $\lam \in [\Lambda]$, is at most
$$
\sum_{u \in \MN_\Lambda(\MT) : u_\lam = v_\lam} P_w(u) = \frac{w_\lam^{v_\lam}}{\sum_{u_\lam=1}^{|\MT_\lam|} w_\lam^{u_\lam}}.
$$
Thus (\ref{eq:pwuv-inequality}) follows. 
\end{proof}

Lemma \ref{lem:terr-lb} shows that if some SOA hypothesis corresponding to a sequence has large error on a point $(x,y)$, then the SOA hypothesis for the sequence must have large truncated error on $(x,y)$. 
\begin{lemma}
  \label{lem:terr-lb}
Fix a sequence $\MF_{1:\Lambda}$ and a point $(x,y) \in \MX \times [0,1]$. If, for some $\lambda \in [\Lambda]$, $|\soal{\MF_\lam}{\alpha_\lam}{x} - y| > 5 \alpha_\lam$, then $\Terr{\MF_{1:\Lambda},x,y} > \alpha_\lam$.
\end{lemma}
\begin{proof}
  Fix $\lam \in [\Lambda]$ so that $|\soal{\MF_\lam}{\alpha_\lam}{x} - y| > 5 \alpha_\lam$, and set $\bar \lam := \cutoff{\MF_{1:\Lambda}, x}$. Since $\Terr{\MF_{1:\Lambda},x,y} \geq \alpha_{\bar \lam}$, the lemma clearly holds if $\bar \lam \leq \lam$. Otherwise, we have that for all $\lam'$ satisfying $\lam \leq \lam' < \bar \lam$, $|\soal{\MF_{\lam'}}{\alpha_{\lam'}}{x} - \soal{\MF_{\lam'+1}}{\alpha_{\lam'+1}}{x}| \leq 2 \alpha_{\lam'}$. Therefore
  \begin{align}
    | \soa{\MF_{1:\Lambda}}{x} - \soal{\MF_\lam}{\alpha_\lam}{x}| \leq  \sum_{\lam'=\lam}^{\bar \lam-1} |\soal{\MF_{\lam'}}{\alpha_{\lam'}}{x} - \soal{\MF_{\lam'+1}}{\alpha_{\lam'+1}}{x}|
    \leq  \sum_{\lam'=\lam}^{\bar \lam-1} 2 \alpha_{\lam'}
    \leq 4 \alpha_\lam\nonumber,
  \end{align}
  and it follows that $\Terr{\MF_{1:\Lambda},x,y} \geq |y - \soa{\MF_{1:\Lambda}}{x}| \geq \alpha_\lam$. 
\end{proof}
Notice that Lemma \ref{lem:terr-lb} is not true if the truncated error $\Terr{\MF_{1:\Lam}, x,y}$ is replaced with the absolute loss $|y - \soa{\MF_{1:\Lam}}{x}|$: it could be the case that for the given $\lambda$ in the lemma statement, the cutoff point $\cutoff{\MF_{1:\Lam}, x}$ is much smaller than $\lam$ but $\soa{\MF_{1:\Lam}}{x} = \soal{\MF_{\bar \lam}}{\alpha_{\bar \lam}}{x}$ happens to be very close to $y$. 

The next lemma shows that if a weighted subclass collection $\MT = \wsc{\MG}{w}$ is ``nearly unanimous'' about the label $y$ of a point $x$ (in the sense of Highvote, defined in (\ref{eq:def-highvote})), then most of the individual (single-scale) hypotheses $\soalf{\MG_\lam^{v_\lam}}{\alpha_\lam}$ from $\MT$ must predict $(x,y)$ approximately correctly. It may be seen as a sort of converse to Lemma \ref{lem:qdist-lb}, which shows that if the truncated error $\Terr{\MT, x,y}$ is small, then many of the SOA hypotheses at individual scales must be inaccurate on $(x,y)$. 
  \begin{lemma}
    \label{lem:indic-terr-lb}
    Consider any weighted subclass collection $\MT = \wsc{\MG}{w}$, and 
  $(x,y) \in \MX \times [0,1]$. If $(x,y) \in \Highvote{\MT, \alpha}$ for some $\alpha \geq 0$, then for all $\lam\in [\Lam]$, %
  \begin{equation}
    \label{eq:indic-terr-lb}
 \frac{\sum_{v_\lambda = 1}^{|\MT_\lambda|} w_\lam^{v_\lam} \cdot\One[|\soal{\MG_\lambda^{v_\lambda}}{\alpha_\lambda}{x} - y| > 5\alpha_\lambda] }{\sum_{u_\lambda=1}^{|\MT_\lambda|} w_\lam^{u_\lam}} \leq \frac{\alpha}{\alpha_\lam}.
  \end{equation}
\end{lemma}
\begin{proof}
The conclusion of the lemma is immediate if $\alpha \geq \alpha_\lam$, so we may assume from here on that $\alpha < \alpha_\lam$. %
  
  By Lemma \ref{lem:terr-lb}, for any $v \in \MN_\Lambda(\MT)$ and $\lam \in [\Lambda]$ for which  $| \soal{\MG_\lam^{v_\lam}}{\alpha_\lam}{x} - y | > 5 \alpha_\lam$, it holds that $\Terr{\MG_{1:\Lambda}^v, x, y} > \alpha_\lam$. Thus, for any tuple $(\MG_\lam^{v_\lam}, w_\lam^{v_\lam}) \in \MT_\lam$ for which $| \soal{\MG_\lam^{v_\lam}}{\alpha_\lam}{x} - y | > 5\alpha_\lam$ there is a set $\MS_{v_\lam} \subset \MN_\Lambda(\MT)$ (namely, the set of all $u$ for which $u_\lam = v_\lam$) so that for all $u \in \MS_{v_\lam}$, $\Terr{\MG_{1:\Lambda}^u, x,y} > \alpha_\lambda$ and so that $\sum_{u \in \MS_{v_\lam}} P_w(u) = \frac{w_\lam^{v_\lam}}{\sum_{u_\lam=1}^{|\MT_\lam|} w_\lam^{u_\lam}}$.

  That $(x,y) \in \Highvote{\MT, \alpha}$ means that $\sum_{v \in \MN_\Lambda(\MT)} P_w(v) \cdot \Terr{\MG_{1:\Lambda}^v, x, y} \leq \alpha$.  By Markov's inequality, for any $\lam$ for which $\alpha_\lam > \alpha$,
  $$
  \sum_{v \in \MN_\Lambda(\MT)} P_w(v) \cdot \One[\Terr{\MG_{1:\Lam}^v, x,y} > \alpha_\lam] \leq \frac{ \alpha}{\alpha_\lam}.
  $$
  Thus, the mass (under $P_w(\cdot)$) of the union of all the sets $\MS_{v_\lam}$ is at most $\frac{\alpha}{\alpha_\lam}$. 
  Since the sets $\MS_{v_\lam}$ are pairwise disjoint, (\ref{eq:indic-terr-lb}) follows.%
\end{proof}

\begin{lemma}
  \label{lem:find-lam}
Suppose $P$ is a distribution supported on $[0,1]$ so that $\E_{Z \sim P}[Z] > \alpha$ for some $\alpha \in [2\alpha_\Lam,1]$. Then there is some $\lam \in [\Lam]$ so that $\E_{Z \sim P}[\One[Z > \alpha_\lam]] > \frac{\alpha}{4\Lam \alpha_\lam}$. 
\end{lemma}
\begin{proof}
  Suppose for the purpose of contradiction that for all $\lam \in [\Lam]$, $\E[\One[Z > \alpha_\lam]] \leq \frac{\alpha}{4\Lam \alpha_\lam}$. Then
  \begin{align}
    \E[Z] \leq & \sum_{\lam=1}^\Lam \p[\alpha_\lam < Z \leq \alpha_{\lam-1}] \cdot \alpha_{\lam-1} + \alpha_\Lam\nonumber\\
    \leq & \alpha_\Lam + \sum_{\lam=1}^\Lam \frac{\alpha \alpha_{\lam-1}}{4\Lam \alpha_\lam}\nonumber\\
    \leq & \alpha_\Lam + \alpha/2\leq \alpha \nonumber,
  \end{align}
  a contradiction.
\end{proof}

Lemma \ref{lem:perturb-hagg} shows that if two length-$\Lambda$ sequences of real numbers have their first $\lam_0$ positions equal, then their hierarchical aggregations (Definition \ref{def:hagg}) differ by only $O(\alpha_{\lam_0})$.
\begin{lemma}
  \label{lem:perturb-hagg}
Suppose $\Lam \in \BN$, and $g_1, \ldots, g_\Lam, g_1', \ldots, g_\Lam' \in [0,1]$. Suppose $\lam_0 \leq \Lam$ is such that for all $\lam \leq \lam_0$, $g_\lam = g_\lam'$. Then $|\hagg{g_1, \ldots, g_\Lam} - \hagg{g_1', \ldots, g_\Lam'}| \leq 8\alpha_{\lam_0}$. 
\end{lemma}
\begin{proof}
  Set  $\bar \lam := \cutoff{g_1, \ldots, g_\Lam}$ and $\bar \lam' := \cutoff{g_1', \ldots, g_\Lam'}$. 
  In particular, we have $\hagg{g_1, \ldots, g_\Lam} = g_{\bar \lam}$ and $\hagg{g_1', \ldots, g_\Lam'} = g_{\bar \lam'}'$.

  By symmetry, we may assume without loss of generality that $\bar \lam \geq \bar \lam'$. If $\bar \lam = \bar \lam' \leq \lam_0$, then $\hagg{g_1, \ldots, g_\Lam} = g_{\bar \lam} = g_{\bar \lam'} = \hagg{g_1', \ldots, g_\Lam'}$, and the claim of the lemma is immediate.  If $\bar \lam > \bar \lam'$, then $\bar \lam' < \Lam$ and 
  the definition of $\bar \lam, \bar \lam'$ gives that $|g_{\bar \lam'}' - g_{\bar \lam'+1}'| > 2 \alpha_{\bar \lam'} > 2 \alpha_{\bar \lam} \geq |g_{\bar \lam'} - g_{\bar \lam'+1}|$, and in particular $|g_{\bar \lam'}' - g_{\bar \lam'+1}'| \neq |g_{\bar \lam'} - g_{\bar \lam'+1}|$. Hence $\lam_0 \leq \bar \lam'$. Thus, from here on, we may assume that $\lam_0 \leq \bar \lam' \leq \bar \lam$. 

  Now note that
  \begin{align}
|g_{\bar \lam} - g_{\lam_0}| \leq \sum_{\lam=\lam_0}^{\bar \lam-1} |g_\lam - g_{\lam+1}| \leq \sum_{\lam=\lam_0}^{\bar \lam-1} 2 \alpha_\lam \leq 4 \alpha_{\lam_0} \nonumber
  \end{align}
  and
  \begin{align}
|g_{\bar \lam'}' - g_{\lam_0}'| \leq \sum_{\lam=\lam_0}^{\bar \lam'-1} |g_\lam' - g_{\lam+1}'| \leq \sum_{\lam=\lam_0}^{\bar \lam'-1} 2 \alpha_\lam \leq 4 \alpha_{\lam_0} \nonumber.
  \end{align}
  Using that $g_{\lam_0} = g_{\lam_0}'$, we get that $|\hagg{g_1, \ldots, g_\Lam} - \hagg{g_1', \ldots, g_\Lam'}| = |g_{\bar \lam} - g_{\bar \lam'}'| \leq 8 \alpha_{\lam_0}$. 
\end{proof}

  For multisets $\MS, \MS'$ of tuples of the form $(\MG^i, w^i)$ for $\MG^i \subset \MF$, $w^i \geq 0$, define
\begin{equation}
  \label{eq:ttprime-tvd}
\tvd{\MS}{\MS'} = \frac 12 \sum_{\MH \subset \MF} \left| \frac{\sum_{(\MG^i, w^i) \in \MS} w^i \cdot \One[\MG^i = \MH]}{\sum_{(\MG^i, w^i) \in \MS} w^i} -   \frac{\sum_{(\MG^{i\prime}, w^{i\prime}) \in \MS'} w^{i\prime} \cdot \One[\MG^{i\prime} = \MH]}{\sum_{(\MG^{i\prime}, w^{i\prime}) \in \MS'} w^{i\prime}}\right|.
\end{equation}
Note that $\tvd{\MS}{\MS'}$ is the total variation distance between the distributions on subclasses of $\MF$ induced by the weights $w_i$ and $w_i'$. 

Lemma \ref{lem:single-level-wchange} shows a sensitivity-type result for the voting rule $\Vote{\MT}{\cdot}$ and for the truncated error $\Terr{\MT, \cdot}$: if two weighted subclass collections $\MT, \MT'$ are such that their components $\MT_\lam, \MT_\lam'$ are close in the sense of (\ref{eq:ttprime-tvd}) for each $\lam$, then the voting rules and truncated errors for $\MT, \MT'$ are close.
\begin{lemma}
  \label{lem:single-level-wchange}
  Fix any $\lam \in [\Lam]$ and consider weighted subclass collections $\MT = \wsc{\MG}{w}, \MT' = \wsc{\MG'}{w'}$. %
  Then for any $x \in \MX$,
  $$
|\Vote{\MT}{x} - \Vote{\MT'}{x}|  \leq 8 \cdot \sum_{\lam=1}^\Lam \alpha_\lam \cdot \tvd{\MT_\lam}{\MT_\lam'}.
$$
Moreover, for any pair $(x,y) \in \MX \times [0,1]$,
\begin{align}
\left|\Terr{\MT, x, y} - \Terr{\MT', x,y}\right| \leq 8  \cdot \sum_{\lam=1}^\Lam \alpha_\lam \cdot \tvd{\MT_\lam}{\MT_\lam'}\label{eq:terr-ttp-diffs}.
\end{align}
\end{lemma}
\begin{proof}
  First we define finite-support distributions $Q_\lam, Q_\lam'$, for each $\lam \in [\Lam]$, over the set of subclasses of $\MF$, as follows: for $\lam \in [\Lam]$ and $\MH_\lam \subset \MF$, define
  \begin{align}
Q_\lam(\MH_\lam) :=  \frac{\sum_{v_\lam=1}^{|\MT_\lam|} w_\lam^{v_\lam} \cdot \One[{\MG_\lam^{v_\lam}} = \MH_\lam]}{\sum_{v_\lam=1}^{|\MT_\lam|} w_\lam^{v_\lam}}, \qquad Q_\lam'(\MH_\lam) :=  \frac{\sum_{v_\lam=1}^{|\MT_\lam'|} w_\lam^{v_\lam\prime} \cdot \One[\MG_\lam^{v_\lam\prime} = \MH_\lam]}{\sum_{v_\lam=1}^{|\MT_\lam'|} w_\lam^{v_\lam\prime}}\nonumber
  \end{align}
  By the definition (\ref{eq:ttprime-tvd}), there is a coupling between $Q_\lam, Q_\lam'$, which we denote as $\til Q_\lam((\MH_\lam, \MH_\lam'))$, so that
  \begin{equation}
    \label{eq:coupling-tl-tpl}
\sum_{(\MH_\lam, \MH_\lam')} \til Q_\lam((\MH_\lam, \MH_\lam')) \cdot \One[\MH_\lam \neq \MH_\lam'] =\tvd{\MT_\lam}{\MT_\lam'}.
\end{equation}
Next we define a distribution $\til Q$ over pairs of sequences $\MH_{1:\Lam}, \MH_{1:\Lam}'$ of subsets of $\MF$, as follows: 
make independent draws $(\MH_1, \MH_1') \sim \til Q_1, \ldots, (\MH_\Lam, \MH_\Lam') \sim \til Q_\Lam$. %
Now $\til Q$ is the distribution of the resulting pair $(\MH_{1:\Lam}, \MH_{1:\Lam}')$; this is equivalent (up to notational differences) to defining $\til Q$ as the product of the distributions $\til Q_1, \ldots, \til Q_\Lam$. Note that the marginal (under $\til Q$) of any sequence $\MH_{1:\Lam}$ is simply $\sum_{v \in \MN_\Lam(\MT) : \MG_{1:\Lam}^v = \MH_{1:\Lam}} P_w(v)$; an analogous statement holds for the marginal of any sequence $\MH_{1:\Lam}'$.
  
Given two sequences $\MG_{1:\Lam}, \MG'_{1:\Lam}$ of subsets of $\MF$, let $\lam^\st(\MG_{1:\Lam}, \MG'_{1:\Lam})$ denote the minimum value of $\lam$ so that $\MG_\lam \not\equiv \MG'_\lam$ (and set $\lam^\st(\MG_{1:\Lam}, \MG'_{1:\Lam})=\infty$ if such $\lam$ does not exist). Now we can compute, for any $x \in \MX$,
\begin{align}
  & |\Vote{\MT}{x} - \Vote{\MT'}{x}| \nonumber\\
  = & \left| \sum_{(\MH_{1:\Lam}, \MH_{1:\lam}')} \til Q(\MH_{1:\Lam}, \MH_{1:\Lam}') \cdot (\soa{\MH_{1:\Lam}}{x} - \soa{\MH_{1:\Lam}'}{x})\right|\nonumber\\
  \leq & \sum_{(\MH_{1:\Lam}, \MH_{1:\Lam}')} \til Q(\MH_{1:\Lam}, \MH_{1:\Lam}') \cdot |\soa{\MH_{1:\Lam}}{x} - \soa{\MH'_{1:\Lam}}{x}|\nonumber\\
  \leq & \sum_{(\MH_{1:\Lam}, \MH_{1:\Lam}')} \til Q(\MH_{1:\Lam}, \MH_{1:\Lam}') \cdot 8 \cdot \alpha_{\lam^\st(\MH_{1:\Lam}, \MH'_{1:\Lam})}\label{eq:use-lamst-perturb}\\
  \leq & 8 \sum_{\lam=1}^\Lam \tvd{\MT_\lam}{\MT'_\lam} \cdot \alpha_\lam\label{eq:tl-tpl-ub},
\end{align}
where (\ref{eq:use-lamst-perturb}) uses Lemma \ref{lem:perturb-hagg}, and (\ref{eq:tl-tpl-ub}) uses the fact that by (\ref{eq:coupling-tl-tpl}), the total mass under $\til Q$ of pairs $(\soaf{\MG_{1:\Lam}}, \soaf{\MG'_{1:\Lam}})$ for which $\lam^\st(\MG_{1:\Lam}, \MG'_{1:\Lam}) = \lam$ (and thus $\soalf{\MG_\lam}{\alpha_\lam} \not\equiv \soalf{\MG'_\lam}{\alpha_\lam}$), is at most $\tvd{\MT_\lam}{\MT'_\lam}$. 
To establish (\ref{eq:terr-ttp-diffs}), we note that
\begin{align}
  & \left|\Terr{\MT, x, y} - \Terr{\MT', x,y}\right| \nonumber\\
  = & \left| \sum_{(\MH_{1:\Lam}, \MH_{1:\Lam}')} \til Q(\MH_{1:\Lam}, \MH_{1:\Lam}') \cdot (\Terr{\MH_{1:\Lam},x,y} - \Terr{\MH_{1:\Lam}',x,y}) \right|\nonumber\\
  \leq & \sum_{(\MH_{1:\Lam}, \MH_{1:\Lam}')} \til Q(\MH_{1:\Lam}, \MH_{1:\Lam}') \cdot \left| \max\{ |\soa{\MH_{1:\Lam}}{x} - y|, \alpha_{\cutoff{\MH_{1:\Lam},x}} \} - \max\{|\soa{\MH'_{1:\Lam}}{x} - y|, \alpha_{\cutoff{\MH'_{1:\Lam},x}} \} \right|\nonumber\\
  \leq & \sum_{(\MH_{1:\Lam}, \MH_{1:\Lam}')} \til Q(\MH_{1:\Lam}, \MH_{1:\Lam}') \cdot \left(\max\{ |\soa{\MH_{1:\Lam}}{x} - \soa{\MH'_{1:\Lam}}{x}|, | \alpha_{\cutoff{\MH_{1:\Lam},x}}- \alpha_{\cutoff{\MH'_{1:\Lam},x}}|\} \right)\nonumber\\
  \leq & \sum_{(\MH_{1:\Lam}, \MH_{1:\Lam}')} \til Q(\MH_{1:\Lam}, \MH_{1:\Lam}')\cdot \max \{ 8 \cdot \alpha_{\lambda^\st(\MH_{1:\Lam}, \MH_{1:\Lam}')}, 2 \cdot \alpha_{\lambda^\st(\MH_{1:\Lam}, \MH_{1:\Lam}')} \} \nonumber\\
  \leq & 8 \sum_{\lam=1}^\Lam \tvd{\MT_\lam}{\MT_\lam'} \cdot \alpha_\lam\nonumber.
\end{align}
\end{proof}

\subsection{The proper learning algorithm}
\label{sec:hpl}
Our proper (and stable) learning algorithm, \HPL, is presented in Algorithm \ref{alg:hpl}. The algorithm maintains a weighted subclass collection $\MT$; for each $n \geq 1$, let $\MT^n = (\MT_1^n, \ldots, \MT_\Lam^n)$ denote this weighted subclass collection $\MT$ directly after round $n$ of the outer while loop. For each $n \geq 1$, let $t(n)$ denote the value of $t$ maintained by the algorithm at the beginning of round $n$. If the round number $n$ is clear, we will often drop the superscript $n$. The algorithm repeatedly performs the following process: at each round $n \geq 1$ of the outer while loop, it first checks if there is some randomized predictor $\bar f$ satisfying the condition in step \ref{it:ms-voting-2}. If so, then it sets $\bar f_t = \bar f$, i.e., it uses this predictor $\bar f$ to predict the next example $(x_t, y_t)$ (step \ref{it:predict-bar-ft}); note that there $t = t(n)$, so $(x_t, y_t) = (x_{t(n)}, y_{t(n)})$. %
The algorithm then uses $(x_t, y_t)$ to update $\MT$ in step \ref{it:receive-yt-2}, replacing various classes $\MG_\lam^{v_\lam}$ in the weighted subclass collection which \emph{incorrectly} predict $(x_t, y_t)$ with appropriate restrictions at the scale $\alpha_\lam$, and downweighting the corresponding weights $w_\lam^{v_\lam}$.

The more challenging case is when the condition in step \ref{it:ms-voting-2} is not satisfied; in this case, it turns out that by the minimax theorem (Lemma \ref{lem:apply-minmax}), we can find a dataset at each scale $\lam$ which satisfies a property (step \ref{it:f-high-err-2} of \HPL) which is roughly ``dual'' to the property of step \ref{it:ms-voting-2}. This property guarantees that we can perform further restrictions on the subclasses $\MG_\lam^{v_\lam}$ (and decrease the weights $w_\lam^{v_\lam}$ accordingly); it turns out that after a bounded number steps of doing so, the property in step \ref{it:ms-voting-2} will be satisfied, and we can process the next example $(x_{t+1}, y_{t+1})$.

We next introduce some further notation. For each $n \geq 1$, and $\lambda \in [\Lambda]$, let $W_{\lambda,n}$ denote the total of all weights in $\MT_\lambda$ after round $n$, i.e.,
$
W_{\lambda, n} =\sum_{(\MG_\lam^{v_\lam}, w_\lam^{v_\lam}) \in \MT_\lam^n} w_\lam^{v_\lam}.%
$
Further set $W_{\lambda,0} = 1$ for all $\lambda \in [\Lambda]$. The quantities $W_{\lam, n}$, for $\lam \in [\Lam]$, will be used as a potential function to track the progress of \HPL over rounds $n$.

Finally, given the class $\MF \subset [0,1]^\MX$, denote by $\MFabs\subset [0,1]^{\MX \times [0,1]}$ the ``absolute loss class'' of $\MF$, namely $\{ (x,y) \mapsto |f(x) - y| \ : \ f \in \MF \}$. We will also need to consider the dual class of $\MFabs$, denoted $\MFabs^\st$:
\begin{align}
\MFabs^\st := \left\{ h : \MF \ra [0,1]: \quad \exists (x,y) \in \MX \times [0,1], \mbox{ such that } h(f) = |f(x) - y| \ \forall f \in \MF \right\}\nonumber.
\end{align}
By Lemma \ref{lem:sfat-closure} with $k = 1$, $\MZ = \MX \times [0,1]$, and $\phi(a,(x,y)) = |a-y|$ (which is 1-Lipschitz), we have that $\sfat_\alpha(\MFabs) \leq O(\sfat_\alpha(\MF) \cdot \log(1/\alpha))$ for all $\alpha > 0$. Thus $\sfat_\alpha(\MFabs)<\infty$ for all $\alpha > 0$, meaning that (by Lemma \ref{lem:sfat-dual}), $\sfat_\alpha(\MFabs^\st) < \infty$ for all $\alpha > 0$. Let $\fat_\alpha(\cdot)$ denote the (non-sequential) fat-shattering dimension (see Section \ref{sec:misc} for the definition). Since $\sfat_\alpha(\MG) \geq \fat_\alpha(\MG)$ holds for any class $\MG$, we get that $\fat_\alpha(\MFabs) < \infty$ and $\fat_\alpha(\MFabs^\st) < \infty$ for all $\alpha > 0$. %

At some points in our proof we will need to use basic uniform convergence properties for the class $\MFabs, \MFabs^\st$; for this we make the following definitions:
\begin{align}
  V := \frac{10C_0 \cdot \fat_{c_0\alpha_\Lam/10}(\MFabs) \log(10/\alpha_\Lam)}{\alpha_\Lam^2}, \qquad & V^\st := \frac{10C_0\cdot  \fat_{c_0\alpha_\Lam/10}(\MFabs^\st) \log(10/\alpha_\Lam)}{\alpha_\Lam^2},\label{eq:define-v-vstar}\\ %
  m_\lambda :=  \left\lceil \frac{C_1 V}{\alpha_\lambda}\right\rceil &  \quad \forall \lam \in [\Lam]\label{eq:define-m-lam}, %
\end{align}
where $C_0, c_0$ are the constants of Theorem \ref{thm:unif-conv}, and $C_1 > 0$ is a sufficiently large constant. As the parameters $V,V^\st$ will not show up in our rates for regret, we do not attempt to optimize their dependence on any of the relevant parameters. %

Finally, set
\begin{align}
  \mu_0 = \min_{\lam \in [\Lam]} \alpha_\lam m_\lam, \qquad \mu_1 = \max_{\lam \in [\Lam]} \alpha_\lam m_\lam.\label{eq:define-mus}
\end{align}
As long as the constant $C_1$ is sufficiently large, we have that $\mu_1 / \mu_0 \leq 2$.

Recall from Algorithm \ref{alg:hpl} that we set $\hvconst = \hvval$. The below lemma uses the minimax theorem to show that if the condition in step \ref{it:ms-voting-2} of \HPL fails at some step, then the condition in step \ref{it:f-high-err-2} does not fail; thus, this lemma establishes that the algorithm is well-defined, i.e., can run as claimed.
\begin{lemma}
  \label{lem:apply-minmax}
  Suppose the condition in step \ref{it:ms-voting-2} of \HPL (Algorithm \ref{alg:hpl}) fails at some time step, i.e., there is no $\bar f \in \Votef{\MF^{V^\st}}$ so that for each $\lambda \in [\Lamdim]$, $$\sup_{(x,y) \in \Highvote{\MT, \alpha_\lambda/\hvconst}} \E_{f \sim \bar f} \left[|f(x) - y|\right] \leq \alpha_\lambda.$$
Then for every $\lambda \in [\Lambda]$ there is a collection of tuples $(\til x_1^\lambda, \til y_1^\lambda), \ldots, (\til x_{m_{\lam}}^\lam, \til y_{m_{\lam}}^\lam) \in \Highvote{\MT, \alpha_{\lam}/\hvconst}$, so that for every $f \in \MF$, there are some $\lam \in [\Lambda],\ \lamp \in [\Lampdim]$ so that $\frac{1}{m_{\lam}} \sum_{j=1}^{m_{\lam}} \One[|f(\til x_j) - \til y_j| > \alpha_{\lam'}] > \frac{\alpha_\lam}{16 \Lam \alpha_{\lam'}}$.  
\end{lemma}
\begin{proof}
  Fix some weighted subclass collection $\MT$ so that there is no $\bar f \in \Votef{\MF^{V^\st}}$ so that for each $\lam \in [\Lamdim]$, $\sup_{(x,y) \in \Highvote{\MT, \alpha_\lam/C_\lam}} \E_{f \sim \bar f} \left[|f(x) - y|\right] \leq \alpha_\lam$. By Theorem \ref{thm:unif-conv} applied to the class $\MF^\st$ and by the definition of $V^\st$ in (\ref{eq:define-v-vstar}), for every finite support measure $Q$ on $\MF$, there is some $\lambda \in [\Lamdim]$ so that
  \begin{align}
    \sup_{(x,y) \in \Highvote{\MT, \alpha_\lambda/\hvconst}} \E_{f \sim Q} \left[|f(x) - y|\right] %
    >2\alpha_\lambda/3.\nonumber
  \end{align}
  For each $(x,y) \in \MX \times [0,1]$, let $\alpha(x,y)$ be the smallest value of $\alpha_\lambda$ (for $\lam \leq \Lamdim$) for which $(x,y) \in \Highvote{\MT, \alpha_\lambda/\hvconst}$.\footnote{Note that it could be the case that $\alpha(x,y) \geq 1$, i.e., $\lam \leq 0$.} Hence for every finite support measure $Q$ on $\MF$, there is some $(x,y) \in \MX \times [0,1]$ so that $\E_{f \sim Q}[|f(x) - y|] > \alpha(x,y)/2$. 
  Now consider the function class %
  $\MG \subset \BR^\MF$, defined by
  \begin{align}
\MG := \left\{ f \mapsto \frac{|f(x) - y|}{\alpha(x,y)} \ : \ (x,y) \in \MX \times [0,1] \right\}\nonumber.
  \end{align}
  Note that there are a finite number of possible values of $\alpha(x,y)$, namely $\alpha_\lam$ for $\Lam-6 \geq \lam \geq -O(\log(C_\Lam))$. %
  For each such possible value of $\lam$, define
  \begin{align}
\MG_\lam := \left\{ f \mapsto \frac{|f(x) - y|}{\alpha_\lam} \ : \ \alpha(x,y) = \alpha_\lam \right\}.\nonumber
  \end{align}
  Note that $\MG = \bigcup_{\Lam-6 \geq \lam \geq -O(\log C_\Lam)} \MG_\lam$. Further, since each $\MG_\lam$ is simply a subclass of $\MFabs^\st$ scaled by $1/\alpha_\lam$, it holds that $\sfat_\alpha(\MG_\lam)  \leq \sfat_{\alpha\alpha_\lam}(\MFabs^\st) < \infty$ for all $\alpha > 0$. Thus, by Corollary \ref{cor:closure-union}, $\sfat_\alpha(\MG) < \infty$ for all $\alpha > 0$. %
  Then by Theorem \ref{thm:online-minimax}, 
  \begin{align}
    1/2 \leq & \inf_{Q \in \Delta^\circ(\MF)} \sup_{P \in \Delta^\circ(\MX \times [0,1])} \E_{f \sim Q, (x,y) \sim P} \left[ \frac{|f(x) - y|}{\alpha(x,y)}\right]\nonumber\\
    =& \sup_{P \in \Delta^\circ(\MX \times [0,1])}  \inf_{Q \in \Delta^\circ(\MF)} \E_{f \sim Q, (x,y) \sim P} \left[ \frac{|f(x) - y|}{\alpha(x,y)}\right]\nonumber.
  \end{align}
  Thus we may find a finite-support measure $P^\st \in \Delta^\circ(\MX \times [0,1])$ so that for every $f \in \MF$,
  \begin{equation}
    \label{eq:allf-highgamma-2}
    \E_{(x,y) \sim P^\st} \left[ \frac{|f(x) - y|}{\alpha(x,y)} \right] > 1/3.
    \end{equation}

    For each $\lambda$ satisfying $2 \leq \lam \leq \Lamdim$, let $P_\lambda^\st \in \Delta^\circ(\MX \times [0,1])$ be the distribution of $(x,y) \sim P^\st$, conditioned on $\alpha(x,y) = \alpha_\lambda$. Further, for the case $\lam = 1$, let $P_1^\st$ be the distribution of $(x,y) \sim P^\st$, conditioned on $\alpha(x,y) \leq \alpha_1$. (If $P^\st\{(x,y) : \alpha(x,y) = \alpha_\lambda\} = 0$ for some $\lambda$, then let $P_\lambda^\st$ be an arbitrary finite-support distribution on $\MX \times [0,1]$.) By Theorem \ref{thm:unif-conv} with the function class given by $\MFabs$ and by definition of $V$ in (\ref{eq:define-v-vstar}) and of $m_\lam$ in (\ref{eq:define-m-lam}) for each $\lam$, we have the following: for each $\lambda \in [\Lamdim]$, there is a dataset $S^\lambda:= \{(\tilde x_1^\lambda, y_1^\lambda), \ldots, (\tilde x_{m_\lambda}^\lambda, \tilde y_{m_\lambda}^\lambda) \}$ of size $m_\lambda$ %
    so that for any $f \in \MF$ satisfying $\E_{(x,y) \sim P_\lambda^\st}[|f(x) - y|] > \alpha_\lambda/3$, we have $\frac{1}{m_\lambda} \sum_{j=1}^{m_\lambda} |f(\tilde x_j^\lambda) - \tilde y_j^\lambda| > \alpha_\lambda/4$.
    
    But by (\ref{eq:allf-highgamma-2}), we see that for every $f \in \MF$, there is some $\lambda \in [\Lamdim]$ so that $\E_{(x,y) \sim P_\lambda^\st} \left[|f(x) - y|\right] > \alpha_\lambda/3$. Hence, for any $f \in \MF$, there is some $\lambda \in [\Lamdim]$ so that $\frac{1}{m_\lambda} \sum_{j=1}^{m_\lambda} |f(\tilde x_j^\lambda) - \tilde y_j^\lambda| > \alpha_\lambda/4$, %
    which implies, by Lemma \ref{lem:find-lam} with the value of $\Lam$ set to $\Lampdim$ (using that $\alpha_\lam/4 \geq 2\alpha_\Lampdim$ since $\lam \leq \Lamdim$), that for some $\lam' \in [\Lampdim]$, $\frac{1}{m_\lam} \sum_{j=1}^{m_\lam} \One[|f(\til x_j^\lam) - \til y_j^\lam| > \alpha_{\lam'}] > \frac{\alpha_\lam}{16 \Lam \alpha_{\lam'}}$. (Here we have used that $| f(\til x_j^\lam) - \til y_j^\lam| \leq 1$ for all $j, \lam$.)
  \end{proof}

  \begin{algorithm}
      \caption{\bf \HPL}\label{alg:hpl}
  \KwIn{Function class $\MF \subset [0,1]^\MX$, time horizon $T \in \BN$, scale parameter $\Lambda \in \BN$, constants $C_1, C_2 > 0$, $\hvconst = \hvval$.} 
  \begin{enumerate}[leftmargin=14pt,rightmargin=20pt,itemsep=1pt,topsep=1.5pt]
\item For $\lam \in [\Lam]$, initialize $\MT = \wsc{\MG}{w}$ to be the weighted subclass collection with $\MT_\lam = \{ (\MF, 1)\}$ for each $\lam$ (i.e., $\MG_\lam^1 = \MF$ and $w_\lam^1 = 1$ for all $\lam \in [\Lam]$). %
Set $\gamma := \frac{\alpha_\Lam}{\hvconst}$, $m_\lambda := \left\lceil \frac{C_1 V}{\alpha_\lambda}\right\rceil$, $A = \frac{\mu_1}{\alpha_\Lam}$, $t \gets 1$ \emph{(recall definitions of $V$ in (\ref{eq:define-v-vstar}) and $\mu_1$ in (\ref{eq:define-mus}))}. %
\item While $t \leq T$:
  \begin{enumerate}
  \item If, there is $\bar f \in \Votef{\MF^{V^\st}}$ so that, for each $\lambda \in [\Lamdim]$, $\sup_{(x,y) \in \Highvote{\MT, \frac{\alpha_\lambda}{\hvconst}}} \E_{f \sim \bar f} \left[|f(x) - y|\right] \leq \alpha_\lambda$: \emph{(recall definition of $V^\st$ in (\ref{eq:define-v-vstar}))}\label{it:ms-voting-2}
    \begin{enumerate}
    \item Choose $\bar f_t = \bar f$. On the next example $x_t$, draw $f \sim \bar f_t$ and predict $f(x_t)$. %
      \label{it:predict-bar-ft}
    \item \label{it:receive-yt-2} Receive $y_t$, and let %
      $\delta_t := \Terr{\MT, x_t, y_t}$. %
      \begin{itemize}
      \item For each $1 \leq \lambda \leq \Lambda$,  %
   and each $v_\lambda \in [|\MT_\lambda|]$:%
        \label{it:for-fi-2}
        \begin{enumerate}
        \item If $|\soal{\MG_\lambda^{v_\lambda}}{\alpha_\lam}{x_t} - y_t| > \alpha_\lambda$, set $w_\lambda^{v_\lambda} \gets \gamma\cdot w_\lambda^{v_\lambda}$. %
          \label{it:2suc-wdec-2}
        \item If $|\soal{\MG_\lambda^{v_\lambda}}{\alpha_\lam}{x_t} - y_t| > \alpha_\lambda$, set %
          $\MG_\lambda^{v_\lambda} \gets \MG_\lambda^{v_\lambda}|^{\alpha_\lambda}_{(x_t, y_t)}$.\label{it:2suc-fres-2}
        \item If $\MG_\lambda^{v_\lambda} = \emptyset$, remove $(\MG_\lambda^{v_\lambda}, w_\lambda^{v_\lambda})$ from $\MT$.
        \end{enumerate}
      \end{itemize}
    \item Set $t \gets t+1$. 
    \end{enumerate}
  \item Else, choose $\{ (\tilde x_1^\lambda, \tilde y_1^\lambda), \ldots, (\tilde x_{m_\lambda}^\lambda, \tilde y_{m_\lambda}^\lambda) \} \subset \Highvote{\MT, \frac{\alpha_\lambda}{\hvconst}}$ for all $\lambda \in [\Lambda]$ so as to satisfy the following property: for every $f \in \MF$, there are some $\lambda \in [\Lam],\lam' \in [\Lampdim]$ so that $\frac{1}{m_\lambda} \sum_{j=1}^{m_\lambda} \One[|f(\tilde x_j^\lambda) - \tilde y_j^\lambda | > \alpha_{\lam'}] > \frac{\alpha_\lam}{16\Lam \alpha_{\lam'}}$: \label{it:f-high-err-2}%
    \begin{enumerate}
    \item For each $\lambda,\lamp \in [\Lambda]$: %
      \begin{itemize}
      \item For each $v_\lamp \in [|\MT_\lamp|]$, and each $j \in [m_{\lambda}]$:
        \begin{enumerate}
        \item \label{it:2fail-soa-close} If $|\soal{\MG_\lamp^{v_\lamp}}{\alpha_\lamp}{\tilde x_j^\lambda} - \tilde y_j^\lambda| \leq  5\alpha_\lamp$, set
          \begin{align}
            w_{\lamp,\lam}^{v_\lamp,j} \gets & \gamma \cdot w_\lamp^{v_\lamp},\nonumber\\
          \MG_{\lamp,\lam}^{v_\lamp,j,b} \gets  & \begin{cases}
            \MG_{\lamp}^{v_\lamp}|_{(\til x_j^\lam, b \alpha_\lamp)}^{\alpha_\lamp}  &:  \quad |b\alpha_\lamp - \til y_j^\lam| > 6\alpha_\lamp \\
            \emptyset  & : \quad | b\alpha_\lamp - \til y_j^\lam| \leq 6 \alpha_\lamp,
          \end{cases}\qquad \forall 0 \leq b \leq \lfloor 1/\alpha_\lamp \rfloor + 1.\nonumber
          \end{align}
      \item \label{it:2fail-soa-far} Otherwise, set %
        \begin{align}
          w_{\lamp,\lam}^{v_\lamp,j} \gets &  w_\lamp^{v_\lamp} \nonumber\\
            \MG_{\lamp,\lam}^{v_\lamp,j,0} \gets & \MG_\lamp^{v_\lamp}, \qquad \MG_{\lamp,\lam}^{v_\lamp,j,b} \gets \emptyset \ \ \  \forall 1 \leq b \leq \lfloor 1/\alpha_\lamp \rfloor + 1.\nonumber          
        \end{align}
      \end{enumerate}
      \end{itemize}
    \item For $\lamp \in [\Lambda]$, set\label{it:make-tlamp}
      $$
 \hspace{-2.4cm}     \MT_\lamp \gets A \cdot \MT_\lamp \cup \bigcup_{\lam \in [\Lam]} \{ (\MG^{v_\lamp,j,b}_{\lamp,\lam}, w^{v_\lamp,j}_{\lamp,\lam}) : \ j \in [m_\lambda], \ b \leq \left\lfloor \frac{1}{\alpha_\lamp} \right\rfloor + 1,\ v_\lamp \in [|\MT_\lamp|],\ \MG^{v_\lamp,j,b}_{\lamp,\lam} \neq \emptyset \}.
      $$
    \end{enumerate}
  \end{enumerate}
\end{enumerate}
\end{algorithm}

Lemma \ref{lem:2suc-w-dec} shows that the potentials $W_{\lam,n}$ (for $\lam \in [\Lam]$) decrease in each round $n$ of \HPL if the condition in step \ref{it:ms-voting-2} succeeds in round $n$.
\begin{lemma}
  \label{lem:2suc-w-dec}
Consider any round $n$ in the algorithm for which the condition in step \ref{it:ms-voting-2} of \HPL (Algorithm \ref{alg:hpl}) holds. Let $t = t(n)$ be the value of $t$ at step \ref{it:receive-yt-2}. If $\delta_t > 32 \alpha_\Lam$, then it holds that for some $\lambda \in [\Lambda]$, $W_{\lambda,n} \leq W_{\lambda,n-1} \cdot \left(1 - \frac{\delta_t}{64\Lam \alpha_\lam} \right)$. Further, for all $\lamp \in [\Lam]$, $W_{\lamp,n} \leq W_{\lamp,n-1}$. 
\end{lemma}
\begin{proof}
  Recall that $\delta_t = \Terr{\MT, x_t, y_t}$, where $\MT = \wsc{\MG}{w}$ is the weighted subclass collection maintained by \HPL at the beginning of round $n$ (equivalently, at the end of round $n-1$). 
  By Lemma \ref{lem:qdist-lb}, there is some $\lam \in [\Lam]$ so that
  $$
\alpha_\lam \cdot \frac{\sum_{v_\lam=1}^{|\MT_\lam|} w_\lam^{v_\lam} \cdot \One[|\soal{\MG_\lam^{v_\lam}}{\alpha_\lam}{x} - y| > \alpha_\lam]}{W_{\lam,n-1}} \geq \frac{1}{\Lambda} \left(\frac{1}{16} \delta_t - \alpha_\Lambda\right) \geq \frac{\delta_t}{32\Lam}.
$$
Since $\gamma \leq 1/2$, it follows that
$$
W_{\lam,n} \leq W_{\lam,n-1} \cdot \left(1 - \frac{\delta_t}{32 \Lam \alpha_\lam} \right) + W_{\lam,n-1} \cdot \frac{\delta_t}{32 \Lam \alpha_\lam} \cdot \gamma \leq W_{\lam,n-1} \cdot \left(1 - \frac{\delta_t}{64\Lam \alpha_\lam} \right).
$$
The fact that $W_{\lamp,n} \leq W_{\lamp,n-1}$ for all $\lamp \in [\Lam]$ is immediate.
\end{proof}

Complementing the previous Lemma \ref{lem:2suc-w-dec}, Lemma \ref{lem:2fail-w-ub} shows that the potential $W_{\lam,n}$ decreases in round $n$ of \HPL if the condition in step \ref{it:ms-voting-2} fails in round $n$.
\begin{lemma}
  \label{lem:2fail-w-ub}
  Consider any round $n$ for which the condition in step \ref{it:ms-voting-2} of \HPL (Algorithm \ref{alg:hpl}) fails. For all $\lamp \in [\Lambda]$, it holds that
  \begin{align}
    W_{\lamp,n} \leq  W_{\lamp,n-1} \cdot \left( A + \frac{3 \mu_1 \Lam}{\alpha_\lamp \cdot \hvconst}\right)\nonumber.
  \end{align}
  Further, for any $\lamp \in [\Lam]$, letting $w_{\lamp,\lam}^{v_\lamp,j}$ denote the weights constructed in step \ref{it:2fail-soa-close} at round $n$, we have
  \begin{align}
    \sum_{(\lam,j,b,v_\lamp) : \MG_{\lamp,\lam}^{v_\lamp,j,b} \neq \emptyset} w_{\lamp,\lam}^{v_\lamp,j} 
    \leq &  \frac{3 \mu_1 \Lam}{\alpha_\lamp \cdot \hvconst} \cdot W_{\lamp,n-1}\label{eq:w4ind-bnd-result},
  \end{align}
  where the summation is over $\lam \in [\Lam], j \in [m_\lam], 0 \leq b \leq \lfloor 1/\alpha_\lamp \rfloor + 1$, and $v_\lamp \in [|\MT_\lamp^{n-1}|]$. 
\end{lemma}
\begin{proof}
  Let  $\MT = \MT^{n-1} = \wsc{\MG}{w}$ denote the weighted subclass collection at the end of round $n-1$ (i.e., at the beginning of round $n$). 
  Consider the datasets $\{ (\til x_1^\lam, \til y_1^\lam), \ldots, (\til x_{m_\lam}^\lam, \til y_{m_\lam}^\lam)\} \subset \Highvote{\MT, \alpha_\lam/C_\Lam}$ constructed in step \ref{it:f-high-err-2}. 
  Fix any $\lam,\lamp \in [\Lam]$, %
  and $j \in [m_{\lam}]$. Since $(\til x_j^\lam, \til y_j^\lam) \in \Highvote{\MT, \frac{\alpha_\lam}{\hvconst}}$, by Lemma \ref{lem:indic-terr-lb} with $\alpha = \alpha_\lam/C_\Lam$, at most a fraction $\frac{ \alpha_\lam}{\alpha_\lamp\cdot \hvconst}$ of the weight $w_\lamp^{v_\lamp}$, for $1 \leq v_\lamp \leq |\MT_\lamp|$, satisfies $\One[|\soal{\MG_\lamp^{v_\lamp}}{\alpha_{\lamp}}{\til x_j^\lam} - \til y_j^\lam | > 5 \alpha_\lamp$. Let $\MS_{\lam,0}$ be the set of such indices $v_\lamp$, and let $\MS_{\lam,1} = [|\MT_\lamp|] \backslash \MS_{\lam,0}$. Thus, letting $w_{\lamp,\lam}^{v_\lamp,j}$ denote the weights constructed in step \ref{it:2fail-soa-close} at round $n$,
  \begin{align}
    \sum_{v_\lamp \in \MS_{\lam,0}} w_{\lamp,\lam}^{v_\lamp,j} + \frac{2}{\alpha_\lamp} \sum_{v_\lamp \in \MS_{\lam,1}} w_{\lamp,\lam}^{v_\lamp, j}
    \leq &  \left( \frac{ \alpha_\lam}{\alpha_\lamp \cdot \hvconst} + \frac{2\gamma}{\alpha_\lamp} \right) \cdot W_{\lamp,n-1}\leq \frac{3 \alpha_\lam}{\alpha_\lamp \cdot \hvconst} \cdot W_{\lamp,n-1}\label{eq:w4ind-bnd}.
  \end{align}
Using that $\lfloor 1/\alpha_\lamp \rfloor + 2 \leq 2/\alpha_\lamp$ for each $\lamp \in [\Lam]$, it follows that
  \begin{align} %
    W_{\lamp,n} \leq & A \cdot W_{\lamp,n-1} + 2 \sum_{\lam =1}^\Lam \sum_{j=1}^{m_\lam} \left( \sum_{v_\lamp \in \MS_{\lam,0}} w_{\lamp,\lam}^{v_\lamp,j} + \frac{2}{\alpha_\lamp} \sum_{v_\lamp \in \MS_{\lam,1}} w_{\lamp,\lam}^{v_\lamp, j}\right) \nonumber\\
    \leq & W_{\lamp,n-1}\cdot \left( A + \sum_{\lam=1}^\Lam \frac{3 \alpha_\lam m_\lam}{\alpha_\lamp \cdot \hvconst}\right)\nonumber\\
    \leq & W_{\lamp,n-1} \cdot \left( A + \frac{3 \mu_1 \Lam}{\alpha_\lamp \cdot \hvconst}\right)\nonumber,
  \end{align}
  as desired. The second claim (\ref{eq:w4ind-bnd-result}) of the lemma follows in an identical manner, except with the leading term $W_{\lamp,n-1} \cdot A$ above deleted.
\end{proof}

Lemma \ref{lem:vote-slow-change-2} forms a crucial part of the \cumls bound proof for \HPL. For each step $n+1$ of \HPL for which the condition in step \ref{it:ms-voting-2} holds, Lemma \ref{lem:vote-slow-change-2} upper bounds the expected error $\E_{f \sim \bar f_{t(n+1)}} \left[|\bar f_{t(n+1)}(x) - y|\right]$ of $\bar f_{t(n+1)}$ on any point $(x,y)$ as the sum of 3 terms: the third term is the truncated error of $\MT$ on $(x,y)$ at some later step (namely, step $n+n_0-1$ for some $n_0 \geq 1$), and the first two terms depend on the behavior of the algorithm between steps $n$ and $n+n_0-1$ (in particular, the first two terms are 0 if $n_0 = 0$, i.e., $n = n+n_0-1$). 
  \begin{lemma}
    \label{lem:vote-slow-change-2}
    Fix integers $n \geq 0$ and $n_0 \geq 1$. %
    Let $\MS_0 \subset \{ n+1, n+2, \ldots, n+n_0-1 \}$ be the subset of rounds $n$ in which the condition on step \ref{it:ms-voting-2} of \HPL (Algorithm \ref{alg:hpl}) holds, and $\MS_1 = \{ n+1, n+2, \ldots, n+n_0-1\} \backslash \MS_0$. Suppose further the condition on step \ref{it:ms-voting-2} holds at step $n+1$. For any point $(x,y) \in \MX \times \MY$, it holds that %
   \begin{equation}
     \label{eq:barft-slowchanging-2}
     \E_{f \sim \bar f_{t(n+1)}} \left[| \bar f(x) - y |\right] \leq  \sum_{n' :\ n+n' \in \MS_0}80\cdot \hvconst \sum_{\lam=1}^\Lam \alpha_\lam \cdot \ln \left( \frac{W_{\lam,n+n'-1}}{W_{\lam,n+n'}}\right) + \frac{|\MS_1|}{A} \cdot 150\hvconst \mu_1 \Lam^2+ 5\cdot \hvconst \cdot \Terr{\MT^{n+n_0-1}, x, y}.
   \end{equation}
   In particular, for any round $n+1$ on which the condition in step \ref{it:ms-voting-2} holds, $\MS_0 = \MS_1 = \emptyset$ and so %
   \begin{align}
\E_{f \sim \bar f_{t(n+1)}} \left[| \bar f(x) - y|\right] \leq 2 \cdot\hvconst \cdot \Terr{\MT^{n}, x, y}\nonumber.
   \end{align}
\end{lemma}
\begin{proof}
  Fix any $1 \leq n' \leq n_0-1$. We consider two cases:
  \begin{itemize}
  \item The condition in \ref{it:ms-voting-2} holds at round $n+n'$, i.e., $n+n' \in \MS_0$. For $\lam \in [\Lam]$, write
    \begin{align}
\zeta_\lam := \frac{\sum_{(\MG_\lam^{v_\lam}, w_\lam^{v_\lam}) \in \MT^{n+n'-1}_\lam} w_\lam^{v_\lam} \cdot \One[|\soal{\MG_\lam^{v_\lam}}{\alpha_\lam}{x_{t(n+n')}} - y_{t(n+n')}| > \alpha_\lam]}{\sum_{(\MG_\lam^{v_\lam}, w_\lam^{v_\lam}) \in \MT^{n+n'-1}_\lam} w_\lam^{v_\lam}}\nonumber.
    \end{align}
Then by the update in step \ref{it:2suc-wdec-2},
    \begin{align}
      \frac{W_{\lam,n+n'}}{W_{\lam,n+n'-1}} \leq 1 - \zeta_\lam \cdot (1-\gamma)   \leq 1-\zeta_\lam/2 \leq \exp(-\zeta_\lam/2).\nonumber
    \end{align}
    Therefore, %
    \begin{align}
      \tvd{\MT_\lam^{n+n'}}{\MT_\lam^{n+n'-1}} \leq & \zeta_\lam \leq 2 \ln \left( \frac{W_{\lam,n+n'-1}}{W_{\lam,n+n'}}\right)\nonumber.
    \end{align}
Therefore, by Lemma \ref{lem:single-level-wchange}, it follows that for any $x \in \MX$,
    \begin{align}
\left| \Vote{\MT^{n+n'}}{x} - \Vote{\MT^{n+n'-1}}{x} \right| \leq 16 \sum_{\lam=1}^\Lam \alpha_\lam \cdot \ln \left( \frac{W_{\lam, n+n'-1}}{W_{\lam, n+n'}} \right)\label{eq:vote-slow-change-0},
    \end{align}
    and for any $(x,y) \in \MX \times [0,1]$,
    \begin{align}
|\Terr{\MT^{n+n'}, x, y} - \Terr{\MT^{n+n'-1},x,y}| \leq 16 \sum_{\lam=1}^\Lam \alpha_\lam \cdot \ln \left( \frac{W_{\lam, n+n'-1}}{W_{\lam, n+n'}} \right)\label{eq:terr-slow-change-0}.
    \end{align}
  \item The condition in step \ref{it:ms-voting-2} does not hold at round $n+n'$, i.e., $n+n' \in \MS_1$.
        By (\ref{eq:w4ind-bnd-result}) of Lemma \ref{lem:2fail-w-ub}, for any $\lam, \lamp \in [\Lam]$, %
        \begin{align}
          \sum_{(\lam, j, b, v_\lamp) : \MG_{\lamp,\lam}^{v_\lamp,j,b} \neq \emptyset} w_{\lamp,\lam}^{v_\lamp,j} \leq 
      \frac{3 \mu_1 \Lam}{\alpha_\lamp \cdot \hvconst} \cdot W_{\lamp, n+n'-1} \leq \frac{3 \mu_1 \Lam}{\alpha_\lamp \cdot \hvconst} \cdot \frac{W_{\lamp,n+n'}}{A}, \nonumber
        \end{align}
        where the summation on the left-hand side is over $\lam \in [\Lam], j \in [m_\lam], 0 \leq b \leq \lfloor 1/\alpha_\lamp \rfloor + 1$, and $v_\lamp \in [|\MT_\lamp^{n-1}|]$. 
Then for each $\lamp \in [\Lam]$,
    \begin{align}
      \tvd{\MT_\lamp^{n+n'}}{\MT_\lamp^{n+n'-1}} \leq &\frac{   \sum_{(\lam, j, b, v_\lamp) : \MG_{\lamp,\lam}^{v_\lamp,j,b} \neq \emptyset} w_{\lamp,\lam}^{v_\lamp,j}}{W_{\lamp,n+n'}}\nonumber\\ %
      \leq & \frac{1}{A} \cdot \frac{3 \mu_1 \Lam}{\alpha_\lamp \cdot \hvconst}\nonumber.
    \end{align}
    Therefore, by Lemma \ref{lem:single-level-wchange}, it follows that for any $x \in \MX$,
    \begin{align}
      |\Vote{\MT^{n+n'}}{x} - \Vote{\MT^{n+n'-1}}{x}| \leq & \frac{8}{A} \sum_{\lamp=1}^\Lam \alpha_\lamp \cdot \frac{3 \mu_1 \Lam}{\alpha_\lamp \cdot \hvconst} \leq \frac{1}{A} \cdot \frac{30 \mu_1 \Lam^2}{\hvconst}\label{eq:vote-slow-change}
    \end{align}
    and for any $(x,y) \in \MX \times [0,1]$,
    \begin{align}
|\Terr{\MT^{n+n'}, x,y} - \Terr{\MT^{n+n'-1},x,y}| \leq & \frac{1}{A} \cdot \frac{30 \mu_1 \Lam^2}{\hvconst}\label{eq:terr-slow-change}.
    \end{align}
  \end{itemize}
Fix any pair $(x,y) \in \MX \times [0,1]$. By Lemma \ref{lem:mean-2approx}, we have
    \begin{align}
      \Terr{\MT^{n+n_0-1},x, \Vote{\MT^{n+n_0-1}}{x}} \leq &  2 \Terr{\MT^{n+n_0-1}, x, y}\label{eq:terr-factor2},
\end{align}
and so, using the triangle inequality,
\begin{align}
  & \Terr{\MT^{n}, x, \Vote{\MT^{n}}{x}} \nonumber\\
  \leq & |\Vote{\MT^{n+n_0-1}}{x} - \Vote{\MT^{n}}{x}| + \Terr{\MT^{n}, x, \Vote{\MT^{n+n_0-1}}{x}} \nonumber\\
  \leq & |\Vote{\MT^{n+n_0-1}}{x} - \Vote{\MT^{n}}{x}| + \left|  \Terr{\MT^{n+n_0-1}, x, \Vote{\MT^{n+n_0-1}}{x}} -  \Terr{\MT^{n}, x, \Vote{\MT^{n+n_0-1}}{x}}\right| \nonumber\\
  & +  \Terr{\MT^{n+n_0-1}, x, \Vote{\MT^{n+n_0-1}}{x}} \nonumber\\
  \leq & \sum_{n' :\ n+n' \in \MS_0}32 \sum_{\lam=1}^\Lam \alpha_\lam \cdot \ln \left( \frac{W_{\lam,n+n'-1}}{W_{\lam,n+n'}}\right) + \frac{|\MS_1|}{A} \cdot \frac{60 \mu_1 \Lam^2}{\hvconst}+ 2 \Terr{\MT^{n+n_0-1}, x, y} =: \delta\nonumber,
\end{align}
where the final inequality uses (\ref{eq:vote-slow-change-0}), (\ref{eq:terr-slow-change-0}),  (\ref{eq:vote-slow-change}), (\ref{eq:terr-slow-change}), and (\ref{eq:terr-factor2}). 
Thus, $(x, \Vote{\MT^{n}}{x}) \in \Highvote{\MT^{n}, \delta}$. Since the condition on step \ref{it:ms-voting-2} holds in round $n+1$ (by assumption), it follows that $\bar f_{t(n+1)}$ satisfies
\begin{align}
\E_{f \sim \bar f_{t(n+1)}} \left[|f(x) - \Vote{\MT^n}{x}|\right] \leq \max \left\{ 64\cdot\alpha_\Lam, 2\hvconst \cdot \delta \right\} = 2\hvconst \cdot \delta\label{eq:ftn-votex-bound},
\end{align}
where the equality above follows since $\hvconst \cdot \delta \geq \hvconst \cdot \Terr{\MT^{n+n_0-1}, x, \Vote{\MT^{n+n_0-1}}{x}} \geq 32 \cdot \alpha_\Lam$, since $\hvconst \geq 32$. It then follows that, writing $\MT^n = [\MG, w]$,
\begin{align}
  & \E_{f \sim \bar f_{t(n+1)}} \left[ | f(x) - y| \right] \nonumber\\
  \leq & \E_{f \sim \bar f_{t(n+1)}}\left[ |f(x) -\Vote{\MT^n}{x} | \right] + |y - \Vote{\MT^n}{x}|\nonumber\\
  \leq & \E_{f \sim \bar f_{t(n+1)}} \left[ | f(x) - \Vote{\MT^n}{x}| \right] + \sum_{v \in \MN_\Lam(\MT^n)} P_w(v) \cdot \left|y - \soa{\MG_{1:\Lam}^v}{x}\right|  \nonumber\\
  \leq &  \E_{f \sim \bar f_{t(n+1)}} \left[ | f(x) - \Vote{\MT^n}{x}| \right] + \sum_{v \in \MN_\Lam(\MT^n)} P_w(v) \cdot \Terr{\MG_{1:\Lam}^v,x,y} \nonumber\\
  = & \E_{f \sim \bar f_{t(n+1)}} \left[ |f(x) - \Vote{\MT^n}{x}| \right] + \Terr{\MT^n, x,y}\nonumber\\
  \leq &  \E_{f \sim \bar f_{t(n+1)}} \left[ |f(x) - \Vote{\MT^n}{x}| \right] + \Terr{\MT^{n+n_0-1}, x,y} \nonumber\\
  &  + \sum_{n' :\ n+n' \in \MS_0}16 \sum_{\lam=1}^\Lam \alpha_\lam \cdot \ln \left( \frac{W_{\lam,n+n'-1}}{W_{\lam,n+n'}}\right) + \frac{|\MS_1|}{A} \cdot \frac{30 \mu_1 \Lam^2}{\hvconst}\label{eq:ftn-y-bound}.
\end{align}
The claim of the lemma then follows from (\ref{eq:ftn-votex-bound}) and (\ref{eq:ftn-y-bound}).

\end{proof}

Recall that for each $\lam \in [\Lam]$, we use the parameter $W_{\lam, n} = \sum_{(\MG_\lam^{v_\lam}, w_\lam^{v_\lam}) \in \MT_\lam^n} w_\lam^{v_\lam}$ as a potential function; in particular, Lemmas \ref{lem:2suc-w-dec} and \ref{lem:2fail-w-ub} show an upper bound for the values of $W_{\lam, n}$ in each round $n$. In order to bound the total number of rounds $n$ (thus showing that the algorithm converges), as well as the total error, it is necessary to have a lower bound for the weights $w_\lam^{v_\lam}$ as well; such a lower bound is provided by Lemma \ref{lem:w-sfat-lb} below.
\begin{lemma}
  \label{lem:w-sfat-lb}
For all rounds $n$, and all $\lam \in [\Lam]$, it holds that for each pair $(\MG_\lam^{v_\lam}, w_\lam^{v_\lam}) \in \MT_\lam$, $w_\lam^{v_\lam} \geq \gamma^{\sfat_{\alpha_\lam}(\MF)}$. 
  \end{lemma}
  \begin{proof}
    We will prove the stronger statement that at any step of the algorithm, for all pairs $(\MG_\lam^{v_\lam}, w_\lam^{v_\lam}) \in \MT_\lam$, it holds that $w_\lam^{v_\lam} \geq \gamma^{\sfat_{\alpha_\lam}(\MF) - \sfat_{\alpha_\lam}(\MG_\lam^{v_\lam})}$. 
    Note that the only two points in the algorithm where any pair $(\MG_\lam^{v_\lam}, w_\lam^{v_\lam}) \in \MT_\lam$ is changed, for any $\lam \in [\Lam]$, are steps \ref{it:2suc-wdec-2} and \ref{it:make-tlamp}. %
    Moreover, note that when any weight $w_\lam^{v_\lam}$ is changed to a new value $w_\lam^{v_\lam\prime}$, we always have $w_\lam^{v_\lam\prime}/w_\lam^{v_\lam} \geq \gamma$. Thus, it suffices to show that whenever any weight $w_\lam^{v_\lam}$ is changed, the sequential fat-shattering dimension of $\MG_\lam^{v_\lam}$ (at the scale $\alpha_\lam$) always decreases by at least 1. 
    To do so, we consider each of the two possibilities in turn:
    \begin{itemize}
    \item If we decrease $w$ by a factor of $\gamma$ in step \ref{it:2suc-wdec-2}, then we also replace $\MG$ with $\MG|^{\alpha_\lam}_{(x_t, y_t)}$ (and it must be the case that $|\soal{\MG}{\alpha_\lam}{x_t} - y_t| > \alpha_\lam$). By Lemma \ref{lem:far-sfat-dec}, it holds that $\sfat_{\alpha_\lam}(\MG|^{\alpha_\lam}_{(x_t, y_t)}) < \sfat_{\alpha_\lam}(\MG)$, i.e., the $\alpha_\lam$-sequential fat shattering dimension of $\MG$ must strictly decrease.
    \item Now suppose we are at step \ref{it:2fail-soa-close}, where we will add $(\MG_{\lamp,\lam}^{v_\lamp, j, b}, w_{\lamp, \lam}^{v_\lamp, j, b})$ to $\MT_\lamp$ (in step \ref{it:make-tlamp}), for some choices of $\lam, \lam' \in [\Lam], v_\lamp \in [|\MT_\lamp|], j \in [m_\lam], b \in \{0, 1, \ldots, \lfloor 1/\alpha_\lamp \rfloor + 1\}$. %
      Then for the pair $(\MG_\lamp^{v_\lamp}, w_\lamp^{v_\lamp})$ which was previously in $\MT_\lamp$, we have $w_{\lamp, \lam}^{v_\lamp, j, b} = \gamma \cdot w_\lamp^{v_\lamp}$, and $\MG_{\lamp, \lam}^{v_\lamp, j, b} = \MG_\lamp^{v_\lamp}|^{\alpha_\lamp}_{(\til x_j^\lam, b \alpha_\lamp)}$, where $|\til y_j^\lam - b \alpha_\lamp| > 6 \alpha_\lamp$. %
      We have $|\soal{\MG}{\alpha_\lamp}{\til x_j^\lam} - \til y_j^\lam| \leq 5 \alpha_\lamp$, and therefore, $|\soal{\MG}{\alpha_\lamp}{\til x_j^\lam} - b \alpha_\lamp| > \alpha_\lamp$, so by Lemma \ref{lem:far-sfat-dec}, we have $\sfat_{\alpha_\lamp}(\MG_{\lamp,\lam}^{v_\lamp,j,b}) < \sfat_{\alpha_\lamp}(\MG)$. 
    \end{itemize}
  \end{proof}

  Lemma \ref{lem:w-sfat-lb} above shows a lower bound on the size of the individual weights $w_\lam^{v_\lam}$ in the weighted subclass collections $\MT^n$; in order to lower bound the weights $W_{\lam,n}$, we need a lower bound on the number of pairs $(\MG_\lam^{v_\lam}, w_\lam^{v_\lam})$ remaining in the multisets $\MT_\lam^n$; Lemma \ref{lem:2fail-q-grow} below aids in obtaining such a lower bound.
    \begin{lemma}
    \label{lem:2fail-q-grow}
Consider any round $n$ for which the condition in step \ref{it:ms-voting-2} of \HPL (Algorithm \ref{alg:hpl}) fails. Suppose that before this round, for each $\lambda \in [\Lambda]$, there are $q_\lambda \in \BN$ tuples $(\MG_{\lambda}^{v_\lambda}, w_\lambda^{v_\lambda}) \in \MT_\lambda$ with $f^\st \in \MG_\lambda^{v_\lambda}$. Then after round $n$, for some $\lampp \in [\Lambda]$, there are at least $q_\lampp \cdot \left( A + \frac{\mu_0}{128\Lam \alpha_\lampp} \right)$ tuples $((\MG_\lampp^{v_\lampp})', (w_\lampp^{v_\lampp})') \in \MT_\lampp$ with $f^\st \in (\MG_\lampp^{v_\lampp})'$. 
\end{lemma}
\begin{proof}
  Set $\lambda \in [\Lam],\lamp \in [\Lampdim]$ to be so that $\frac{1}{m_\lambda} \sum_{j=1}^{m_\lambda} \One[|f^\st(\tilde x_j^\lambda) -\tilde y_j^\lambda| > \alpha_\lamp] > \frac{\alpha_\lam}{16\Lam \alpha_\lamp}$. (This is possible by the property in step \ref{it:f-high-err-2} of \HPL.) Set $\lampp = \lamp + 3\in [\Lam]$. Then $\frac{1}{m_\lam} \sum_{j=1}^{m_\lam} \One[|f^\st(\til x_j^\lam) - \til y_j^\lam| > 7 \alpha_\lampp] > \frac{1}{8} \cdot \frac{\alpha_\lam}{16\Lam \alpha_\lampp}$. Note that if $|f^\st(\til x_j^\lam) - \til y_j^\lam| > 7 \alpha_\lampp$, then $f^\st(\til x_j^\lam) \in [b\alpha_\lampp, (b+1)\alpha_\lampp)$ for some $b$ satisfying $|b\alpha_\lampp - \til y_j^\lam| > 6 \alpha_\lampp$. 
  
  Therefore,  for any $v_\lampp$ such that $f^\st \in \MG_\lampp^{v_\lampp}$, there are at least $\frac{\alpha_\lam m_\lam}{128 \Lam \alpha_\lampp}$ tuples $(j,b) \in [m_\lam] \times \{0, 1, \ldots, \lfloor \alpha_\lampp \rfloor + 1\}$ so that $f^\st \in \MG_{\lampp,\lam}^{v_\lampp,j,b}$: in particular, these correspond to the (at least) $\frac{\alpha_\lam m_\lam}{128 \Lam \alpha_\lampp}$ values of $j$ for which $|f^\st(\til x_j^\lam) - \til y_j^\lam| > 7\alpha_\lampp$, each of which is handled as follows:
  \begin{itemize}
  \item If $|\soal{\MG_\lampp^{v_\lampp}}{\alpha_\lampp}{\til x_j^\lampp} - \til y_j^\lampp| \leq 5 \alpha_\lampp$, then as we have remarked above, there is some $b$ satisfying $|b\alpha_\lampp - \til y_j^\lam | > 6\alpha_\lampp$ so that $f^\st \in \MG_{\lampp,\lam}^{v_\lampp,j,b} = \MG_\lampp^{v_\lampp}|_{(\til x_j^\lam, b\alpha_\lampp)}^{\alpha_\lampp}$ (this corresponds to step \ref{it:2fail-soa-close} of \HPL).
  \item Otherwise, we have $f^\st \in \MG_{\lampp,\lam}^{v_\lampp,j,0} = \MG_\lampp^{v_\lampp}$ (this corresponds to step \ref{it:2fail-soa-far} of \HPL). 
  \end{itemize}
  Since $\MT_\lampp$ after step $n$ contains $A$ copies of the collection $\MT_\lampp$ before step $n$ in addition to the tuples $(\MG_{\lampp,\lam}^{v_\lampp,j,b}, w_{\lampp,\lam}^{v_\lampp,j})$ for each $\lam \in [\Lam]$, $j \in [m_\lam], b \in \{0, 1, \ldots, \lfloor 1/\alpha_\lampp \rfloor + 1\}, v_\lampp \in [|\MT_\lampp|]$ (so that $\MG_{\lampp,\lam}^{v_\lampp,j,b} \neq \emptyset$), it follows that the number of tuples $((\MG_\lampp^{v_\lampp})', (w_\lampp^{v_\lampp})') \in \MT_\lampp$ so that $f^\st \in (\MG_\lampp^{v_\lampp})'$ is at least $$q_\lampp \cdot \left( A + \frac{\alpha_\lam m_\lam}{128 \Lam \alpha_\lampp}\right) \geq q_\lampp \cdot \left( A + \frac{\mu_0}{128 \Lam \alpha_\lampp} \right).$$
\end{proof}

  For $n \in \BN$, let $N_n$ be the number of rounds, up to and including round $n$, for which the condition in step \ref{it:ms-voting-2} fails. Also let $T_n$ be the value of $t$ immediately before executing round $n$ of the algorithm. Finally, let $\MS_n$ denote the set of rounds $n' \leq n$ for which the condition in step \ref{it:ms-voting-2} holds. Note that $|\MS_n| = T_n$.

The following lemma uses the previous lemmas of this section to bound various parameters that show up in our error bounds. The first two bounds (on $\delta_t$ and $\log(W_{\lam,n'-1}/W_{\lam,n'})$) will be used together with Lemma \ref{lem:vote-slow-change-2} to bound the error of the predictors $\bar f_t$, and the final bound (on $N_n$) shows that \HPL terminates after a finite number of steps. %
\begin{lemma}
  \label{lem:bound-potentials}
  For each $n \in \BN$, it holds that
  \begin{align}
    \sum_{t=1}^{T_n} \delta_t \leq & 32 \alpha_\Lam T_n + 64 \Lam \log(1/\gamma) \cdot \sum_{\lam \in [\Lam]} \alpha_\lam \cdot \sfat_{\alpha_\lam}(\MF)\nonumber\\
    \sum_{n' \in \MS_n}\sum_{\lam \in [\Lam]}\alpha_\lam \cdot \log \left( \frac{W_{\lam,n'-1}}{W_{\lam,n'}} \right) \leq & \log(1/\gamma) \cdot \sum_{\lam \in [\Lam]} \alpha_\lam \cdot \sfat_{\alpha_\lam}(\MF)\nonumber\\    
    N_n \leq & 1024\Lam \cdot \frac{A}{\mu_1} \cdot \log(1/\gamma) \sum_{\lam\in[\Lam]} \alpha_\lam \cdot \sfat_{\alpha_\lam}(\MF)\nonumber.
  \end{align}
\end{lemma}
\begin{proof}
  The realizability assumption gives us that for each $t$, $y_t = f^\st(x_t)$. Thus, in each round $n$ in which the condition in step \ref{it:ms-voting-2} holds, for each $\lam \in [\Lam]$ and each $(\MG_\lam^{v_\lam}, w_\lam^{v_\lam}) \in \MT_\lam$ for which $f^\st \in \MG_\lam^{v_\lam}$ at the beginning of round $n$, after restricting $\MG_\lam^{v_\lam} \gets \MG_\lam^{v_\lam}|^{\alpha_\lam}_{(x_t, y_t)}$, it still holds that $f^\st \in \MG$. For each round $n$ in which the condition in step \ref{it:ms-voting-2} fails, for each $\lam \in [\Lam]$, if $\MT_\lam$ has $q_\lam$ tuples $(\MG_\lam^{v_\lam}, w_\lam^{v_\lam})$ so that $f^\st \in \MG_\lam^{v_\lam}$ at the beginning of round $n$, then step \ref{it:make-tlamp} ensures that after round $n$, $\MT_\lam$ has $A\cdot q_\lam$ tuples $(\MG_\lam^{v_\lam}, w_\lam^{v_\lam})$ so that $f^\st \in \MG_\lam^{v_\lam}$. Moreover, we have the following two facts:
  \begin{itemize}
  \item By Lemma \ref{lem:2fail-q-grow}, for each round $n$ in which the condition in step \ref{it:ms-voting-2} fails, there is some value of $\lamp = \lamp(n) \in [\Lam]$ so that after round $n$, $\MT_\lamp$ has at least $q_\lamp \cdot \left( A + \frac{\mu_0}{128\Lam \alpha_\lamp} \right)$ tuples $(\MG_\lam^{v_\lam}, w_\lam^{v_\lam})$ with $f^\st \in \MG_\lam^{v_\lam}$. (If the condition in step \ref{it:ms-voting-2} holds in round $n$, $\lamp(n)$ is not defined; we write $\lamp(n) = \perp$.) For each $\lam \in [\Lam]$, and $n \geq 1$, let $N_{\lam, n}$ denote the number of rounds $n' \leq n$, for which $\lamp(n') = \lam$. Note that $\sum_{\lam \in [\Lam]} N_{\lam, n} = N_n$.
  \item By Lemma \ref{lem:2suc-w-dec}, for each round $n$ in which the condition in step \ref{it:ms-voting-2} holds, if we let $t$ be the value of $t = t(n)$ at step \ref{it:receive-yt-2}, then the following holds: if $\delta_t > 32 \alpha_\Lam$, then for some $\lampp  = \lampp(n) \in [\Lam]$, $W_{\lampp,n} \leq W_{\lampp,n-1} \cdot \left( 1 - \frac{\delta_t}{64\Lam \alpha_\lampp} \right) \leq W_{\lampp,n-1} \cdot \exp \left( \frac{-\delta_t}{64 \Lam \alpha_\lampp} \right)$. (If $\delta_t \leq 32 \alpha_\Lam$ or the condition in step \ref{it:ms-voting-2} fails, then $\lampp(n)$ is not defined; we write $\lampp(n) = \perp$ for such $n$.) For each $\lam \in [\Lam]$ and $n \geq 1$, let $\MS_{\lam, n}$ denote the set of rounds $n' \leq n$ for which $\lampp(n') = \lam$. %
  \end{itemize}
 By definition of $N_{\lam,n}$, $\MT_\lam$ has at least $A^{N_n} \cdot \left(1 + \frac{\mu_0}{128 \Lam \alpha_\lam A} \right)^{N_{\lam,n}}$ tuples $(\MG_\lam^{v_\lam}, w_\lam^{v_\lam})$ with $f^\st \in \MG$. Combining this fact with Lemma \ref{lem:w-sfat-lb}, we get that the total weight of tuples in $\MT_\lam$ is lower bounded as follows:
  \begin{align}
W_{\lam,n} \geq A^{N_n} \cdot \left(1 + \frac{\mu_0}{128 \Lam \alpha_\lam A} \right)^{N_{\lam,n}} \cdot \gamma^{\sfat_{\alpha_\lam}(\MF)}.\label{eq:reg-wn-lb}
  \end{align}
  We next proceed to compute an upper bound on $W_{\lam,n}$. By Lemma \ref{lem:2suc-w-dec}, for all rounds $n'$ in which the condition in step \ref{it:ms-voting-2} holds, we have $W_{\lam,n} \leq W_{\lam,n-1}$. Further, by Lemma \ref{lem:2fail-w-ub}, for all rounds $n'$ in which the condition in step \ref{it:ms-voting-2} fails, we have $W_{\lam, n} \leq W_{\lam,n-1} \cdot \left( A+ \frac{3\mu_1 \Lam}{\alpha_\lam \hvconst}\right)$ for all $\lam \in [\Lam]$. Combining these facts with the definition of $\MS_{\lam,n}$ above, we get that
  \begin{align}
    W_{\lam,n} \leq &A^{N_n} \cdot \left(1 + \frac{3 \mu_1 \Lam}{\alpha_\lam \hvconst A} \right)^{N_n} \cdot \prod_{n' \in \MS_n} \frac{W_{\lam,n'}}{W_{\lam,n'-1}}  \leq  A^{N_n} \cdot \left(1 + \frac{3 \mu_1 \Lam}{\alpha_\lam \hvconst A} \right)^{N_n} \cdot \prod_{n' \in \MS_{\lam,n}} \exp \left( \frac{-\delta_{t(n')}}{64\Lam \alpha_\lam}\right)\label{eq:reg-wn-ub}
  \end{align}
  Combining (\ref{eq:reg-wn-lb}) and (\ref{eq:reg-wn-ub}), we obtain that, for each $\lam \in [\Lam]$,
  \begin{align}
    \sfat_{\alpha_\lam}(\MF) \cdot \ln(1/\gamma) \geq & \sum_{n' \in \MS_{n}} \ln \left( \frac{W_{\lam,n'-1}}{W_{\lam,n'}}\right)  + N_{\lam, n} \cdot \ln \left( 1 + \frac{\mu_0}{128\Lam \alpha_\lam A}\right) - N_n \cdot \ln \left( 1 + \frac{3\mu_1 \Lam}{\alpha_\lam \hvconst A} \right) \nonumber\\
    \geq & \sum_{n' \in \MS_{n}} \ln \left( \frac{W_{\lam,n'-1}}{W_{\lam,n'}}\right) + N_{\lam,n} \cdot \frac{\mu_1}{512 \Lam \alpha_\lam A} - N_n \cdot \frac{3 \mu_1 \Lam}{\alpha_\lam \hvconst A}\label{eq:bound-lns-0}\\
    \geq & \sum_{n' \in \MS_{\lam,n}} \frac{\delta_{t(n')}}{64\Lam \alpha_\lam} + N_{\lam,n} \cdot \frac{\mu_1}{512 \Lam \alpha_\lam A} - N_n \cdot \frac{3 \mu_1 \Lam}{\alpha_\lam \hvconst A}\label{eq:bound-lns},
  \end{align}
  where the inequality (\ref{eq:bound-lns-0}) uses that $A \geq  \frac{\mu_0}{\alpha_\lam}$ for all $\lam \in [\Lam]$ and $\mu_0 \geq \mu_1/2$. Note that our choice of $\hvconst = \hvval$ gives
  \begin{align}
\sum_{\lam \in [\Lam]} \left( N_{\lam, n} \cdot \frac{\mu_1}{512\Lam  A_\lam} - N_n \cdot \frac{3 \mu_1\Lam}{ \hvconst A_\lam}\right) = N_n \cdot \left( \frac{\mu_1}{512\Lam A} - \frac{3 \mu_1 \Lam^2}{\hvconst A} \right) \geq N_n \cdot \frac{\mu_1}{1024 \Lam A} > 0. 
  \end{align}
  Thus, multiplying (\ref{eq:bound-lns}) by $\alpha_\lam$ and summing it over $\lam \in [\Lam]$ gives that the following hold: 
  \begin{align}
    \sum_{t=1}^{T_n} \delta_t \leq &  32 \alpha_\Lam T_n + \sum_{\lam \in [\Lam]} \sum_{n' \in \MS_{\lam,n}} \delta_{t(n')}\leq 32 \alpha_\Lam T_n + 64 \Lam \ln(1/\gamma) \cdot \sum_{\lam \in [\Lam]} \alpha_\lam \cdot \sfat_{\alpha_\lam}(\MF)\nonumber \\
    \sum_{n' \in \MS_n}\sum_{\lam \in [\Lam]} \alpha_\lam \cdot \ln \left( \frac{W_{\lam,n'-1}}{W_{\lam,n'}} \right) \leq & \ln(1/\gamma) \cdot \sum_{\lam \in [\Lam]} \alpha_\lam \cdot \sfat_{\alpha_\lam}(\MF)\nonumber\\
    N_n \leq & 1024\Lam\cdot \frac{A}{\mu_1} \cdot \ln(1/\gamma) \sum_{\lam\in[\Lam]} \alpha_\lam \cdot \sfat_{\alpha_\lam}(\MF)\nonumber.
  \end{align}
\end{proof}

Fix any smoothing parameter $\eta > 0$ so that $1/\eta \in \BN$, and write $\xi := 1/\eta$. Consider the distributions $\bar f_1, \ldots, \bar f_T$ output by \HPL. Fix arbitrary distributions $\bar f_0, \ldots, \bar f_{1-\xi} \in \Delc(\MF)$. %
We define the \emph{$\eta$-stabilized distributions} $\bar h_1, \ldots, \bar h_T$ as follows: for $t \in [T]$, we define
\begin{align}
  \label{eq:def-eta-smoothed}
\bar h_t = \eta \sum_{s=0}^{\xi-1} \bar f_{t-s}.
\end{align}
Clearly it holds that $\tvnorm{ \bar h_t - \bar h_{t+1}}\leq 2\eta$ for all $T \in [T-1]$, and $\bar h_t \in \Delc(\MF)$ for each $t \in [T]$. The below lemma gives a bound on the \cumls for the sequence $\bar h_t$:
\begin{theorem}[Near-optimal stable proper learning]
  \label{thm:proper-stable-main}
Fix any $\eta > 0$. In the realizable setting, the $\eta$-stabilized distributions $\bar h_t$ defined in (\ref{eq:def-eta-smoothed}) from the output of \HPL (Algorithm \ref{alg:hpl}) satisfy:
  \begin{align}
\sum_{t=1}^T \E_{h \sim \bar h_t} \left[|h(x_t) - y_t|\right] \leq \frac{1}{\eta} \cdot O \left( \log^6(T) \cdot \min_{\alpha \in [0,1]} \left\{ \alpha T + \int_\alpha^1  \sfat_\eta(\MF) d\eta \right\}\right)\label{eq:proper-learner-integral}.
  \end{align}
  Further, $\tvnorm{\bar h_t - \bar h_{t+1}} \leq 2\eta$ for all $t \in [T-1]$ and $\bar h_t \in \Delc(\MF)$ for all $T \in [T]$.
\end{theorem}
\begin{proof}
As in the proof of Proposition \ref{prop:reg-realizable}, choose $\alpha \in [1/T,1]$ minimizing the expression on the right-hand side of (\ref{eq:proper-learner-integral}), and set $\Lam = \lfloor 1/(2\alpha) \rfloor \leq \log T$. Since $\hvconst = \hvval = O(\log^3 T)$ and $\log(1/\gamma) = O(\log T)$, it suffices to show
    \begin{align}
\sum_{t=1}^T \E_{h \sim \bar h_t} \left[ |h(x_t) - y_t|\right] \leq O \left( \frac{\hvconst \cdot \Lam^2}{\eta} \right) \cdot \left( \alpha_\Lam T + \log(1/\gamma) \sum_{\lam=1}^\Lam \alpha_\lam \sfat_{\alpha_\lam}(\MF) \right),\nonumber
    \end{align}
    when \HPL is run with the chosen scale parameter $\Lam$.

  For $t \in [T]$ and $0 \leq s < \xi$, let $\MM_{0,t,s}$ be the set of rounds $n$ starting at the round when $x_t$ is observed and up to (but not including) the round where $x_{t+s}$ is observed, so that the condition in step \ref{it:ms-voting-2} holds in round $n$. Let $\MM_{1,t,s}$ be the set of such rounds $n$ (i.e., starting at $x_t$, and up to but not including $x_{t+s}$) for which the condition in step \ref{it:ms-voting-2} fails in round $n$. Let $\MM_0$ be the set of all rounds $n$ (up to, and including, the round that $x_T$ is observed) for which the condition in step \ref{it:ms-voting-2} holds in round $n$, and let $\MM_1$ be the set of all rounds $n$ for which the condition in step \ref{it:ms-voting-2} fails in round $n$. Furthermore, recall the definition of $\delta_t$ in step \ref{it:receive-yt-2} of Algorithm \ref{alg:hpl}. 
  By the definition of $\bar h_t$, we have 
  \begin{align}
    & \sum_{t=1}^T \E_{h \sim \bar h_t} \left[|h(x_t) - y_t|\right] \nonumber\\
    \leq & \eta \sum_{t=1}^T\sum_{s=0}^{\xi-1} \E_{f \sim \bar f_{t-s}} \left[| f(x_t) - y_t |\right] \nonumber\\
    \leq & \xi + \eta \sum_{t=1}^{T-\xi+1}\sum_{s=0}^{\xi-1} \E_{f \sim \bar f_t} \left[| f(x_{t+s}) - y_{t+s} |\right] \nonumber \\
    \leq & \xi + \eta \sum_{t=1}^{T-\xi+1} \sum_{s=0}^{\xi-1} \left[ 80 \cdot \hvconst \sum_{n' \in \MM_{0,t,s}}\sum_{\lam \in [\Lam]} \alpha_\lam \cdot \ln \left( \frac{W_{\lam,n'-1}}{W_{\lam,n'}}\right) + \frac{|\MM_{1,t,s}|}{A} \cdot 150 \hvconst\mu_1 \Lam^2 + 5 \hvconst \cdot \delta_{t+s}\right]\label{eq:apply-lazy-lemma}\\
    \leq & \xi + \frac{80 \hvconst}{\eta} \cdot \sum_{n' \in \MM_0} \sum_{\lam \in [\Lam]} \alpha_\lam \ln \left( \frac{W_{\lam,n'-1}}{W_{\lam,n'}}\right) + \frac{|\MM_1| \cdot 150 \hvconst \mu_1 \Lam^2}{A\eta} + 5\hvconst \sum_{t=1}^T \delta_t\label{eq:ts-exchange}\\
    \leq & \xi + \frac{80 \hvconst}{\eta} \cdot \log(1/\gamma) \sum_{\lam=1}^\Lam \alpha_\lam \sfat_{\alpha_\lam}(\MF) + 1024\Lam \log(1/\gamma) \cdot \frac{150\hvconst\Lam^2}{\eta}\sum_{\lam=1}^\Lam \alpha_\lam \sfat_{\alpha_\lam}(\MF)\label{eq:bound-by-sfat}\\
    &+ 5\hvconst \cdot \left(32 \alpha_\Lam T + 64 \Lam \log(1/\gamma) \sum_{\lam=1}^\Lam \alpha_\lam \sfat_{\alpha_\lam}(\MF)\right)\nonumber\\
    \leq & O \left( \frac{\hvconst \cdot \Lam^2}{\eta} \right) \cdot \left( \alpha_\Lam T + \log(1/\gamma) \sum_{\lam=1}^\Lam \alpha_\lam \sfat_{\alpha_\lam}(\MF) \right)\nonumber,
  \end{align}
  where:
  \begin{itemize}
  \item (\ref{eq:apply-lazy-lemma}) follows from Lemma \ref{lem:vote-slow-change-2};
  \item (\ref{eq:ts-exchange}) follows from exchanging the order of summation and noting that for each $n' \in \MM_0$, there are at most $\xi^2$ values of $(t,s)$ so that $n' \in \MM_{0,t,s}$ and for each $n' \in \MM_1$, there are at most $\xi^2$ values of $(t,s)$ so that $n' \in \MM_{1,t,s}$;
  \item (\ref{eq:bound-by-sfat}) follows from Lemma \ref{lem:bound-potentials}; notice that we have used here that $|\MM_1| = N_n$, where $n$ is total number of iterations of the outer while loop of \HPL. 
  \end{itemize}
\end{proof}

\section{Path-length regret bound for a stable proper learner}
\label{sec:stable-proper}
In this section we prove Theorem \ref{thm:path-stable}, obtaining a proper agnostic learner that gets a path-length regret bound. As we do throughout the paper, we assume that the given function class $\MF$ has finite sequential fat-shattering dimension at all scales.
\begin{theorem}[Path-length regret bound for a stable online learner]
  \label{thm:path-stable}
  Suppose that $\alpha$ is chosen so that $1 \leq \alpha T \leq \sfat_\alpha(\MF)$ and $\alpha \leq \kappa$. Moreover suppose that for all $t < T$, the examples $(x_t, y_t)$ satisfy $\infnorms{ x_t - x_{t+1}}{\MF} \leq \kappa$ and $|y_t - y_{t+1}| \leq \kappa$. Then, for any $\Gamma \geq 1$ \OSE with step size $\etaOH = \etaPSR = \kappa/\Gamma$ obtains a regret of
  \begin{align}
    \sum_{t=1}^T \E_{f_t \sim \bar f_t}[|f_t(x_t) - y_t|] - \sum_{t=1}^T |f_\st(x_t) - y_t| \leq &  O \left( \frac{\Gamma \cdot \sfat_\alpha(\MF) \cdot \log^6 T}{\kappa} + \frac{\kappa^3 \cdot T}{\Gamma} \right)\label{eq:kappa-regret}
  \end{align}
  Further, for any choices of $\etaOH, \etaPSR > 0$ (and without restriction on the $(x_t, y_t)$), the iterates $\bar f_t$ of \OSE belong to $\Delc(\MF)$ and are stable in the following sense: for all $t < T$,
  \begin{align}
\tvnorm{ \bar f_t - \bar f_{t+1}} \leq & 5\etaOH + 3 \etaPSR \nonumber. %
  \end{align}
  
  In particular, if we have $\Gamma \etaOH = \Gamma \etaPSR = \kappa = (\sfat_\alpha(\MF)/T)^{1/4} \cdot \log^{3/4}(T)$, then the regret is bounded above by $\tilde O \left( \Gamma \cdot \sfat_\alpha(\MF)^{3/4} \cdot T^{1/4} \right)$.
\end{theorem}
The main ingredient in the proof is Theorem \ref{thm:proper-stable-main} from the previous section which gives an (optimal) stable and proper learner for the setting of realizable online regression. Given this result, the proof of Theorem \ref{thm:proper-stable-main} is mostly standard, using results of \cite{ben-david_agnostic_2009,rakhlin_online_2015} and \cite{syrgkanis_fast_2015}. %

\subsection{Defining the experts} As discussed in Section \ref{sec:techniques}, the general idea of the proof is to use the SOA-experts framework of \cite{ben-david_agnostic_2009,rakhlin_online_2015}.\footnote{For references in this section to \cite{rakhlin_online_2015}, see in particular the version at \url{https://arxiv.org/pdf/1006.1138v1.pdf}.} We begin by defining the experts in this setting in Definition \ref{def:experts} below. Let $\MX^\st$ be the set of all finite sequences of elements of $\MX$. Each expert is a function $E : \MX^\st \ra [0,1]$; $E(x_1, \ldots, x_t)$ should be interpreted as the label that the expert $E$ predicts for $x_t$ given that it has already seen $x_1, \ldots, x_{t-1}$. 
\begin{defn}[\cite{rakhlin_online_2015}]
  \label{def:experts}
  Fix $T \in \BN$, any $\MF \subset [0,1]^\MX$ and $\alpha \in (0,1)$, and set $d_\alpha := \sfat_\alpha(\MF)$. For each tuple $(I, \sigma)$, where $I$ is a subset $I \subset [T]$ of size $|I| \leq d_\alpha$, and $\sigma \in \{0, 1, \ldots, \lceil 1/\alpha \rceil - 1 \}^{|I|}$, define the expert $E_{(I,\sigma)} : \MX^\st \ra [0,1]$ by
  $$
E_{(I,\sigma)}(x_1, \ldots, x_t) = \soal{\MF(t)}{\alpha}{x_t},
$$
where $\MF(t)$ is defined inductively via $\MF(1) = \MF$ and
\begin{align}
  \MF(t+1) = \begin{cases}
    \MF(t) \qquad : t \not \in I \\
    \MF(t)|^\alpha_{(x_t, \sigma_{i_t} \cdot \alpha)} \qquad : t \in I,
    \end{cases}\nonumber
\end{align}
where for $t \in I$, $i_t \in \{ 1, \ldots, |I| \}$ is defined so that $t$ is the $i_t$th smallest element of $I$. We denote the set of experts $E_{(I,\sigma)}$ given $T, \alpha$ by $\CE_{T,\alpha}$ (the class $\MF$ is implicit in our notation).
\end{defn}
The set of all experts $E_{(I,\sigma)}$ of Definition \ref{def:experts} can be seen as an algorithmic version of a sequential cover \cite[Definition 4]{rakhlin_sequential_2015}. Lemma \ref{lem:count-experts} bounds the number of experts in $\CE_{T,\alpha}$. 
\begin{lemma}
  \label{lem:count-experts}
Given $T \in \BN$, $\MF \subset [0,1]^\MX$ and $\alpha \in (0,1)$, the number of experts in the set $\CE_{T,\alpha}$ of Definition \ref{def:experts} is at most $\left( \frac{2eT}{\alpha} \right)^{\sfat_\alpha(\MF)}$.
\end{lemma}
\begin{proof}
  The number of experts $(I,\sigma)$ is at most
  $$
\sum_{s=1}^{\sfat_\alpha(\MF)} {T \choose s} \cdot \left \lceil \frac{1}{\alpha} \right\rceil^s \leq \left( \frac{2eT}{\alpha} \right)^{\sfat_\alpha(\MF)}.
$$
\end{proof}

Lemma \ref{lem:cover-experts} shows that the set of experts covers the class $\MF$ in an online sense.
\begin{lemma}[Lemma 15, \cite{rakhlin_online_2015}]
  \label{lem:cover-experts}
  Given $T \in \BN$, $\MF \subset [0,1]^\MX$, and $\alpha \in (0,1)$, for each $f \in \MF$ and any sequence $(x_1, \ldots, x_T)$, there exists some expert $E \in \CE_{T,\alpha}$ so that for all $t \in [T]$,
  \begin{align}
 | f(x_t) - E(x_1, \ldots, x_t)| \leq \alpha\nonumber.
  \end{align}
\end{lemma}

Since Lemma \ref{lem:cover-experts} only proimses that some expert has error $\alpha$ with respect to any given hypothesis $f$, yet Lemma Theorem \ref{thm:proper-stable-main} (for proper learning) requires that its input sequence be \emph{exactly} realizable, we need to work with the $\alpha$-augmented class for a given class $\MF$, defined below. 
\begin{defn}[$\alpha$-augmented class]
  \label{def:augmented}
  For a real-valued class $\MF \subset [0,1]^\MX$ and $\alpha > 0$, define the \emph{$\alpha$-augmented class} $\MF^\alpha$ by
  \begin{align}
\MF^\alpha := \left\{ f' \in [0,1]^\MX : \ \exists f \in \MF \mbox{ such that } \infnorms{ f' - f }{\MX} \leq \alpha \right\}.\nonumber
  \end{align}
\end{defn}
For each element $f' \in \MF^\alpha$, fix some element $\can{f'}{\alpha} \in \MF$ (a ``canonical element'') so that $\infnorms{f' - f}{\MX} \leq \alpha$. We may extend this definition to elements of $\Delc(\MF^\alpha)$ as follows: for $\bar f' = \sum_{i=1}^K w_i \cdot \delta_{f_i'}$, $f_i' \in \MF^\alpha$, set $\can{\bar f'}{\alpha} := \sum_{i=1}^K w_i \cdot \delta_{\can{f_i'}{\alpha}}$. This definition will be used to ensure stability of the learner \OSE.

Lemma \ref{lem:augmented-fs} bounds the $\alpha$-sequential fat-shattering dimension of an augmented class in terms of that of the original class. 
\begin{lemma}
  \label{lem:augmented-fs}
For any class $\MF \subset [0,1]^\MX$, and any $\alpha > 0$, it holds that, for all $\alpha' \geq 4\alpha$, $\sfat_{\alpha'}(\MF^\alpha) \leq \sfat_{\alpha'/2}(\MF)$.
\end{lemma}
\begin{proof}
  Set $d_0 := \sfat_{\alpha'/2}(\MF)$. 
Suppose for the purpose of contradiction that there were some trees $\bx, \bs$ of depth $d > d_0$ so that for all $k_{1:d} \in \{-1,1\}^d$, there is some $f \in \MF^\alpha$ so that $k_t \cdot (f(\bx_t(k_{1:t-1})) - \bs_t(k_{1:t-1})) \geq \alpha'/2$ for all $t \in [d]$. Then there is some $f' \in \MF$ so that $k_t \cdot (f'(\bx_t(k_{1:t-1})) - \bs_t(k_{1:t-1})) \geq \alpha'/2 - \alpha \geq \alpha'/4$ for all $t \in [d]$, i.e., the trees $\bx,\bs$ witness an $\alpha'/2$-shattering of $\MF$, which is a contradiction to $d_0 < d$.
\end{proof}

\begin{algorithm}[!htp]
  \caption{\bf \OSE}\label{alg:ose}
  \KwIn{Function class $\MF \subset [0,1]^\MX$, time horizon $T \in \BN$, scale $\alpha > 0$, step size $\eta > 0$.} 
  \begin{enumerate}[leftmargin=14pt,rightmargin=20pt,itemsep=1pt,topsep=1.5pt]
  \item Set $\MF^\alpha := \{ f \in [0,1]^\MX : \ \exists f' \in \MF \mbox{ such that } \infnorms{ f' - f }{\MX} \leq \alpha \}$.
  \item Initialize $\omega_{E,1} = 1/|\CE_{T,\alpha}|$ for all $E \in \CE_{T,\alpha}$.
  \item For $1 \leq t \leq T$:
    \begin{enumerate}
    \item For each $E \in \CE_{T,\alpha}$, define $\bar h_t(E) \in \Delc(\MF^\alpha)$ to be the $\etaPSR/2$-smoothed hypotheses (as defined in (\ref{eq:def-eta-smoothed})) of the output of  \HPL (Algorithm \ref{alg:hpl}) %
      given the class $\MF^\alpha$, the parameter $\Lam = \lfloor \log 1/(4\alpha) \rfloor$, and the input sequence $(x_1, E(x_1)), (x_2, E(x_{1:2})), \ldots, (x_{t-1}, E(x_{1:t-1}))$.
    \item For each $E \in \CE_{T,\alpha}$, set $\bar g_t(E) := \can{\bar h_t(E)}{\alpha}\in \Delc(\MF)$ to be canonical randomized hypothesis for $\bar h_t(E)$. %
    \item Predict the hypothesis $\bar f_t := \sum_{E \in \CE_{T,\alpha}} \omega_{E,t} \cdot \bar g_t(E) \in \Delc(\MF)$.
    \item Receive $(x_t, y_t)$, draw $f_t \sim \bar f_t$ and suffer loss $|f_t(x_t) - y_t|$.
    \item For each expert $E \in \CE_{T,\alpha}$, compute the loss $\ell_t(E) := \E_{g \sim \bar g_t(E)} \left[|g(x_t) - y_t|\right]$.
    \item \label{it:omwu-omega} Update the weights $\{ \omega_{E,t} \}_{E \in \CE_{T,\alpha}}$ using Optimistic Exponential Weights, i.e.,
      \begin{align}
\omega_{E,t+1} := \frac{\omega_{E,t} \cdot \exp \left( -\eta \cdot (2\ell_t(E) - \ell_{t-1}(E)) \right)}{\sum_{E' \in \CE_{T,\alpha}}  \omega_{E',t}\cdot \exp \left( -\eta \cdot (2\ell_t(E') - \ell_{t-1}(E')) \right)}.\nonumber
      \end{align}
    \end{enumerate}
  \end{enumerate}
\end{algorithm}

\begin{lemma}
  \label{lem:losses-close}
  For any $t < T$, the following hold:
  \begin{itemize}
  \item For all $E \in \CE_{T,\alpha}$, $\tvnorm{ \bar g_t(E) - \bar g_{t+1}(E) } \leq \etaPSR$;%
  \item Suppose that $\infnorms{ x_t - x_{t+1}}{\MF} \leq \kappa$ and $|y_t - y_{t+1}| \leq \kappa$, and that \PSR is run with step size $\etaPSR$. Then, for all $E \in \CE_{T,\alpha}$, $| \ell_t(E) - \ell_{t+1}(E) | \leq 2\kappa + \etaPSR$. %
  \end{itemize}
\end{lemma}
\begin{proof}
By Theorem \ref{thm:proper-stable-main}, we have that $\tvnorm{ \bar h_t(E) - \bar h_{t+1}(E)} \leq \etaPSR$ for all experts $E$. Thus, by the definition of $\can{\cdot}{\alpha}$ and the data processing inequality, $\tvnorm{ \bar g_t(E) - \bar g_{t+1}(E) }\leq \etaPSR$ for all experts $E$. %
  
  We may now compute 
  \begin{align}
    | \ell_t(E) - \ell_{t+1}(E) | =& \left|\ \E_{g \sim \bar g_t(E)}[|g(x_t) - y_t|] - \E_{g \sim \bar g_{t+1}(E)}[|g(x_{t+1}) - y_{t+1}|]\ \right|\nonumber\\
    \leq & \tvnorm{\bar g_t(E) - \bar g_{t+1}(E) } + \left| \ \E_{g \sim \bar g_{t+1}(E)} \left[ |g(x_t) - y_t| - |g(x_{t+1}) - y_{t+1}| \right] \ \right|\nonumber\\
    \leq &  \tvnorm{\bar g_t(E) - \bar g_{t+1}(E) } + |y_t - y_{t+1}| + \E_{g \sim \bar g_{t+1}(E)} \left[ |g(x_t) - g(x_{t+1}) | \right]\nonumber\\
    \leq &  \tvnorm{\bar g_t(E) - \bar g_{t+1}(E) } + |y_t - y_{t+1}| + \infnorms{x_t - x_{t+1}}{\MF}\nonumber\\
    \leq & 2\kappa + \etaPSR\nonumber,
  \end{align}
  where the final inequality uses that $\tvnorm{ \bar g_t(E) - \bar g_{t+1}(E) } \leq \etaPSR$, $|y_t - y_{t+1}| \leq \kappa$, and $\infnorms{x_t - x_{t+1}}{\MF} \leq \kappa$.
\end{proof}

Finally, we are ready to prove Theorem \ref{thm:path-stable}.
\begin{proof}[Proof of Theorem \ref{thm:path-stable}]
  Without loss of generality we may assume that $\kappa \geq (\sfat_\alpha(\MF)/T)^{1/4} \cdot \log^{3/4}T$ (since the expression on the right-hand side of (\ref{eq:kappa-regret}) is minimized at $\kappa = (\sfat_\alpha(\MF)/T)^{1/4} \cdot \log^{3/4}T$, meaning that we can make $\kappa$ larger if it is less than $(\sfat_\alpha(\MF)/T)^{1/4} \cdot \log^{3/4}T$).
  
  By Lemma \ref{lem:losses-close} and the fact that $\max\{\alpha ,\etaOH, \etaPSR \}\leq \kappa$, we have that for each $t < T$ and each expert $E \in \CE_{T,\alpha}$, $|\ell_t(E) - \ell_{t+1}(E)| \leq 5\kappa$.
  Set
\begin{align}
  f_\st = \argmin_{f \in \MF} \sum_{t=1}^T |f(x_t) - y_t |.\nonumber
\end{align}
By Lemma \ref{lem:cover-experts}, there is some expert $E_\st \in \CE_{T,\alpha}$ so that for all $t \in [T]$, $|f_\st(x_t) - E_\st(x_1, \ldots, x_t)| \leq \alpha$. Thus, there is some $f_\st^\alpha \in \MF^\alpha$ so that for all $t \in [T]$, $f_\st^\alpha(x_t) = E_\st(x_1, \ldots, x_t)$. By Theorem \ref{thm:proper-stable-main} with $\eta = \etaPSR/2$, it follows that
\begin{align}
  \sum_{t=1}^T \E_{g \sim \bar g_t(E_\st)} \left[ \left| g(x_t) - E_\st(x_1, \ldots, x_t) \right|\right]   \leq & \alpha T +  \sum_{t=1}^T \E_{h \sim \bar h_t(E_\st)} \left[\left| h(x_t) - E_\st(x_1, \ldots, x_t) \right|\right] \nonumber\\
  \leq & O \left( \frac{\log^6 T}{\etaPSR} \cdot \left( \alpha T + \sfat_{4\alpha}(\MF^\alpha) \right)\right)\nonumber\\
  \leq & O \left( \frac{\log^6 T}{\etaPSR} \cdot \left( \alpha T + \sfat_{\alpha}(\MF) \right)\right)\label{eq:best-expert-mb},
\end{align}
where the final inequality above follows from Lemma \ref{lem:augmented-fs}.

By \cite[Theorem 11]{syrgkanis_fast_2015}, we have that
\begin{align}
  \sum_{t=1}^T \E_{f_t \sim \bar f_t} \left[|f_t(x_t) - y_t|\right] = & \sum_{t=1}^T \sum_{E \in \CE_{T,\alpha}} \omega_{E,t} \cdot \E_{g \sim \bar g_t(E)} [|g(x_t) - y_t|]\nonumber\\
  \leq & \min_{E \in \CE_{T,\alpha}} \left\{ \sum_{t=1}^T \E_{g \sim \bar g_t(E)}[|g(x_t) - y_t|] \right\} + \frac{\log |\CE_{T,\alpha}|}{\etaOH} + \etaOH \cdot (5\kappa)^2 \cdot T \nonumber\\
  \leq & \sum_{t=1}^T \E_{g \sim \bar g_t(E_\st)}[|g(x_t) - y_t|] + \frac{\log |\CE_{T,\alpha}|}{\etaOH} + \etaOH \cdot (5\kappa)^2 \cdot T \nonumber\\
  \leq & \sum_{t=1}^T \E_{g \sim \bar g_t(E_\st)}[|g(x_t) - E_\st(x_1, \ldots, x_t)|] + \sum_{t=1}^T |E_\st(x_1, \ldots, x_t) - f_\st(x_t)| \label{eq:triangle-inequality-3}\\
  &+ \sum_{t=1}^T |y_t - f_\st(x_t)| + \frac{\log |\CE_{T,\alpha}|}{\etaOH} +  \etaOH \cdot (5\kappa)^2 \cdot T \nonumber\\
  \leq & \sum_{t=1}^T |y_t - f_\st(x_t)| + O \left( \frac{\log^6 T}{\etaPSR} \cdot \left( \alpha T + \sfat_{\alpha}(\MF) \right)\right) \label{eq:realizable-plugin}\\
  & + O \left(\frac{\sfat_\alpha(\MF) \cdot \log(T/\alpha) }{\etaOH}\right) +  \etaOH \cdot (5\kappa)^2 \cdot T \nonumber,
\end{align}
where (\ref{eq:triangle-inequality-3}) uses the triangle inequality and (\ref{eq:realizable-plugin}) uses (\ref{eq:best-expert-mb}) and Lemma \ref{lem:count-experts} (which bounds $|\CE_{T,\alpha}|$). 

By choosing $\etaOH = \etaPSR = \kappa/\Gamma \leq \kappa$ and using that $\alpha \geq 1/T$ and $\alpha T \leq \sfat_\alpha(\MF)$, we obtain
\begin{align}
  \sum_{t=1}^T \E_{f_t \sim \bar f_t}[|f_t(x_t) - y_t|] - \sum_{t=1}^T |f_\st(x_t) - y_t| \leq & O \left( \frac{\Gamma \cdot \log^6 T}{\kappa} \cdot (\alpha T + \sfat_\alpha(\MF)) + \frac{\kappa^3 \cdot T}{\Gamma} \right)\nonumber\\
  \leq &  O \left( \frac{\Gamma \cdot \sfat_\alpha(\MF) \cdot \log^6 T}{\kappa} + \frac{\kappa^3 \cdot T}{\Gamma} \right)\nonumber.
\end{align}
Finally, when $\kappa = \frac{1}{\Gamma} \cdot (\sfat_\alpha(\MF)/T)^{1/4} \cdot \log^{3/4}T$, we obtain
\begin{align}
  \sum_{t=1}^T \E_{f_t \sim \bar f_t}[|f_t(x_t) - y_t|] - \sum_{t=1}^T |f_\st(x_t) - y_t| \leq & \tilde O \left(\Gamma \cdot  \sfat_\alpha(\MF)^{3/4} \cdot T^{1/4} \right)\nonumber.
\end{align}

Since each $\bar f_t$ is a finite convex combination of the collection of $\bar g_t(E)$, each of which is an element of $\Delc(\MF)$, it holds that $\bar f_t \in \Delc(\MF)$ as well. 
Finally we bound the stability of the iterates $\bar f_t$: for any $t < T$,
\begin{align}
  \tvnorm{ \bar f_t -\bar f_{t+1} } =& \tvnorm{ \sum_{E \in \CE_{T,\alpha}} \left( \omega_{E,t} \cdot \bar g_t(E) - \omega_{E,t+1} \cdot \bar g_{t+1}(E) \right) } \nonumber\\
  \leq &  \tvnorm{\sum_{E \in \CE_{T,\alpha}} \left( \omega_{E,t} \cdot \bar g_t(E) - \omega_{E,t+1} \cdot \bar g_{t}(E) \right) } +  \tvnorm{ \sum_{E \in \CE_{T,\alpha}} \left( \omega_{E,t+1} \cdot \bar g_t(E) - \omega_{E,t+1} \cdot \bar g_{t+1}(E) \right) }\nonumber\\
  \leq & \sum_{E \in \CE_{T,\alpha}} \left|  \omega_{E,t} - \omega_{E,t+1} \right| + \sum_{E \in \CE_{T,\alpha}} \omega_{E,t+1} \cdot \tvnorm{ \bar g_t(E) - \bar g_{t+1}(E) } \nonumber\\
  \leq & (\exp(2\etaOH) - 1) \cdot \frac{1}{\exp(-2\etaOH)} +  \max_{E \in \CE_{T,\alpha}} \tvnorm{ \bar g_t(E) - \bar g_{t+1}(E)} \label{eq:use-omwu-close}\\
  \leq & 5\etaOH + 3\etaPSR \label{eq:final-stability},
\end{align}
where (\ref{eq:use-omwu-close}) follows from the Optimistic Exponential Weights updates in step \ref{it:omwu-omega} of Algorithm \ref{alg:ose}, and (\ref{eq:final-stability}) follows from the fact that $\exp(4\eta) - \exp(2\eta) \leq 5\eta$ for $0 < \eta \leq 1/4$, and Lemma \ref{lem:losses-close}. %

\end{proof}

\section{Fast rates for learning in games}
\label{sec:games}
In this section we present a key application of the stable proper learner \OSE in Section \ref{sec:stable-proper}: we show that when multiple agents in a game each run the algorithm \OSE, then they can converge to equilibrium at faster rates than the typical $1/\sqrt{T}$ ones.

\subsection{Problem setting: Littlestone games}
We begin by defining the notion of games we consider, which generalizes finite-action normal form games to the case of extremely large or infinite action spaces. Further, we focus on the case of games for which the payoff for each player is in $\{0,1\}$ under any pure strategy profile; our setup generalizes that of \cite{hanneke_online_2021}, which considered the special case of 2-player 0-sum Littlestone games.
\begin{defn}[General-sum Littlestone games]
  \label{def:lgame}
  Consider any integer $K \in \BN$, denoting the number of players, and sets $\MF_1, \ldots, \MF_K$. Write $\MF_{-k} := \prod_{j \in [K] \backslash \{ k\}} \MF_j$. A function $\ell : \MF_1 \times \cdots \MF_K \ra \{0,1\}$ is said to define a \emph{Littlestone payoff function} if the following holds: for each $k \in [K]$, the class
  \begin{align}
    \label{eq:fk-define}
\lgame{\MF}{k}{\ell}:= \left\{ f_{-k} \mapsto \ell(f_k, f_{-k}) \ : \ f_k \in \MF_k \right\} \subset \{0,1\}^{\MF_{-k}}
  \end{align}
  has finite Littlestone dimension. %
A \emph{Littlestone (general-sum)} game is a $K$-tuple of Littlestone payoff functions, namely a tuple $\ell = (\ell_1, \ldots, \ell_K)$. We say that \emph{Littlestone dimension of the game} is $\max_{k \in [K]} \{\Ldim(\lgame{\MF}{k}{\ell_k})\}$.
\end{defn}
For $k \in [K]$, the payoff function $\ell_k$ in Definition \ref{def:lgame} denotes the payoff function for player $k$ in the Littlestone game. %
It is immediate that all finite-action normal form games are Littlestone games.

Before proceeding, we make the following definition: for a class $\MF \subset \{0,1\}^\MX$, define a class $\mix{\MF} \subset [0,1]^{\Delc(\MX)}$, in bijection with $\MF$, as follows: for each $f \in \MF$, the corresponding $f \in \mix{\MF}$ is defined by, for $P \in \Delc(\MX)$, $f(P) := \E_{x \sim P}[f(x)]$.

To help describe how we apply \OSE in the context of learning in Littlestone games, we need to make an additional definition: for a Littlestone game $\ell = (\ell_1, \ldots, \ell_K)$ with action sets $\MF_1, \ldots, \MF_K$, for each $k \in [K]$, define the \emph{loss set} $\lgame{\ML}{k}{\ell_k}$ of player $k$ as  follows:
\begin{align}
\lgame{\ML}{k}{\ell_k} := \{ f_k \mapsto \ell(f_k, f_{-k}) \ : \ f_{-k} \in \MF_{-k}\} \subset \{0,1\}^{\MF_k}\nonumber.
\end{align}
In words, $\lgame{\ML}{k}{\ell_k}$ is the set of mappings from $f_k$ to $\{0,1\}$ which may be realized as the loss of player $k$ given some valid actions of all other players. To avoid confusion, we denote elements of $\lgame{\ML}{k}{\ell_k}$ with a capital $L$. It is evident that $\lgame{\ML}{k}{\ell_k}$ is the dual class of $\lgame{\MF}{k}{\ell_k}$; thus we may view $\lgame{\MF}{k}{\ell_k}$ as a set of mappings from $\lgame{\ML}{k}{\ell_k}$ to $\{0,1\}$, i.e., for $f_k \in \lgame{\MF}{k}{\ell_k}$ and $L_k \in \lgame{\ML}{k}{\ell_k}$, we have $f_k(L_k) := L_k(f_k)$. 

We consider the following \emph{independent learning setting} in Littlestone games, which directly generalizes the setting of independent learning in normal-form games. Consider a Littlestone game $\ell = (\ell_1, \ldots, \ell_K)$, with action sets $\MF_1, \ldots, \MF_K$:
\begin{itemize}
\item For each time step $1 \leq t \leq T$:
  \begin{enumerate}
  \item Each player $k$ plays a distribution over actions $\bar f_k^t \in \Delc(\MF_k)$. 
  \item Each player $k$ observes %
    its loss function $L_k^t \in \Delc(\lgame{\ML}{k}{\ell_k})$ at time step $t$, namely the mapping $L_k^t(f_k) := \E_{f_{-k} \sim \bar f_{-k}^t} \left[ \ell_k(f_k, f_{-k})\right]$. %
  \item Each player $k$ suffers loss $\ell_k(\bar f_k^t, \bar f_{-k}^t)$; notice that this loss value may also be written as $\bar f_k^t(L_k^t)$, by viewing $\bar f_k^t$ as an element of $\Delc(\mix{\lgame{\MF}{k}{\ell_k}})$. %
  \end{enumerate}
\end{itemize}

\subsection{Independent learning algorithm for fast rates in games}
\begin{lemma}
  \label{lem:sfat-mix}
  Given a class $\MF \subset \{0,1\}^\MX$, %
  it holds that $\sfat_\alpha(\mix{\MF}) \leq O \left(\Ldim(\MF) \cdot \log(\Ldim(\MF)/\alpha)\right)$.
\end{lemma}
\begin{proof}
  Denote $L = \Ldim(\MF), V = \VCdim(\MF)$. 
  Suppose $\bp$ is an $\alpha$-shattered $\Delc(\MX)$-valued tree of depth $d$ for $\mix{\MF}$, witnessed by $\bs$. For some constant $C > 1$, let us form a new $\MX$-valued tree, $\bp'$, by replacing each node $v$ of $\bp$, labeled by $P_v$, with a new tree $\bt_v$ of depth $m := \lceil C \cdot V/\alpha^2\rceil$. For $i \in [m]$, each node on the $i$th level of $\bt_v$ is labeled by $x_v^i$, where the points $x_v^1, \ldots, x_v^m \in \MX$ satisfy the following: for all $f \in \MF$, %
  \begin{align}
    \label{eq:tree-unif-conv}
\left| \E_{x \sim P_v}[f(x)] - \frac 1m \sum_{i=1}^m f(x_v^i) \right| \leq \frac{\alpha}{4}.
\end{align}
By classic uniform convergence bounds \cite{talagrand_sharper_1994,van_der_vaart_weak_1996} 
such points $x_v^1, \ldots, x_v^m \in \MX$ exist as long as $C$ is sufficiently large (this holds even in the absence of additional measurability assumptions on $\MX$ since $P_v$ is finite-support). Let $s_v \in [0,1]$ be the label of the node of $\bs$ corresponding to node $v$ of $\bp$. 
For each of the $2^m$ leaves of the tree $\bt_v$, indexed by $(\delta_1, \ldots, \delta_m) \in \{-1,1\}^m$, we will assign to each such leaf the subbtree rooted by either the left ($-1$) or right ($+1$) child of $v$, as follows: if $\frac{1}{m} \cdot \sum_{i=1}^m \left(\frac{1+ \delta_i}{2} \right) \geq s_v$, then use the subtree rooted by the right child of $v$, and otherwise use the subtree rooted by the left child of $v$. %
Formally, we have the following: for any sequence $\delta_1, \ldots, \delta_{dm} \in \{-1,1\}^{dm}$, and any $i \in [dm]$, writing $i = mt + j$ for $1 \leq j \leq m$, then
$ \bp_i'(\delta_{1:i-1}) = x_v^j,$
where $v$ is the node of $\bp$ corresponding to the sequence $(\ep_1, \ldots, \ep_{t})$, where
\begin{align}
  \label{eq:eps-sign}
\ep_\ell = {\rm sign}\left( \frac 1m \cdot \sum_{i=1}^m \left( \frac{1+\delta_{(\ell-1)m+i}}{2} \right) - s_{(\ep_1, \ldots, \ep_{\ell-1})} \right)\qquad \forall \ell \in [t].
\end{align}

The depth of the new tree $\bp'$ we have constructed is $dm$. By (\ref{eq:tree-unif-conv}) and (\ref{eq:eps-sign}), $\bp'$ satisfies the following property: for any $t \in [d]$, %
and $\ep \in \{-1,1\}^d$, consider any function $f \in \MF$ so that for $i < t$, $\ep_i \cdot (f(\bp_i(\ep_{1:i-1})) - \bs_i(\ep_{1:i-1})) > \alpha/2$.
Let $v$ be the node of $\bp$ corresponding to the sequence $\ep_1, \ldots, \ep_{t-1}$, $P_v = \bp_t(\ep_{1:t-1})$ (as above), and consider the sequence $x_v^1, \ldots, x_v^m$. Now define the sequence $\delta \in \{-1,1\}^{dm}$ inductively via $\delta_i = 2 \cdot f(\bp_i'(\delta_{1:i-1})) - 1$ for $i \geq 1$. Then the sequence $\bp_{tm+1}'(\delta_{1:tm}), \ldots, \bp_{tm+m}'(\delta_{1:tm+m-1})$ is exactly the sequence $x_v^1, \ldots, x_v^m$; we will say that $f,f'$ \emph{encounter the sequence $x_v^1, \ldots, x_v^m$ in the tree $\bp'$.}

Now consider $f,f' \in \mix{\MF}$ which lead to different leaves of the tree $\bp$, in the sense that there are $\ep \neq \ep' \in \{-1,1\}^d$ so that, for each $t \in [d]$, $\ep_t \cdot (f(\bp_t(\ep_{1:t-1})) - \bs_t(\ep_{1:t-1})) > \alpha/2$ and $\ep'_t \cdot (f'(\bp_t(\ep'_{1:t-1})) - \bs_t(\ep'_{1:t-1})) > \alpha/2$. Let $t_0 \in [d]$ be as small as possible so that $\ep_{t_0} \neq \ep_{t_0}'$, and let $v$ be the node of $\bp$ corresponding to the sequence $\ep_1, \ldots, \ep_{t_0-1}$; let $P_v = \bp_{t_0}(\ep_{1:t_0-1}) \in \Delta(\MX)$ be the label of $v$ and $s_v = \bs_{t_0}(\ep_{1:t_0-1}) \in [0,1]$ be the label of the corresponding node of $\bs$, and 
$\bt_v$ be the tree constructed in place of $v$ (as above). By the choice of $v$, it holds that %
$f(P_v) > s_v + \alpha/2$ and $f'(P_v) < s_v-\alpha/2$. Thus, letting $x_v^1, \ldots, x_v^m$ be the sequence constructed as above for the node $v$, we have $\sum_{i=1}^m f(x_v^i) >s >  \sum_{i=1}^m f'(x_v^i)$. %

Thus $f,f'$ lead to different leaves of the tree $\bt_v$, and hence (since $f,f'$ both encounter the sequence $x_v^1, \ldots, x_v^m$ in the tree $\bp'$) also to different leaves of the tree $\bp'$. Thus the sequential 0-covering number of the tree $\bp'$ (see \cite[Definition 13.2]{rakhlin_statistical_2014}) is at least $2^d$.\footnote{In more detail, what we have directly shown is that the \emph{thicket shatter function} of the tree $\bp'$ is at least $2^d$; then \cite[Lemma 2.7]{ghazi_near-tight_2021} implies that the sequential 0-covering number of the tree $\bp'$ is at least $2^d$.} On the other hand, by the Sauer-Shelah lemma for trees \cite[Theorem 13.7]{rakhlin_statistical_2014}, the sequential 0-covering number of the (depth-$dm$) tree $\bp'$ is at most $(e dm)^L$.

Summarizing, we have that $2^d \leq (edm)^L$, i.e., 
$d \leq L \log(edm)$, meaning that $d \leq O(L \log(Lm)) \leq O(L \log(LV/\alpha^2)) \leq O(L \log (L/\alpha))$. 
\end{proof}

\begin{algorithm}[!htp]
  \caption{\bf \OSEgame}\label{alg:ose-game}
  \KwIn{Littlestone game $\ell = (\ell_1, \ldots, \ell_K)$ with action sets $\MF_1, \ldots, \MF_K$, time horizon $T \in \BN$.

{\bf Input to each player:}  Each player $k \in [K]$ only knows its action set $\MF_k$ and its loss class $\lgame{\ML}{k}{\ell_k}$, as well as the horizon $T$. } 
\begin{enumerate}[leftmargin=14pt,rightmargin=20pt,itemsep=1pt,topsep=1.5pt]
\item Each player $k \in [K]$ initializes some online proper learning algorithm $\CA_k$ (e.g., \OSE, Algorithm \ref{alg:ose}) with function class $\MF = \mix{\lgame{\MF}{k}{\ell_k}}$ and feature space $\MX = \Delc(\lgame{\ML}{k}{\ell_k})$. 
\item For $1 \leq t \leq T$:
  \begin{enumerate}
  \item Each player $k \in [K]$ plays a distribution $\bar f_k^t \in \Delc(\MF_k)$ according to their respective algorithm $\CA_k$. 
  \item Each player $k \in [K]$ observes the loss function $L_k^t \in \Delc(\lgame{\ML}{k}{\ell_k}) = \MX$ (defined as $L_k^t(f_k) = \E_{f_{-k} \sim \bar f_{-k}^t}[\ell_k(f_k, f_{-k})]$), and feeds the example $(L_k^t, 0)$ to its algorithm $\CA_k$. 
  \item Each player $k$ suffers loss $\ell_k(\bar f^t) = \E_{f \sim \bar f_k^t}[f (L_k^t)]$. 
  \end{enumerate}
\end{enumerate}
\end{algorithm}

\begin{theorem}
  \label{thm:games-formal}
  Fix a Littlestone game with $K$ players and a time horizon $T$. %
  If the players play according to Algorithm \ref{alg:ose-game} with each player using the algorithm \OSE (Algorithm \ref{alg:ose}) with step sizes $\etaPSR, \etaOH$ as in (\ref{eq:step-sizes}) below and scale $\alpha = 1/T$, then each player $k \in [K]$ suffers regret $\tilde O(\Ldim(\lgame{\MF}{k}{\ell_k})^{3/4} \cdot \sqrt{K} \cdot T^{1/4})$, where the $\tilde O(\cdot)$ hides logarithmic factors in $T$ and $\Ldim(\lgame{\MF}{k}{\ell_k})$.
\end{theorem}
\begin{proof}
  Set
  \begin{align}
    \label{eq:step-sizes}
    \eta = \etaPSR = \etaOH = \frac{\Ldim(\lgame{\MF}{k}{\ell_k}) \cdot \log(\Ldim(\lgame{\MF}{k}{\ell_k})\cdot T)}{K^{1/2} \cdot T^{1/4}}
  \end{align}
  and $\alpha = 1/T$. Also write $\LL_k = \Ldim(\lgame{\MF}{k}{\ell_k})$. In Algorithm \ref{alg:ose-game}, each player $k$ applies \OSE with function class $\MF = \mix{\lgame{\MF}{k}{\ell_k}}$ with feature space $\MX = \Delc(\lgame{\ML}{k}{\ell_k})$; by Lemma \ref{lem:sfat-mix}, it holds that $\sfat_\alpha(\mix{\lgame{\MF}{k}{\ell_k}}) \leq O(\LL_k \cdot \log(\LL_k T))$. 
  
  By Theorem \ref{thm:path-stable}, the hypotheses $\bar f_k^t \in \Delc(\lgame{\MF}{k}{\ell_k})$ output by each player $k$ satisfy $\tvnorm{ \bar f_k^t - \bar f_k^{t+1} } \leq 8 \eta$. %

  Now let us consider any player $k \in [K]$; by symmetry we may assume $k = 1$; then for any $t < T$ and any $f_1 \in \MF_1$, abbreviating $f = (f_1, \ldots, f_K)$, we have
  \begin{align}
    & |L_1^{t+1}(f_1) - L_1^t(f_1)|\nonumber\\
    \leq & \left| \E_{f_2 \sim \bar f_2^{t+1}, \ldots, f_K \sim \bar f_K^{t+1}} \left[ \ell_1(f) \right] - \E_{f_2 \sim \bar f_2^t, \ldots, f_K \sim \bar f_K^t} \left[ \ell_1(f) \right]\right|\nonumber\\
    \leq & \sum_{j=2}^K \left| \E_{f_2 \sim \bar f_2^{t+1}, \ldots, f_j \sim \bar f_j^{t+1}, f_{j+1} \sim \bar f_{j+1}^t, \ldots, f_K \sim \bar f_K^t} \left[ \ell_1(f) \right] - \E_{f_2 \sim \bar f_2^{t+1}, \ldots, f_{j-1} \sim \bar f_{j-1}^{t+1}, f_{j} \sim \bar f_j^t, \ldots, f_K \sim \bar f_K^t} \left[ \ell_1(f) \right] \right|\nonumber\\
    \leq & \sum_{j=2}^K \left|  \E_{f_2 \sim \bar f_2^{t+1}, \ldots, f_{j-1} \sim \bar f_j^{t+1}, f_{j+1} \sim \bar f_{j+1}^{t}, \ldots, f_K \sim \bar f_K^t} \left[(\E_{f_j \sim \bar f_j^{t+1}} - \E_{f_j \sim \bar f_j^t})[\ell_1(f)] \right] \right|\nonumber\\
    \leq & \sum_{j=2}^K \tvnorm{ \bar f_j^{t+1} - \bar f_j^t } \nonumber\\
    \leq & 8 \eta K \nonumber.
  \end{align}
  It follows that $\infnorms{L_1^{t+1} - L_1^t}{\lgame{\MF}{1}{\ell_1}} \leq 8 \eta K$ for all $t < T$. Since the choice of player $k = 1$ here is arbitrary, we have in a similar manner that for all $k \in [K]$, $\infnorms{L_k^{t+1} - L_k^t}{\lgame{\MF}{k}{\ell_k}} \leq 8 \eta K$. Since, by assumption, each player runs \OSE with step size $\etaOH = \etaPSR = \eta$, we may apply Theorem \ref{thm:path-stable} with $\kappa = 8 \eta K$ and $\Gamma = 8K$ to obtain that each player's regret is bounded above by
  \begin{align}
O \left( \frac{K \cdot \sfat_\alpha(\lgame{\MF}{k}{\ell_k}) \cdot \log^3 T}{\eta K} + \eta^3 K^2 T \right) \leq O \left( \sqrt{K} \cdot T^{1/4} \cdot \LL_k^{3/4} \cdot \log^3(\LL_kT) \right)\nonumber.
  \end{align}
\end{proof}

\section{On real-valued games satisfying the minimax theorem}
\label{sec:minimax}
In this section we show that all online learnable (real-valued) classes satisfy the minimax theorem, in the absense of any topological assumptions on the class $\MF$ or the space $\MX$, thus generalizing a corresponding result from \cite{hanneke_online_2021} which treated the binary setting.

\subsection{Additional preliminaries}
We first introduce some additional preliminaries. 
We begin by describing a way to discretize a hypothesis class $\MF \subset [0,1]^\MX$ at some scale $\eta > 0$. Roughly speaking, this is done by subdividing the interval $[0,1]$ into $\lceil 1/\eta \rceil$ intervals each of length $1/\lceil 1/\eta \rceil \leq \eta$, and rounding the output of each hypothesis to its interval. Formally, we make the following definitions: For a real number $y \in [0,1]$, define the discretiztion of $y$ at scale $\eta$, denoted $\disc{y}{\eta} \in \{ 1/\lceil 1/\eta \rceil, 2/\lceil 1/\eta \rceil, \ldots, 1 \}$, as follows: $\disc{y}{\eta} := \frac{1}{\lceil 1/\eta \rceil} \cdot \left(1 + \lfloor y \cdot \lceil 1/\eta \rceil \rfloor\right)$ for $0 \leq y < 1$ and $\disc{y}{\eta} =1$ for $y = 1$. %
It is straightforward from this definition that for all $y \in [0,1]$,
\begin{align}
| y - \disc{y}{\eta} | \leq 1/\lceil 1/\eta \rceil \leq \eta\nonumber.
\end{align}

\nc{\tmarg}{10}
\nc{\ttmarg}{5}
\begin{defn}[Thresholds with margin; similar to \cite{jung_equivalence_2020}, Definition 7]
  \label{def:thresholds-margin}
  Consider a hypothesis class $\MF \subset [0,1]^\MX$, $\alpha > 2\beta > 0$, and $d \in \BN$. $\MF$ is said to contain \emph{$d$ thresholds with margin $\alpha$ and tightness $\beta$} (respectively, \emph{infinitely many thresholds with margin $\alpha$} and tightness $\beta$) if  there are $x_1, \ldots, x_d \in \MX$ and $f_1, \ldots, f_d \in \MF$ (respectively, $x_1, x_2, \ldots \in \MX$ and $f_1, f_2, \ldots \in \MF$) as well as $u,u' \in [0,1]$ so that:
  \begin{itemize}
  \item $|u-u'| \geq \alpha$;
  \item $|f_i(x_j) - u| \leq \beta$ for $i \leq j$ and $|f_i(x_j) - u'| \leq \beta$ for $i > j$. 
  \end{itemize}
  We further say that $\MF$ contains $d$ (or infinitely) many \emph{ordered} thresholds with margin $\alpha$ and tightness $\beta$ if the above conditions hold and furthermore $u' > u$. 
\end{defn}

The following lemma, which gives a lower bound on the sequential fat-shattering dimension for a class with many thresholds, is standard, but we include a proof for completeness.
\begin{lemma}
  \label{lem:sfat-thresholds}
Suppose that $\MF$ contains $d$ thresholds with margin $\alpha$ and tightness $\beta$. Then $\sfat_{\alpha-2\beta}(\MF) \geq \lfloor \log d \rfloor$. 
\end{lemma}
\begin{proof}
  The proof closely follows the analogous result for Littlestone dimension (see \cite{shelah_classification_1978,hodges_shorter_1997,alon_private_2019}). Set $m = \lfloor \log d \rfloor$, and suppose that $x_1, \ldots, x_{2^m}$ and $f_1, \ldots, f_{2^m}$ are a collection of $2^m$ thresholds with margin $\alpha$ and tightness $\beta$, together with the values $u,u'$ as in Definition \ref{def:thresholds-margin}; we may assume $u' > u$ without loss of generality (otherwise we can reverse the order of the thresholds). We construct a tree $\bx$ of depth $m$ that is shattered (together with the witness tree $\bs$) as follows: the labels of the tree $\bx$ correspond to the binary search process on $[2^m]$, so that $\bx_t(\ep_1, \ldots, \ep_{t-1}) = x_{2^{m-1} + \ep_1 \cdot 2^{m-2} + \cdots + \ep_{t-1} \cdot 2^{m-t}}$. All nodes of the tree $\bs$ are labeled by $(u+u')/2$. It is straightforward to see that the function $f_i$ leads to the leaf which is $i$ spots from the left (viewing $-1$ as the left child and $1$ as the right child for each node).

  The lower bound on the sequential fat-shattering dimension then follows from the fact that for $i > j$, we have $f_i(x_j) \geq (u+u')/2 + \alpha/2 - \beta = (u+u')/2 + (\alpha-2\beta)/2$ and for $i \leq j$, we have $f_i(x_j) \leq (u+u')/2 - (\alpha-2\beta)/2$. 
\end{proof}

Lemma \ref{lem:thresholds-sfat} below provides a sort of converse to Lemma \ref{lem:sfat-thresholds}, giving a lower bound on the number of thresholds in a real-valued class of large sequential fat-shattering dimension. 
\begin{lemma}
  \label{lem:thresholds-sfat}
For some constant $c > 0$ the following holds.  Suppose that $\alpha \geq 4\eta > 0, d \in \BN$ are so that $\sfat_\alpha(\MF) \geq d$. Then $\MF$ contains $c \cdot \frac{\eta \log(\eta \log d)}{\log 1/\eta}$ thresholds with margin $\alpha/4$ and tightness $\eta$.
\end{lemma}
A similar result to Lemma \ref{lem:thresholds-sfat} was claimed in \cite[Theorem 8]{jung_equivalence_2020}, though with a stronger quantitative bound (namely, the lower bound on the number of thresholds was $\Omega_\eta(\log d)$, not $\Omega_\eta(\log \log d)$, as we show). Unfortunately, there appears to be a gap in the proof \cite[Theorem 8]{jung_equivalence_2020}: in particular, the proof of Proposition 5 in \cite{jung_equivalence_2020} (which is used to prove Theorem 8) begins with the following claim: ``Since $\sfat_\eta(\MF) \geq d$, in the online learning setting an adversary can force any deterministic learner to suffer $\eta/2$ loss for $d$ rounds.'' This sentence is incorrect, even if the adversary only reveals the discretized labels to the learner: in particular, fix $\eta > 0$, $X := \log(\lfloor 1/\eta \rfloor / 2)$, and set $\MX = \{1, 2, \ldots,X\}$. Consider the following class $\MF$ which  consists of $2^X$ hypotheses: for each $(\ep_1, \ldots, \ep_X) \in \{-1,1\}^X$, let $n(\ep) \in \{1, 2, \ldots, \lfloor 1/\eta \rfloor / 2\}$ be the integer corresponding to $\ep$ in base 2. Then there is a hypothesis $f_\ep \in \MF$ so that $f_\ep(i) = \frac 12 + \ep_i \cdot \eta \cdot n(\ep)$ for each $i \in \MX$. It is evident that $\sfat_\eta(\MF) \geq \Omega(\log 1/\eta )$, yet no matter which point $x_1 \in \MX$ which the adversary first  reveals to the learner, the value $\disc{f^\st(x_1)}{\eta}$ reveals the identity of $f^\st$, meaning that the learner will always make at most 1 mistake. This gap is filled in our proof of Lemma \ref{lem:thresholds-sfat}, at the cost of a weaker quantitative bound; the question of whether $\MF$ contains $\Omega_\eta(\log d)$ thresholds (in the context of Lemma \ref{lem:thresholds-sfat}) is left open for future work.

\begin{proof}[Proof of Lemma \ref{lem:thresholds-sfat}]
  We will first construct a weaker notion of a collection of thresholds: in particular, we will first prove the following claim:
  \begin{claim}
    \label{clm:weak-thresholds}
    For some constant $C_1 > 0$, the following holds. 
    Fix $\alpha \geq 4\eta > 0$ and $m \in \BN$. %
    If $\sfat_\alpha(\MF) \geq \left( C_1/\eta\right)^m$, then there is a collection of hypotheses $f_1, \ldots, f_m \in \MF$ and points $x_1, \ldots, x_m \in \MX$ so that, for each $i \in [m]$, the following holds: there are some $\nu_i, \mu_i \in \MD_\eta$ satisfying $|\nu_i - \mu_i| \geq \alpha/4$, so that one of the below options holds:
  \begin{itemize}
  \item $\nu_i > \mu_i$; moreover, for $j \geq i$, we have $\disc{f_i(x_j)}{\eta} = \nu_i$ and for $j > i$ we have $\disc{f_j(x_i)}{\eta} \leq \mu_i$; or
  \item $\nu_i < \mu_i$; moreover, for $j \geq i$, we have $\disc{f_i(x_j)}{\eta} = \nu_i$, and for $j > i$, we have $\disc{f_j(x_i)}{\eta} \geq \mu_i$. %
  \end{itemize}
\end{claim}
    \begin{proof}[Proof of Claim \ref{clm:weak-thresholds}]
    We use induction on $m$. In the case $m = 1$, then since $\sfat_\alpha(\MF) \geq 1$, $\MF, \MX$ are nonempty so the proof is completed by choosing any $x \in \MX$ and $f \in \MF$.%

    Now suppose $m > 1$ and the claim statement holds for all values $m' < m$.  Set $d = \left(C_1/\eta \right)^m$, and write $\MD_\eta := \{ 1/\lceil 1/\eta \rceil, 2/\lceil 1/\eta \rceil, \ldots, 1\}$ to denote the set of discretized points with discretization of $\eta$, and $k = \lceil 1/\eta \rceil = |\MD_\eta|$. Before continuing, we need to introduce the notion of \emph{subtree}: given a tree $\bx$, a subtree of $\bx$ of depth $t$ is defined inductively as follows. Any node of $\bx$ is a subtree of depth 0. A subtree of depth $t$ is obtained by taking any internal node $v$ of $\bx$ together with a subtree of the trees rooted at the left and right children of $v$. Note that if the tree $\bx$ is $\alpha$-shattered by a hypothesis class $\MF$, then so is any subtree of $\bx$.  

    Let $\bx$ be an $\MX$-valued tree of depth $d$ shattered by $\MF$, witnessed by a $[0,1]$-valued tree $\bs$.
    Let $f$ be an arbitrary hypothesis in $\MF$, and define a $k$-coloring of the nodes of $\bx$ as follows: color a node corresponding to the sequence $\ep_{1:t-1}$ by the element $\disc{f(\bx_t(\ep_{1:t-1}))}{\eta} \in \MD_\eta$. By \cite[Lemma 16]{jung_equivalence_2020}, there is a subtree $\bx'$ of $\bx$ of depth $d' := \lceil (d+1)/k \rceil \geq d/k$ so that all nodes are colored by some color $\nu^\st \in \MD_\eta$. Denote the corresponding subtree of $\bs$ by $\bs'$.
    Set $\MX' := \{x \in \MX : \disc{f(x)}{\eta} = \nu^\st \}$, so that $\bx'$ is $\MX'$-valued and is shattered by $\MF$, as witnessed by $\bs'$. 

    Define the following subclasses of $\MF$, viewed as classes of hypotheses on the restricted set $\MX'$:
    \begin{align}
      \MF_+ := \left\{ f \in \MF \ : \ f(\bx'_1) \geq \bs'_1 + \alpha/2 \right\} \subset [0,1]^{\MX'} \nonumber\\
      \MF_- := \left\{ f \in \MF \ : \ f(\bx'_1) \leq \bs'_1 - \alpha/2 \right\} \subset [0,1]^{\MX'}\nonumber.
    \end{align}
    It is immediate that $\sfat_{\alpha}(\MF_+) \geq d'-1$ and $\sfat_\alpha(\MF_-) \geq d'-1$ (in particular, $\MX'$-valued trees shattering $\MF_+, \MF_-$ are obtained by taking the subtrees of $\bx'$ rooted at the right and left children, respectively, of its root).
        Set $\nu_+ = \disc{\bs_1' + \alpha/2}{\eta} \geq \bs_1'+\alpha/2-\eta$ and $\nu_- = \disc{\bs_1' - \alpha/2}{\eta} \leq \bs_1' - \alpha/2 + \eta$. 

    We must have either $\nu^\st \geq \bs_1'$ or $\nu^\st \leq \bs_1'$. 
    We consider each of the cases in turn:
    \begin{itemize}
    \item %
      If $\nu^\st \geq \bs_1'$, 
      then we apply the inductive hypothesis on the class $\MF_-$ and the data (feature) space $\MX'$. We have that $\sfat_\alpha(\MF_-) \geq d'-1 \geq \frac{d}{2k} \geq \left( C_1/\eta\right)^{m-1}$ (as long as $C_1$ is chosen sufficiently large), meaning that, by the inductive hypothesis with the value $m-1$ (and the same values of $\alpha,\eta$), we can find $f_2, \ldots, f_m \in \MF_-$, $x_2, \ldots, x_m \in \MX'$ so that the constraints of the claim statement are staisfied. Now we add $f_1 = f,\ x_1 = \bx_1'$ to this collection. Note that, for $i \geq 1$, we have $\disc{f_1(x_i)}{\eta} = \disc{f(x_i)}{\eta} = \nu^\st$ since all $x_i$ (including $\bx_1'$) lie in $\MX'$. Further, for $i > 1$, we have $\disc{f_i(x_1)}{\eta} = \disc{f_i(\bx_1')}{\eta} \leq \nu_-$ by definition of $\MF_-$. Since $|\nu^\st - \nu_-| \geq \alpha/2-\eta \geq \alpha/4$, we have verified the inductive step in this case; in particular, we may set $\nu_1 = \nu^\st$ and $\mu_1 = \nu_-$ ($f_1, x_1$ correspond to the first case in the claim statement).
    \item If $\nu^\st \leq \bs_1'$, then we apply exactly the same argument except with $\MF_+$ replacing $\MF_-$. Again setting $f_1 = f, x_1 = \bx_1'$, we have, for $i \geq 1$, $\disc{f_1(x_i)}{\eta} = \disc{f(x_i)}{\eta} = \nu^\st$, while for $i > 1$, we have $\disc{f_i(x_i)}{\eta} = \disc{f_i(\bx_1')}{\eta} \geq \nu_+$. Since $|\nu^\st - \nu_+| \geq \alpha/2-\eta \geq \alpha/4$, we have verified the inductive step; in particular, we may set $\nu_1 = \nu^\st$ and $\mu_1 = \nu_+$ ($f_1, x_1$ now correspond to the second case in the claim statement). 
    \end{itemize}
      \end{proof}
  Given Claim \ref{clm:weak-thresholds}, we may now complete the proof of Lemma \ref{lem:thresholds-sfat}, as follows. Given that $\sfat_\alpha(\MF) \geq d$, set $m = \left \lfloor \frac{\log d}{\log C_1/\eta} \right\rfloor$, wheree $C_1$ is the constant of Claim \ref{clm:weak-thresholds}. Then we may consider a collection $f_1, \ldots, f_m \in \MF$ and $x_1, \ldots, x_m \in \MX$ satisfying the guarantee of Claim \ref{clm:weak-thresholds}. Since there are $\lceil 1/\eta\rceil$ possibilities for the value $\nu_i$, for $\ell := m / \lceil 1/\eta \rceil$, we can extract a subset  $g_1 := f_{i_1}, \ldots, g_{\ell} := f_{i_{\ell}}, w_1 := x_{i_1}, \ldots, w_{\ell} := x_{i_{\ell}}$ so that, for some fixed $\nu \in \MD_\eta$, $\nu_{i_j} = \nu$ for all $j \in [\ell]$; in particular, for $1 \leq i \leq j \leq m$, it holds that $\disc{g_i(w_j)}{\eta} = \nu$. 

  Now we color each tuple $(i,j)$ with $1 \leq i < j \leq \ell$ with the value $\disc{g_j(w_i)}{\eta} \in \MD_\eta$; note that for all such $i,j$, by our choice of $\nu$, we must have that $|\disc{g_j(w_i)}{\eta}-\nu| \geq \alpha/4$. By Ramsey's theorem,\footnote{In particular, we use the following estimate on the multi-color Ramsey numbers \cite{greenwood_combinatorial_1955}: for $N \geq c^{rc}$, if the edges of the complete graph on $N$ vertices are colored with $c$ colors, there is a monochromatic clique of size $r$.} there is some $\mu \in \MD_\eta$ and a sub-collection $h_1 := g_{i_1}, \ldots, h_p := g_{i_p}, v_1 = w_{i_1}, \ldots, v_p := w_{i_p}$ for some $p \geq \frac{\log \ell}{\lceil 1/\eta \rceil \log \lceil 1/\eta \rceil}$, so that for all $1 \leq i < j \leq p$, $\disc{h_j(v_i)}{\eta} = \mu$. Further, it must be the case that $|\nu - \mu| \geq \alpha/4$.

  Summarizing, we have found a collection of thresholds (namely, $h_1, \ldots, h_p, v_1, \ldots, v_p$) with margin $\alpha/4$, tightness $\eta$, and of size
  \begin{align}
p \geq \Omega \left( \frac{\eta \log \ell}{\log 1/\eta} \right) \geq \Omega \left( \frac{\eta \log (\eta m)}{\log 1/\eta} \right) \geq \Omega \left( \frac{\eta \log(\eta \log(d))}{\log 1/\eta} \right)\nonumber.
  \end{align}
\end{proof}

Finally, we may combine Lemmas \ref{lem:sfat-thresholds} and \ref{lem:thresholds-sfat} to show that the sequential fat-shattering dimension of a class is finite if and only if the sequential fat-shattering dimension of the dual class is finite.
\begin{lemma}
  \label{lem:sfat-dual}
Suppose $\MF \subset [0,1]^\MX$. Then for any $\alpha > 0$, the dual class $\MF^\st$ satisfies $\sfat_{\alpha/8}(\MF^\st) \geq \Omega(\log (\alpha \cdot \log \log (\sfat_\alpha(\MF))))$. 
\end{lemma}
For binary valued classes, it is known (see \cite{bhaskar_thicket_2021}) that $\Ldim(\MF^\st) \geq \Omega(\log \log(\Ldim(\MF)))$. The additional logarithm in Lemma \ref{lem:sfat-dual} is due to the double logarithm in the lower bound of Lemma \ref{lem:thresholds-sfat}; we leave the question of improving the quantitative bound in Lemma \ref{lem:sfat-dual} to future work.
\begin{proof}[Proof of Lemma \ref{lem:sfat-dual}]
  Set $\eta = \alpha/16$. 
  Write $d = \sfat_\alpha(\MF)$. By Lemma \ref{lem:thresholds-sfat}, for some constant $c > 0$, $\MF$ contains $m := c \cdot \frac{\alpha \log (\alpha \log d)}{\log 1/\alpha}$ thresholds with margin $\alpha/4$ and tightness $\eta$, which we denote $f_1, \ldots, f_m \in \MF,\ x_1, \ldots, x_m \in \MX$. Thus, the functions in $\MF^\st$ corresponding to $x_m, \ldots, x_1$ furnish $m$ thresholds in the dual class on the points $f_m,\ldots, f_m$, with margin $\alpha/4$ and tightness $\eta$. Then by Lemma \ref{lem:sfat-thresholds}, we have that
  \begin{align}
    \sfat_{\alpha/8}(\MF^\st) \geq \lfloor \log m \rfloor \geq \left \lfloor \log \left( c \cdot \frac{\alpha \log(\alpha \log d)}{\log 1/\alpha} \right) \right\rfloor \geq \Omega \left( \log(\alpha \cdot  \log \log d) \right)\nonumber.
  \end{align}
\end{proof}

\subsection{A minimax theorem for online learnable games}
In this section we will consider infinite two-player zero-sum games: in particular, fix sets $\MX, \MF$ and a loss function $\ell : \MX \times \MF \ra [0,1]$. The loss $\ell$ defines a function class in bijection with $\MF$, namely the class $\MF^\ell := \{ x \mapsto \ell(x,f) \ : \ f \in \MF \}$, as well as its dual class $\MX^\ell$, namely the class $\{ f \mapsto \ell(x,f) \ : \ x \in \MX \}$. We say that the game $(\MX, \MF, \ell)$ is a \emph{GC game} if $\fat_\alpha(\MF^\ell) < \infty$ for all $\alpha > 0$ (here ``GC''  stands for ``Glivenko-Cantelli'', refelcting the fact that the hypothesis class $\MF^\ell$ is a Glivenko-Cantelli class). It is folklore (see \cite[Corollary 3.8]{kleer_primal_2021}) that  for any real-valued hypothesis class $\MG \subset [0,1]^\MX$, its dual class $\MG^\st$ satisfies $\fat_{\alpha/2}(\MG^\st) \geq \Omega(\log(\alpha \cdot \fat_\alpha(\MG)))$. Thus, for a GC game $(\MX, \MF, \ell)$, we have $\fat_\alpha(\MX^\ell) < \infty$ for all $\alpha$. 
 For $\alpha > 0$, we say that $(\MX, \MF, \ell)$ is an \emph{$\alpha$-GC game} if $\max\left\{\fat_\alpha(\MF^\ell), \fat_\alpha(\MX^\ell) \right\} < \infty$.

We further define sequential analogues of the above notions: $(\MX, \MF, \ell)$ is defined to be an \emph{SGC game} (``Sequential Glivenko-Cantelli'') if $\sfat_\alpha(\MF^\ell) < \infty$ for all $\alpha > 0$; by Lemma \ref{lem:sfat-dual} this implies that $\sfat_\alpha(\MX^\ell) < \infty$ for all $\alpha > 0$. Further, the game $(\MX, \MF, \ell)$ is said to be an \emph{$\alpha$-SGC game} if $\max \left\{\sfat_\alpha(\MF^\ell), \sfat_\alpha(\MX^\ell) \right\} < \infty$.
\begin{lemma}
  \label{lem:use-sugc}
There is a constant $C_0 > 2$ so that the following holds. Fix any $\alpha > 0$, and suppose that $(\MX, \MF, \ell_0)$ is a $[0,1]$-valued $(\alpha/C_0)$-GC game that does not contain infinitely many ordered thresholds with margin $\alpha$ and tightness $\alpha/C_0$. Then
  \begin{align}
\inf_{P_X \in \Delta(\MX)} \sup_{P_F \in \Delta(\MF)} \E_{(x,f) \sim P_X \times P_F} [\ell_0(x,f)] \leq \sup_{P_F \in \Delta(\MF)} \inf_{P_X \in \Delta(\MX)} \E_{(x,f) \sim P_X \times P_F} [\ell_0(x,f)] + 4\alpha\nonumber.
  \end{align}
  Furthermore, the same statement holds if $P_X, P_F$ are restricted to $\Delc(\MX), \Delc(\MF)$, respectively. 
\end{lemma}
\begin{proof}
  The proof closely follows the technique of \cite[Proposition 9]{hanneke_online_2021}.
  Fix a $[0,1]$-valued GC game $(\MX, \MF, \ell_0)$, set $\eta := \alpha/C_0$, and define the discretization $(\MX, \MF, \ell)$ as follows: for $(x,f) \in \MX \times \MF$,
  \begin{align}
\ell(x,f) := \disc{\ell_0(x,f)}{\eta}\nonumber.
  \end{align}
  Since $\infnorms{\ell - \ell_0}{\MX \times \MF} \leq \eta$, it follows that
  \begin{align}
    \left| \inf_{P_X \in \Delta(\MX)} \sup_{P_F \in \Delta(\MF)} \E_{(x,f) \sim P_X \times P_F} [\ell_0(x,f)] - \inf_{P_X \in \Delta(\MX)} \sup_{P_F \in \Delta(\MF)} \E_{(x,f) \sim P_X \times P_F} [\ell(x,f)] \right| \leq & \eta \label{eq:infnorm-1}\\
    \left|  \sup_{P_F \in \Delta(\MF)} \inf_{P_X \in \Delta(\MX)} \E_{(x,f) \sim P_X \times P_F} [\ell_0(x,f)] -  \sup_{P_F \in \Delta(\MF)} \inf_{P_X \in \Delta(\MX)} \E_{(x,f) \sim P_X \times P_F} [\ell(x,f)] \right| \leq & \eta\label{eq:infnorm-2}.
  \end{align}
  It is furthermore straightforward to see that $(\MX, \MF, \ell)$ is a $3\eta$-GC game; one may see this by noting, for instance, that $\fat_{3\eta}(\{ x \mapsto \ell(x,f) \ : \ f \in \MF \}) \leq \fat_\eta(\{ x \mapsto \ell_0(x,f) \ : \ f \in \MF\})$, and similarly for the dual class. Thus, by uniform convergence (i.e., Theorem \ref{thm:unif-conv}), there is a universal constant $C$ so that, for all $\eta, \MX, \MF, \ell$, there is an integer $V$\footnote{In particular, we may take $V =  O \left( \frac{\max\{ \fat_\eta (\{ x \mapsto \ell_0(x,f) : f \in \MF \}), \fat_\eta(\{ f \mapsto \ell_0(x,f) : x \in \MX \})\} \cdot \log 1/\eta}{\eta^2}\right)$, which is finite, by assumption.}%
, so that for all finite-support distributions $P_X \in \Delc(\MX), P_F \in \Delc(\MF)$, there are elements $x_1, \ldots, x_V \in \MX, f_1, \ldots, f_V \in \MF$, so that
  \begin{align}
    \sup_{f \in \MF} \left| \E_{x \sim P_X} [\ell(x,f)] - \frac 1V \sum_{i=1}^V \ell(x_i, f) \right| \leq & C \eta \label{eq:minmax-unifconv-x}\\
    \sup_{x \in \MX} \left| \E_{f \sim P_F}[\ell(x,f)] - \frac 1V \sum_{i=1}^V \ell(x, f_i) \right| \leq & C \eta\label{eq:minmax-unifconv-f}.
  \end{align}

Now set
  \begin{align}
    \theta =& \sup_{P_F \in \Delta(\MF)} \inf_{P_X \in \Delta(\MX)} \E_{(x,f) \sim P_X \times P_F} [\ell(x,f)] \nonumber\\
    \omega =& \inf_{P_X \in \Delta(\MX)} \sup_{P_F \in \Delta(\MF)} \E_{(x,f) \sim P_X \times P_F} [\ell(x,f)]\nonumber.
  \end{align}
  Suppose for the purpose of contradiction that $\omega > \theta + 4\alpha - 2\eta$ (if this is not the case, then by (\ref{eq:infnorm-1}) and (\ref{eq:infnorm-2}) the proof of the lemma is complete). We next construct two sequences of finite-support distributions $P^t_X, P^t_F$, $t \in \BN$, where each of $P^t_X, P_F^t$ is a uniform distribution over exactly $V$ elements of $\MX, \MF$, respectively (possibly with some elements being duplicates), and so that the following inequalities hold: %
  \begin{align}
    \sup_{P_F \in \Delta(\bigcup_{i < t} \supp(P_F^i))} \E_{(x,f) \sim P_X^t \times P_F} [\ell(x,f)] \leq & \theta + \frac{4\alpha - 2\eta}{3} \qquad \forall t > 1 \label{eq:pf-minmax}\\
    \inf_{P_X \in \Delta(\bigcup_{i \leq t} \supp(P_X^i))} \E_{(x,f) \sim P_X \times P_F^t} [\ell(x,f)] \geq & \theta + 2 \cdot \frac{4\alpha - 2\eta}{3} \qquad \forall t \geq 1\label{eq:px-minmax}.
  \end{align}
  We construct $P_X^t, P_F^t$ inductively as follows: 
  \begin{itemize}
  \item Suppose we have constructed $P_X^1, \ldots, P_X^{t-1}, P_F^1, \ldots, P_F^{t-1}$ satisfying (\ref{eq:pf-minmax}) and (\ref{eq:px-minmax}) up to step $t-1$. To construct $P_X^t$ so as to satisfy (\ref{eq:pf-minmax}) at step $t$, we argue as follows: if $t = 1$, set $P_X^1 = \delta_x$ for any $x \in \MX$ (i.e., the point mass at $x$), which suffices because (\ref{eq:pf-minmax}) is vacuous for $t = 1$. Otherwise, let $F_{<t} = \bigcup_{i < t} \supp(P_F^i)$, which is a finite set. Since $\ell(x,f)$ takes values in a finite set, there are a finite number of distinct sequences $\{ \ell(x,f) \}_{f \in F_{<t}}$ given by elements $x \in \MX$. Thus, there is a finite subset $\MX' \subset \MX$ so that for each $x \in \MX$, there is some $x' \in \MX'$ so that $\ell(x,f) = \ell(x',f)$ for all $f \in F_{<t}$. Thus, by the von Neumann minimax theorem applied to the finite game $(\MX', F_{<t}, \ell)$\footnote{With a slight abuse of notation, the loss function $\ell$ is restricted to $\MX' \times F_{<t}$.}, there is a distribution $P_X^\st \in \Delta(\MX')$ so that
    \begin{align}
      & \sup_{P_F \in \Delta(F_{<t})} \E_{(x,f) \sim P_X^\st \times P_F}[\ell(x,f)] = \sup_{P_F \in \Delta(F_{<t})} \inf_{P_X \in \Delta(\MX')} \E_{(x,f) \sim P_X \times P_F}[\ell(x,f)] \nonumber\\
      &= \sup_{P_F \in \Delta(F_{<t})} \inf_{P_X \in \Delta(\MX)} \E_{(x,f) \sim P_X \times P_F} [\ell(x,f)] \leq \sup_{P_F \in \Delta(\MF)} \inf_{P_X \in \Delta(\MX)} \E_{(x,f) \sim P_X \times P_F}[\ell(x,f)] = \theta\label{eq:use-theta},
    \end{align}
    where the first equality follows from the choice of $P_X^\st$ as a minimax strategy for the $X$-player in the finite game, the second equality follows from the defining property of $\MX'$, and the inequality follows from the fact that points on $\MF$ are measurable, meaning that every measure $P_F \in \Delta(F_{<t})$ may be realized as the corresponding distribution on $\MF$ restricted to $F_{<t}$.

    By (\ref{eq:minmax-unifconv-x}) and using the fact that $P_X^\st$ is a finite-support measure, there is a sequence $x_1, \ldots, x_V \in \MX$ so that, setting $P_X^t$ to be the empirical measure $P_X^t(S) := \frac 1V \sum_{i=1}^V \One[x_i \in S]$, every $f \in F_{<t} \subset \MF$ satisfies
    \begin{align}
\E_{x \sim P_X^t} [\ell(x,f)] - \E_{x \sim P_X^\st}[\ell(x,f)] \leq C \eta \leq \frac{4\alpha - 2\eta}{3}\nonumber,
    \end{align}
    where the final inequality may be ensured by choosing $C_0$ so that $C_0 \geq 3C+ 2$ (which implies that $C\eta \leq \frac{\alpha-2\eta}{3}$). 
    Thus
    \begin{align}
\sup_{P_F \in \Delta(F_{<t})} \E_{(x,f) \sim P_X^t \times P_F}[\ell(x,f)] \leq \sup_{P_F \in \Delta(F_{<t})} \E_{(x,f) \sim P_X^\st \times P_F} [\ell(x,f)] + \frac{4\alpha - 2\eta}{3}  \leq \theta + \frac{4\alpha - 2\eta}{3}\nonumber,
    \end{align}
    showing that (\ref{eq:pf-minmax}) holds at step $t$. 
  \item Next suppose  we have constructed $P_X^1, \ldots, P_X^t, P_F^1, \ldots, P_F^{t-1}$ satisfying (\ref{eq:pf-minmax}) up to step $t$ and satisfying (\ref{eq:px-minmax}) up to step $t-1$. We then construct $P_F^{t+1}$ so as to satisfy (\ref{eq:px-minmax}) in a very similar manner as to the previous case: setting $X_{\leq t} := \bigcup_{i \leq t} \supp(P_X^i)$, we get that there is a finite set $\MF' \subset \MF$ and a distribution $P_F^\st \in \Delta(\MF')$ so that
    \begin{align}
      & \inf_{P_X \in \Delta(X_{\leq t})} \E_{(x,f) \sim P_X \times P_F^\st}[\ell(x,f)] = \inf_{P_X \in \Delta(X_{\leq t})} \sup_{P_F \in \Delta(\MF')} \E_{(x,f) \sim P_X \times P_F}[\ell(x,f)] \nonumber\\
      &= \inf_{P_X \in \Delta(X_{\leq t})} \sup_{P_F \in \Delta(\MF)} \E_{(x,f) \sim P_X \times P_F}[\ell(x,f)] \geq \inf_{P_X \in \Delta(\MX)} \sup_{P_F \in \Delta(\MF)} \E_{(x,f) \sim P_X \times P_F} [\ell(x,f)] = \omega\label{eq:use-omega}.
    \end{align}
    By (\ref{eq:minmax-unifconv-f}) and the fact that $P_F^\st$ is a finite support measure, there is a sequence $f_1, \ldots, f_V \in \MF$ so that, setting $P_F^{t+1}$ to be the empirical measure $P_F^{t+1}(S) := \frac 1V \sum_{i=1}^V \One[f_i \in S]$, every $x \in X_{\leq t} \subset \MX$ satisfies $\E_{f \sim P_F^{t+1}} [\ell(x,f)] \geq \E_{f \sim P_F^\st} [\ell(x,f)] - \frac{4\alpha-2\eta}{3}$. Thus
    \begin{align}
\inf_{P_X \in \Delta(X_{\leq t})} \E_{(x,f) \sim P_X \times P_F^{t+1}} \geq \omega - \frac{4\alpha - 2\eta}{3} \geq \theta + 2 \cdot \frac{4\alpha - 2\eta}{3},\nonumber
    \end{align}
    thus verifying (\ref{eq:px-minmax}) since we have assume $\omega - \theta \geq 4\alpha - 2\eta$.
  \end{itemize}

  As we have constructed each of $P_X^i, P_F^i$ to be a uniform distribution over $V$ elements of $\MX, \MF$, respectively, we may denote these elements as $x_{i,1}, \ldots, x_{i,V}$ and $f_{i,1}, \ldots, f_{i,V}$, for each $i \in \BN$. For each $i,j \in \BN$ with $i < j$, we define matrices $A^{ij}, B^{ij} \in \{ 1/\lceil 1/\eta \rceil, 2/\lceil 1/\eta \rceil, \ldots, 1 \}^{V \times V}$, as follows: for $k,m \in [V]$, we set
  \begin{align}
A^{ij}_{km} := \ell(x_{i,k}, f_{j,m}), \qquad B^{ij}_{km} := \ell(x_{j,k}, b_{i,m})\nonumber.
  \end{align}
  The number of different possible matrix pairs $(A^{ij}, B^{ij})$ is $\lceil 1/\eta \rceil^{2V^2}$, which is finite. The infinite Ramsey theorem implies that there exists an infinite increasing sequence $i_1, i_2, \ldots, \in \BN$ and a pair of matrices $A^\st, B^\st \in \{ 1/\lceil 1/\eta \rceil, 2/\lceil 1/\eta \rceil, \ldots, 1 \}^{V \times V}$ so that for all $s,t \in \BN$ with $s < t$, we have $(A^{i_si_t}, B^{i_si_t}) = (A^\st, B^\st)$.

  We next claim that there are $k^\st, m^\st \in [V]$ so that $A^\st_{k^\st m^\st} - B^\st_{k^\st m^\st} \geq \frac{4\alpha - 2\eta}{3}$. To see this, note that, for any $s, t \in \BN$ with $s < t$, we have
  \begin{align}
    & \inf_{P_X \in \Delta(\supp(P_X^{i_s}))} \sup_{P_F \in \Delta(\supp(P_F^{i_t}))} \E_{(x,f) \sim P_X \times P_F} [\ell(x,f)] \label{eq:astar-value}\\
    \geq & \inf_{P_X \in \Delta(X_{\leq i_t})} \E_{(x,f) \sim P_X \times P_F^{i_t}} [\ell(x,f)] \geq \theta + 2 \cdot \frac{4\alpha - 2\eta}{3}\nonumber,
  \end{align}
  and
  \begin{align}
   &  \inf_{P_X \in \Delta(\supp(P_X^{i_t}))} \sup_{P_F \in \Delta(\supp(P_F^{i_s}))} \E_{(x,f) \sim P_X \times P_F} [\ell(x,f)]\label{eq:bstar-value}\\
    \leq & \sup_{P_F \in \Delta(F_{<i_t})} \E_{(x,f) \sim P_X^{i_t} \times P_F} [\ell(x,f)] \leq \theta + \frac{4\alpha- 2\eta}{3}\nonumber.
  \end{align}
  Notice that the quantity in (\ref{eq:astar-value}) is the value of the game represented by the matrix $\{ \ell(x_{i_s,k}, f_{i_t,m})\}_{k,m \in [V]}$, and this matrix is $A^{i_si_t} = A^\st$. Similarly, the quantity in (\ref{eq:bstar-value}) is the value of the game represented by the matrix $\{ \ell(x_{i_t,k}, f_{i_s,m})\}_{k,m \in [V]}$, and this matrix is $B^{i_si_t} = B^\st$. Hence the value of the game $A^\st$ is at least $\frac{4\alpha - 2\eta}{3}$ greater than the value of the game $B^\st$. Thus some entry of $A^\st$ must be at least $\frac{4\alpha- 2\eta}{3}$ greater than the corresponding entry of $B^\st$. Indeed, if this were not the case, then we would have that
  \begin{align}
\min_{p \in \Delta([V])} \max_{q \in \Delta([V])} p^\t A^\st q < \frac{4\alpha - 2\eta}{3} + \min_{p \in \Delta([V])} \max_{q \in \Delta([V])} p^\t B^\st q,\nonumber
  \end{align}
  a contradiction to the previous sentence. 
 This shows the existence of the $k^\st, m^\st$ as desired.

  Finally we may construct a collection of infinitely many thresholds with margin $\alpha$ and tightness $\beta$. For $t \geq 1$, define $x_t^\st := x_{i_{2t},k^\st}$ and $f_t^\st = f_{i_{2t-1}, m^\st}$. For $s,t \in \BN$ with $s < t$, we have $i_{2s} < i_{2t-1}$, and so $\ell(x_s^\st, f_t^\st) = \ell(x_{i_{2s},k^\st}, f_{2t-1, m^\st}) = A_{k^\st m^\st}^\st$. For $s,t \in \BN$ with $s \geq t$, we have $i_{2s} > i_{2t-1}$, and so $\ell(x_s^\st, f_t^\st) = \ell(x_{i_{2s},k^\st}, f_{2t-1,m^\st}) = B_{k^\st m^\st}^\st$.

  Recalling that $|\ell(x,f) - \ell_0(x,f)| \leq \eta$ for all $x,f$, it follows that $|\ell_0(x_s^\st, f_t^\st) - A_{k^\st m^\st}^\st| \leq \eta$ for $s \leq t$ and $|\ell_0(x_s^\st, f_t^\st) - B_{k^\st m^\st}^\st| \leq \eta$ for $s > t$. Further, since $2 \eta \leq \alpha$ (as $C_0 > 2$), we have $A^\st_{k^\st m^\st} - B^\st_{k^\st m^\st} \geq \alpha$, as desired (in particular, in the context of Definition \ref{def:thresholds-margin}, we may take $u' = A_{k^\st m^\st}^\st,\ u = B_{k^\st m^\st}^\st$). 

  Finally, to establish the statement of the lemma about finite support measures $P_X, P_F$, we note that exactly the same proof presented above works; the only difference is that in (\ref{eq:use-theta}) and (\ref{eq:use-omega}), we replace $\Delta(\MF)$ with $\Delc(\MF)$ and $\Delta(\MX)$ with $\Delc(\MF)$; it is evident that the claimed inequalities in (\ref{eq:use-theta}) and (\ref{eq:use-omega}) hold even with these substitutions.
\end{proof}

We next show a converse to Lemma \ref{lem:use-sugc}, thus obtaining a necessary and sufficient condition for the minimax theorem to hold in all subgames of a GC game. 
\begin{lemma}[Converse to Lemma \ref{lem:use-sugc}]
  \label{lem:sugc-converse}
  For any $\alpha \in (0,1)$, any $[0,1]$-valued game $(\MX, \MF, \ell_0)$ which contains infinitely many ordered thresholds with margin $\alpha$ and tightness $\beta$ satisfies, for some $\MX' \subset \MX,\ \MF' \subset \MF$,
  \begin{align}
\inf_{P_X \in \Delta(\MX')} \sup_{P_F \in \Delta(\MF')} \E_{(x,f) \sim P_X \times P_F} [\ell_0(x,f)] > \sup_{P_F \in \Delta(\MF')} \inf_{P_X \in \Delta(\MX')} \E_{(x,f) \sim P_X \times P_F} [\ell_0(x,f)] + \alpha - 2\beta.\nonumber
  \end{align}
\end{lemma}
\begin{proof}
  Let $x_1, x_2, \ldots \in \MX$ and $f_1, f_2, \ldots, \in \MF$ denotes a collection of infinitely many ordered thresholds with margin $\alpha$ and tightness $\beta$. Write $\MX' = \{x_1, x_2, \ldots \}$ and $\MF' = \{f_1, f_2, \ldots \}$. By definition there are $u,u' \in [0,1]$ so that $u'-u \geq \alpha$, $|f_i(x_j) - u| \leq \beta$ for $i \leq j$ and $|f_i(x_j) - u' | \leq \beta$ for $i > j$.  Then for any $P_X \in \Delta(\MX')$, we have
  \begin{align}
    & \sup_{P_F \in \Delta(\MF')} \E_{(x,f) \sim P_X \times P_F} [\ell_0(x,f)] \geq \sup_{i \geq 1} \E_{x_j \sim P_X} [\ell_0(x_j, f_i)] \geq \liminf_{i \ra \infty} \E_{x_j \sim P_X} [\ell_0(x_j, f_i)]\nonumber\\
    & \geq \E_{x_j \sim P_X} [ \liminf_{i \ra \infty} \ell_0(x_j, f_i) ] \geq u' - \beta\nonumber,
  \end{align}
  where the second-to-last inequality follows from Fatou's lemma.

  On the other hand, for any $P_F \in \Delta(\MF')$, we have
  \begin{align}
    & \inf_{P_X \in \Delta(\MX')} \E_{(x,f) \sim P_X \times P_F} [\ell_0(x,f)] \leq \inf_{j \geq 1} \E_{f_i \sim P_F} [\ell_0(x_j, f_i)] \leq \limsup_{j \ra \infty} \E_{f_i \sim P_F} [\ell_0(x_j, f_i)] \nonumber\\
    & \leq \E_{f_i \sim P_F} [\limsup_{j \ra \infty} \ell_0(x_j,f_i)] \leq u + \beta\nonumber,
  \end{align}
  where again the second-to-last inequality follows from Fatou's lemma. The two displays above complete the proof.
\end{proof}

By combining Lemmas \ref{lem:use-sugc} and \ref{lem:sugc-converse}, we are able to show the following necessary and sufficient condition for all subgames of an infinite GC game to satisfy the minimax theorem:
\begin{theorem}
Let $C_0 > 2$ be the constant of Lemma \ref{lem:use-sugc}.  A $[0,1]$-valued GC game $(\MX, \MF, \ell_0)$ satisfies
  \begin{align}
\inf_{P_X \in \Delta(\MX')} \sup_{P_F \in \Delta(\MF')} \E_{(x,f) \sim P_X \times P_F} [\ell_0(x,f)] = \sup_{P_F \in \Delta(\MF')} \inf_{P_X \in \Delta(\MX')} \E_{(x,f) \sim P_X \times P_F} [\ell_0(x,f)] \label{eq:minimax-subset}
  \end{align}
  for all $\MX' \subset \MX, \ \MF' \subset \MF$ if and only if it does not contain infinitely many ordered thresholds with margin $\alpha$ and tightness $\alpha / C_0$, for all $\alpha > 0$.
\end{theorem}
\begin{proof}
  First suppose that (\ref{eq:minimax-subset}) holds for all $\MX', \MF'$. If the game $(\MX, \MF, \ell_0)$ contained infinitely many ordered thresholds with margin $\alpha$ and tightness $\alpha / C_0$, then by Lemma \ref{lem:sugc-converse}, there would be some $\MX', \MF'$ so that the left-hand side of (\ref{eq:minimax-subset}) is at least the sum of $\alpha - 2 \alpha / C_0 > 0$ and the right-hand isde of (\ref{eq:minimax-subset}). This is a contradiction.

  Conversely, suppose that the game $(\MX, \MF, \ell_0)$ does not contain infinitely many ordered thresholds with margin $\alpha$ and tightness $\alpha /C_0$ for all $\alpha > 0$. Since, for each $\alpha > 0$ and $\MX'\subset \MX, \ \MF' \subset \MF$, $(\MX', \MF', \ell_0)$ is a $(\alpha/C_0)$-GC game, Lemma \ref{lem:use-sugc} gives that for each $\alpha > 0$,
  \begin{align}
\inf_{P_X \in \Delta(\MX')} \sup_{P_F \in \Delta(\MF')} \E_{(x,f) \sim P_X \times P_F} [\ell_0(x,f)] \leq \sup_{P_F \in \Delta(\MF')} \inf_{P_X \in \Delta(\MX')} \E_{(x,f) \sim P_X \times P_F} [\ell_0(x,f)] + 4\alpha\nonumber.
  \end{align}
  Then (\ref{eq:minimax-subset}) follows by taking $\alpha \downarrow 0$. 
\end{proof}

It is also immediate to show the minimax theorem for online learnable games, as follows:
\begin{theorem}[Minimax theorem for online learnable games]
  \label{thm:online-minimax}
  Any SGC game $(\MX, \MF, \ell)$ satisfies the min-max theorem, i.e.,
  \begin{align}
\inf_{P_X \in \Delta(\MX)} \sup_{P_F \in \Delta(\MF)} \E_{(x,f) \sim P_X \times P_F} [\ell(x,f)] = \sup_{P_F \in \Delta(\MF)} \inf_{P_X \in \Delta(\MX)} \E_{(x,f) \sim P_X \times P_F}[\ell(x,f)]\nonumber.
  \end{align}
  Further, the above equality remains true even if $P_X, P_F$ are restricted to lie in $\Delc(\MX), \Delc(\MF)$, respectively.
\end{theorem}
\begin{proof}
  By Lemma \ref{lem:sfat-dual} and since the given game is an SGC game, we have that the sequential fat-shattering dimension of the classes $\MF^\ell := \{ x \mapsto \ell(x,f) : f \in \MF \}$ and $\MX^\ell := \{ f \mapsto \ell(x,f) : x \in \MX \}$ is finite at all scales. Thus $(\MX, \MF, \ell)$ is a GC game. 
  The inequality
  \begin{align}
    \inf_{P_X \in \Delta(\MX)} \sup_{P_F \in \Delta(\MF)} \E_{(x,f) \sim P_X \times P_F} [\ell(x,f)] \geq \sup_{P_F \in \Delta(\MF)} \inf_{P_X \in \Delta(\MX)} \E_{(x,f) \sim P_X \times P_F}[\ell(x,f)]\nonumber
  \end{align}
  is immediate. To see the opposite direction, fix any $\alpha > 0$, and let $C_0$ be the constant of Lemma \ref{lem:use-sugc}. Certainly $(\MX, \MF, \ell)$ is a $(\alpha/C_0)$-GC game. Further, by Lemma \ref{lem:sfat-thresholds}, the maximum number of thresholds in the game $(\MX, \MF, \ell)$ with margin $\alpha$ and tightness $\alpha/C_0$ is $2^{O(\sfat_{\alpha \cdot (1 - 2/C_0)}(\MF^\ell))} < \infty$. It follows from Lemma \ref{lem:use-sugc} that
  \begin{align}
    \inf_{P_X \in \Delta(\MX)} \sup_{P_F \in \Delta(\MF)} \E_{(x,f) \sim P_X \times P_F} [\ell(x,f)] \leq \sup_{P_F \in \Delta(\MF)} \inf_{P_X \in \Delta(\MX)} \E_{(x,f) \sim P_X \times P_F}[\ell(x,f)] + 4\alpha\nonumber,
  \end{align}
and even if $P_X, P_F$ are restricted to $\Delc(\MX), \Delc(\MF)$, respectively. The statement of the theorem follows since $\alpha > 0$ may be taken arbitrarily small.
\end{proof}

\appendix
\section{Miscellaneous lemmas}
\label{sec:misc}
In this section we state some miscellaneous lemmas on the fat-shattering dimension of real-valued hypothesis classes. Many of these lemmas are well-known (see for instance \cite{golowich_differentially_2021}), but we state the proofs for completeness.

\begin{lemma}
  \label{lem:diff-1}
  Fix a class $\MF \subset [0,1]^\MX$ and $\alpha \in (0,1)$. There are at most 2 integers $j$, $0 \leq j < \lfloor 1/ \alpha \rfloor + 1$ so that
  $$
\sfat_\alpha(\MF) = \sfat_\alpha \left( \left\{ f \in \MF : f(x) \in [j\alpha, (j+1) \alpha) \right\} \right).
$$
Moreover, if there are 2 such integers $j$, they differ by 1.
\end{lemma}
\begin{proof}
  Suppose for the purpose of contradiction that for some $j_1, j_2$ with $|j_2 - j_1| \geq 2$, we have
  $$
  \sfat_\alpha(\MF) = \sfat_\alpha \left( \left\{ f \in \MF : f(x) \in [j_1\alpha, (j_1+1) \alpha) \right\} \right)= \sfat_\alpha \left( \left\{ f \in \MF : f(x) \in [j_2\alpha, (j_2+1) \alpha) \right\} \right).
  $$
  Set $j' = (j_1 + j_2)/2$. We therefore have that
  $$
\sfat_\alpha(\MF) = \sfat_\alpha( \{ f \in \MF : f(x) \geq j'\alpha + \alpha/2 \}) = \sfat_\alpha( \{ f \in \MF : f(x) \leq j'\alpha - \alpha/2 \}),
$$
which is a contradiction.
\end{proof}

\begin{lemma}
  \label{lem:far-sfat-dec}
For $\MF \subset [0,1]^\MX$ and $\alpha \in (0,1)$, and $(x,y) \in \MX \times [0,1]$, if $|y - \soal{\MF}{\alpha}{x}| > \alpha$, then $\sfat_\alpha(\MF|^\alpha_{(x,y)}) < \sfat_\alpha(\MF)$.
\end{lemma}
\begin{proof}
  Suppose for the purpose of contradiction that $\sfat_\alpha(\MF|_{(x,y)}^\alpha) = \sfat_\alpha(\MF)$. Let $j = \lfloor y/\alpha \rfloor$. Then by definition of $\MF|^\alpha_{(x,y)}$, we have that
  \begin{align}
\sfat_\alpha(\MF) = \sfat_\alpha ( \{ f \in \MF : f(x) \in [j\alpha, (j+1)\alpha) \} ).\nonumber
  \end{align}
  By definition of $\soalf{\MF}{\alpha}$, we have that $\soal{\MF}{\alpha}{x} = j^\st \alpha$ for some $0 \leq j^\st < \lfloor 1/\alpha \rfloor + 1$. It must hold that $\sfat_\alpha(\{ f \in \MF : f(x) \in [j^\st \alpha, (j^\st+1)\alpha) \}) = \sfat_\alpha(\MF)$. Since $|y - j^\st \alpha| > \alpha$, we have that $y \not \in [(j^\st-1)\alpha, (j^\st+1)\alpha)$, meaning that $j \not \in \{j^\st, j^\st - 1 \}$. By Lemma \ref{lem:diff-1} we must have $j= j^\st+1$. But the definition of $\soalf{\MF}{\alpha}$ requires that in this case that $\soal{\MF}{\alpha}{x} = (j^\st+1)\alpha$, which is a contradiction. %
\end{proof}

\paragraph{Uniform convergence.}
Next we state a uniform convergence bound for real-valued hypothesis classes. The below bound is not optimal (as it only considers the fat-shattering dimension at a single scale), but as it does not quantitatively affect our statistical rates  for online learning, it will suffice for our purposes. %

For uniform convergence (which implies learnability under i.i.d.~data), finiteness of the \emph{fat-shattering dimension} \cite{alon_scale-sensitive_1997}, which is smaller than the sequential fat-shattering dimension, is sufficient (and necessary). The fat-shattering dimension of a hypothesis class $\MF \subset [0,1]^\MX$ at scale $\alpha > 0$, denoted $\fat_\alpha(\MF)$, is defined as follows. It is the largest positive integer $d$ so that there are $x_1, \ldots, x_d \in \MX$ and $s_1, \ldots, s_d \in [0,1]$ so that for each choice of $\ep_1, \ldots, \ep_d \in \{-1,1\}$ it holds that there is some $f \in \MF$ so that, for each $i \in [d]$, $\ep_i \cdot (f(x_i) - s_i) \geq \alpha/2$. 
\begin{theorem}[Uniform convergence; \cite{mendelson_entropy_2002}\footnote{For an explanation of how the theorem follows from \cite{mendelson_entropy_2002}, see \cite[Corollary 20]{golowich_differentially_2021}.}]
  \label{thm:unif-conv}
  There are constants $C_0 \geq 1$ and $0 < c_0 \leq 1$ so that the following holds. For any $\MF \subset [0,1]^\MX$, and finite-support distribution $P$\footnote{The finite-suportedness assumption can be dropped if $\MF$ is countable.} on $\MX$, and any $\gamma \in (0,1/2), \eta \in (0,1/2)$, it holds that for any
  \begin{align}
n \geq C_0 \cdot \frac{\fat_{c_0\eta}(\MF) \log(1/\eta) + \log(1/\gamma)}{\eta^2}\nonumber,
  \end{align}
  we have
  \begin{align}
    \Pr_{x_1, \ldots, x_n \sim P} \left[
\sup_{f \in \MF} \left| \E_{x\sim P}[f(x)] - \frac 1n \sum_{i=1}^n f(x_i) \right| > \eta
    \right] \leq \gamma\nonumber.
  \end{align}
\end{theorem}

\paragraph{Closure bound for the sequential fat-shattering dimension.}
Next we establish a closure bound for the sequential fat-shattering dimension; the result is the real-valued analogue of \cite[Proposition 2.3]{ghazi_near-tight_2021}, and is also similar to \cite[Lemma 4]{rakhlin_online_2015}, which proves an analogue for the sequential Rademacher complexity. To begin, we establish some additional preliminaries, following \cite{rakhlin_online_2015}: for some set $\MZ$ and a function class $\MF \subset [0,1]^\MZ$, fix a $\MZ$-valued tree $\bz$ of depth $d$, and consider a set $\MV$ of $\BR$-valued trees of depth $d$. For $\alpha > 0$, the set $\MV$ is defined to be a \emph{sequential $\alpha$-cover} of $\MF$ on the tree $\bz$ if for all $f \in \MF$, and all $\ep \in \{-1,1\}^d$, there is some $\bv \in \MV$ so that
\begin{align}
\max_{t \in [d]} \left|\bv_t(\ep_{1:t-1}) - f(\bz_t(\ep_{1:t-1})) \right| \leq \alpha.\label{eq:infty-cover}
\end{align}
Given a class $\MF$, the \emph{sequential $\alpha$-covering number} (with respect to $\ell_\infty$) for the tree $\bz$ is defined as follows:
\begin{align}
\MN_\infty(\MF, \bz, \alpha) := \min \left\{ |\MV| \ : \ \text{$\MV$ is a sequential $\alpha$-cover of $\MF$ on the tree $\bz$} \right\}\nonumber.
\end{align}

Next we need a few basic lemmas that related the sequential covering numbers of classes and their sequential fat-shattering dimension.
\begin{lemma}[Theorem 14.5, \cite{rakhlin_statistical_2014}]
  \label{lem:seq-cover-ub}
  Consider a class $\MF \subset [0,1]^\MZ$. Then for any $\alpha > 0$, and $d \in \BN$, and any $\MZ$-valued tree $\bz$ of depth $d \geq \sfat_\alpha(\MF)$,
  \begin{align}
\MN_\infty(\MF, \bz, \alpha) \leq & \left( \frac{2e d}{\alpha \cdot \sfat_\alpha(\MF)} \right)^{\sfat_\alpha(\MF)}\nonumber.
  \end{align}
\end{lemma}
We remark that in the statement of \cite[Theorem 14.5]{rakhlin_statistical_2014} the term $\sfat_\alpha(\MF)$ does not appear in the denominator in the upper bound on $\MN_\infty(\MF, \bz, \alpha)$. However, a close inspection of their proof shows that they establish $\MN_\infty(\MF, \bz, \alpha) \leq \left( \frac{2ed}{\alpha m} \right)^m$ for some $m \leq \sfat_\alpha(\MF)$ (namely, $m$ is the pararmeter called $\fat_2(\MG)$ therein). The statement of Lemma \ref{lem:seq-cover-ub} then follows by noting that the function $m \mapsto \left( \frac{2ed}{\alpha m} \right)^m$ is non-decreasing for $m \leq d$.

\begin{lemma}
  \label{lem:seq-cover-lb}
Suppose a tree $\bz$ of depth $d$ is $\alpha$-shattered by a class $\MF$. Then $\MN_\infty(\MF, \bz, \beta) \geq 2^d$ for any $\beta < \alpha/2$.
\end{lemma}
\begin{proof}
  Let $\bs$ be an $\BR$-valued tree that witnesses the shattering of $\bz$. 
  Let $\MV$ be a $\beta$-cover of $\MF$ on the tree $\bz$. Consider any two leaves $\ep, \ep' \in \{-1,1\}^d$ of the tree $\bz$, and let  corresponding functions in $\MF$ be denoted $f,f'$. (For a leaf $\ep \in \{-1,1\}^d$, a corresponding function $f \in \MF$ is any function so that $\ep_t \cdot (f(\bz_t(\ep_{1:t-1})) - \bs_t(\ep_{1:t-1})) \geq \alpha/2$ for each $t \in [d]$.) Let $\bv, \bv' \in \MV$ be the elements of the cover $\MV$ as guaranteed by (\ref{eq:infty-cover}) for the leaves $\ep, \ep'$. We claim that $\bv \neq \bv'$, which would immediately complete the proof; so suppose to the contrary that $\bv = \bv'$.

  Choose $t$ as small as possible so that $\ep_t \neq \ep_t'$. Then (perhaps after interchanging the roles of $f,f'$), it holds that
  \begin{align}
  f(\bz_t(\ep_{1:t-1})) \geq \bs_t(\ep_{1:t-1}) + \alpha/2, \qquad  f'(\bz_t(\ep_{1:t-1})) \leq \bs_t(\ep_{1:t-1}) - \alpha/2.\label{eq:ffp-far}
  \end{align}
  On the other hand, since $\MV$ is a $\beta$-cover of $\MF$, we have (since $\bv = \bv'$) that 
  \begin{align}
|\bv_t(\ep_{1:t-1}) - f(\bz_t(\ep_{1:t-1}))| \leq \beta, \qquad |\bv_t(\ep_{1:t-1}) - f'(\bz_t(\ep_{1:t-1}))| \leq \beta\label{eq:ffp-close}.
  \end{align}
  Using that $\beta < \alpha/2$, we get that (\ref{eq:ffp-far}) and (\ref{eq:ffp-close}) lead to a contradiction, thus completing the proof of the lemma.
\end{proof}

Given some $k \in \BN$, some function $\phi : \BR^k \times \MZ \ra \BR$, and function classes $\MF_1, \ldots, \MF_k \subset [0,1]^\MX$, define the \emph{$\phi$-composition} of $\MF_1, \ldots, \MF_k$ as follows:
\begin{align}
\phi(\MF_1, \ldots, \MF_k) := \left\{ z \mapsto \phi(f_1(z), \ldots, f_k(z), z) \ : \ f_1 \in \MF_1, \ldots, f_k \in \MF_k \right\}\nonumber.
\end{align}
\begin{lemma}
  \label{lem:sfat-closure}
  Consider classes $\MF_1, \ldots, \MF_k \subset [0,1]^\MZ$, and consider a function $\phi : \BR^k \times \MX \ra \BR$ so that $\phi(\cdot, z)$ is $L$-Lipschitz for each $z \in \MZ$. Fix any $\alpha > 0$, and suppose that $\sfat_{\alpha/(4L)}(\MF_i) \leq d$ for each $i \in [k]$ and some $d \in \BN$. Then
  \begin{align}
    \sfat_\alpha(\phi(\MF_1, \ldots, \MF_k)) \leq & O\left( dk \log \left( \frac{Lk}{\alpha} \right) \right).\nonumber
  \end{align}
\end{lemma}
\begin{proof}
  Write $\MG := \phi(\MF_1, \ldots, \MF_k)$ to denote the composed class. 
  Fix $\alpha > 0$, and write $N := \sfat_\alpha(\phi(\MF_1, \ldots, \MF_k))$. Let $\bz$ be a $\MZ$-valued binary tree of depth $N$ that is $\alpha$-shattered by $\MG$. By Lemma \ref{lem:seq-cover-ub}, for each $i \in [k]$, we have that, for $\beta = \alpha/(4L)$,
  \begin{align}
\MN_\infty(\MF_i, \bz, \beta) \leq \left( \frac{2e N}{\beta \cdot \sfat_\beta(\MF_i)} \right)^{\sfat_\beta(\MF_i)}\nonumber.
  \end{align}
  For each $i \in [k]$, let $\MV_i$ be a minimal $\beta$-cover for the class $\MF_i$ on the tree $\bz$. Now consider the set
  \begin{align}
\MV := \left\{ \bv = \phi(\bv^1, \ldots, \bv^k) \ : \ \bv^1 \in \MV_1, \ldots, \bv^k \in \MV_k \right\}\nonumber,
  \end{align}
  where $\phi(\bv^1, \ldots, \bv^k)$ denotes the $\BR$-valued tree defined by
  \begin{align}
\phi(\bv^1, \ldots, \bv^k)_t(\ep_{1:t-1}) := \phi(\bv_t^1(\ep_{1:t-1}), \ldots, \bv_t^k(\ep_{1:t-1}), \bz_t(\ep_{1:t-1}))\nonumber.
  \end{align}
  Now fix any $g \in \MG$; it can be written as $g(z) = \phi(f_1(z), \ldots, f_k(z), z)$ for some $f_1 \in \MF_1, \ldots, f_k \in \MF_k$. For each $i \in [k]$, let $\bv^i \in \MF_i$ denote a representative for $f_i$ in the sense that for each $i \in [k]$,
  \begin{align}
\max_{t \in [d]} \left| \bv_t^i(\ep_{1:t-1}) - f_i(\bz_t(\ep_{1:t-1})) \right| \leq \beta\label{eq:vi-close}.
  \end{align}
  Then
  \begin{align}
    & \max_{t \in [d]} \left| \phi(\bv^1, \ldots, \bv^k)_t(\ep_{1:t-1}) - g(\bz_t(\ep_{1:t-1})) \right| \nonumber\\
    =& \max_{t \in [d]} \left| \phi(\bv_t^1(\ep_{1:t-1}), \ldots, \bv_t^k(\ep_{1:t-1}), \bz_t(\ep_{1:t-1})) - \phi(f_1(\bz_t(\ep_{1:t-1})), \ldots, f_k(\bz_t(\ep_{1:t-1})), \bz_t(\ep_{1:t-1})) \right|\nonumber\\
    \leq & L\beta = \alpha/4 \tag{Using (\ref{eq:vi-close}) and $L$-Lipschitzness of $\phi$}.
  \end{align}
  By Lemma \ref{lem:seq-cover-lb}, since $\bz$ is $\alpha$-shattered by $\MG$, we have that $\MN_\infty(\MG, \bz, \alpha/4) \geq 2^N$. On the other hand, as we have shown above, the set $\MV$ is a sequential $\alpha/4$-cover for the class $\MG$ on the tree $\MV$. Thus,
  \begin{align}
2^N \leq \MN_\infty(\MG, \bz, \alpha/4) \leq |\MV| \leq \prod_{i=1}^k \left( \frac{4Le N}{\alpha \cdot \sfat_{\alpha/(4L)}(\MF_i)} \right)^{\sfat_{\alpha/(4L)}(\MF_i)}\nonumber,
  \end{align}
  which implies that
  \begin{align}
N \leq \sum_{i=1}^k \sfat_{\alpha/(4L)}(\MF_i) \cdot \log \left( \frac{4Le N}{\alpha \cdot \sfat_{\alpha/(4L)}(\MF_i)} \right)\nonumber.
  \end{align}
  Recalling the assumption that $\sfat_{\alpha/(4L)}(\MF_i) \leq d$ for each $i$ and using that $m \mapsto m \cdot \log \left( \frac{4LeN}{\alpha m} \right)$ is a non-decreasing function for $m \leq N$, we obtain that $N \leq O\left(dk \log \left( \frac{Lk}{\alpha}\right)\right)$. 
  
\end{proof}

\begin{corollary}
  \label{cor:closure-union}
There is a constant $C \geq 1$ so that the following holds.  Suppose $\MF_1, \ldots, \MF_k \subset [0,1]^\MX$ satisfy $\sfat_{\alpha/4}(\MF_i) \leq d$ for $i \in [k$]. Then for any $\alpha > 0$,
  \begin{align}
\sfat_\alpha(\MF_1 \cup \cdots \cup \MF_k) \leq C \cdot dk \cdot \log(k/\alpha).
  \end{align}
\end{corollary}
\begin{proof}
  For $i \in [k]$, define $\MF_i' := \MF_i \cup \{ \mathbf{0} \}$, where $\mathbf{0}$ denotes the function that is identically 0 on $\MZ$. Then clearly $\sfat_{\alpha/4}(\MF_i') \leq \sfat_{\alpha/4}(\MF_i)+1 \leq d+1$. For $a_1, \ldots, a_k \in \BR$ and $z \in \MZ$, define $\phi(a_1, \ldots, a_k, z) := a_1 + \cdots + a_k$, which is clearly 1-Lipschitz with respect to $\| \cdot \|_\infty$. Further,
  \begin{align}
\MF_1 \cup \cdots \cup \MF_k \subset \phi(\MF_1', \ldots, \MF_k')\nonumber,
  \end{align}
  since for each $i \in [k]$ and $f_i \in \MF_i$, $\phi(\MF_1', \ldots, \MF_k')$ contains the function $z \mapsto \phi(\mathbf{0}, \ldots, f_i(z), \ldots, \mathbf{0}, z) = f_i(z)$. The result now follows from Lemma \ref{lem:sfat-closure} with $L=1$.
\end{proof}

\section{Proof of Proposition \ref{prop:sfat-lb}}
In this section we prove Proposition \ref{prop:sfat-lb}, %
thus showing that the bound of Theorem \ref{thm:proper-stable-main} is optimal up to a $\poly \log T$ factor

\begin{proof}[Proof of Proposition \ref{prop:sfat-lb}]
  By compactness of $[1/T,1]$, the infimum $M:= \inf_{\alpha \in [1/T,1]} \left\{ \alpha T + \int_\alpha^1 s(\eta) d\eta \right\}$ is obtained at some $\alpha_0 \in [1/T,1]$. %

  Note also that the mapping $\alpha \mapsto \alpha T + \int_\alpha^1 s(\eta) d\eta$ is convex (its derivative is $T - s(\alpha)$, which is non-decreasing), meaning that %
  for all $\alpha > \alpha_0$, $ T - s(\alpha)\geq 0$, and for $1/T \leq \alpha < \alpha_0$, $T - s(\alpha)\leq 0$. Thus, by increasing $\alpha_0$ by a factor of $3/2$, we can ensure that $T - s(\alpha_0) \geq 0$ and for all $1/T \leq \alpha \leq \alpha_0/2$, $T - s(\alpha) \leq 0$ (further, doing so can only increase $\alpha T + \int_\alpha^1 s(\eta) d\eta$ by definition of $\alpha_0$). %

    For any $\alpha \in [0,1]$, note that
  \begin{align}
\alpha T + \int_\alpha^1 s(\eta) d\eta \leq \alpha T + \sum_{i \geq 0  : \ 2^i \alpha \leq 1} 2^{i}\alpha \cdot s(2^i \alpha)\nonumber.
  \end{align}
  Thus setting $\alpha = \alpha_0$ in the above display, one of the following possibilities holds:
  \begin{enumerate}
  \item $\alpha_0 T \geq M/2$. In this case, set $\alpha_0' = \alpha_0/2$, and set $d := T \leq s(\alpha_0')$. \label{it:adv-T}
  \item There is some $i \leq \lceil \log 1/\alpha_0 \rceil \leq \lceil \log T \rceil$ so that $2^i \alpha_0 \cdot s(2^i \alpha_0) \geq M/(2 \lceil \log T \rceil)$. In this case, set $d := s(2^i \alpha_0) \leq T$ (by definition of $\alpha_0$, and using that $s(\alpha_0) \leq T$). Now set $\alpha_0' := 2^i \cdot \alpha_0$.\label{it:adv-d}
  \end{enumerate}

  Set $\MX = \{1, 2, \ldots, d \}$, and let $\MF$ be the class of all functions on $\MX$ so that for each $x \in \MS$, $f(x) \in \{(1-\alpha_0')/2, (1+\alpha_0')/2\}$. Clearly, $\sfat_{\alpha}(\MF) = d$ for all $\alpha \leq \alpha_0'$, and $\sfat_\alpha(\MF) = 0$  for all $\alpha > \alpha_0'$. Thus $\sfat_\alpha(\MF) \leq s(\alpha)$ for all $\alpha \in [0,1]$.

  Further, the adversary can clearly force a \cumls of at least $\frac{\alpha_0'}{2} \cdot d$: simply feed each of the examples $x_1, \ldots, x_d$ (using that $d \leq T$), and set $y_t$ to be whichever of $(1-\alpha_0'/2), (1+\alpha_0')/2$ is further from the algorithm's prediction at time $t$. In case \ref{it:adv-T} above, this \cumls becomes $\alpha_0 T/2 \geq \Omega(M)$, and in case \ref{it:adv-d} above, this \cumls becomes $2^i\alpha_0 \cdot s(2^i \alpha_0) \geq \Omega(M/\log T)$. Thus, in both cases, we get a \cumls of $\Omega(M/\log T)$, as desired.
\end{proof}

\section*{Acknowledgements}
We are grateful to Sasha Rakhlin for helpful suggestions and to Steve Hanneke for a useful conversation. 

\bibliographystyle{alpha}
\bibliography{online-learning-games}

\noah{(major) TODOs:
\begin{enumerate}
\item Topological assumptions needed for min-max swap. DONE
\item Littlestone dimension game case (using path length reg bound for a stable online learner). DONE
\item Dig up the stat papers I found a while ago which claim to get some improve online learning guarantees for the realizable case but don't really...DONE
\item Check sasha's min-max thing. DONE
\item {\bf lower bound (should be easy).DONE }
\item {\bf Uniform convergence params. DONE (without checking carefully)} Checked more carefully now.
\end{enumerate}
}
\end{document}